\documentclass{article}

\usepackage{microtype}
\usepackage{graphicx}
\usepackage{subfigure}
\usepackage{booktabs} 

\usepackage{hyperref}
\usepackage{natbib}
\usepackage{breakurl}

\usepackage[accepted]{icml2021}

\usepackage{bm}
\usepackage{bbm}
\usepackage{multicol,multirow}
\usepackage{comment}
\usepackage{booktabs}

\usepackage{enumerate}   
\usepackage{amsmath}

\usepackage{wrapfig}

\usepackage{appendix}
\usepackage{amssymb}

\DeclareMathOperator*{\argmin}{arg\,min}

\usepackage{amsthm}
\usepackage{hyperref}
\usepackage[capitalise]{cleveref}
\Crefname{assumption}{Assumption}{Assumptions}

\theoremstyle{plain}
\newtheorem{theorem}{Theorem}
\newtheorem{lemma}{Lemma}

\theoremstyle{definition}
\newtheorem{assumption}{Assumption}
\newtheorem{proposition}{Proposition}
\newtheorem{definition}{Definition}
\newtheorem{remark}{Remark}

\newcommand{\iid}{ \stackrel{i.i.}{\sim} }

\newcommand{\bigO}{\mathcal{O}} %

\newcommand{\pa}{\mathrm{\pa}}

\renewcommand{\eqref}[1]{(\ref{#1})}

\newcommand{\RN}[1]{%
  \textup{\uppercase\expandafter{\romannumeral#1}}%
}

\newcount\Comments  
\def \Rpositive {{\mathbb{R}^{+}}}

\def \hr {\hat{r}}

\def \rstar {r^*}

\def \dx {{\mathrm{d}\bm{x}}}

\def \iid {\overset{\mathrm{i.i.d.}}{\sim}}
\def \DOne {D_1}
\def \DTwo {D_2}
\def \Reg {\mathcal{R}}

\newcommand{\annot}[2]{\underbrace{#1}_{\text{#2}}}

\def \Re {\mathbb{R}}
\def \Na {\mathbb{N}}

\def \r {r}
\def \Ltwo {L^2}
\def \ltwo {\ell_2}
\newcommand \Probability[1]{\mathbb{P}\left(#1\right)}
\def \hE {\hat{\E}}
\def \F {\mathcal{F}}

\def \f {f}

\def \supf {\sup_{\f \in \F}}

\def \loss {\ell}
\newcommand{\br}{f} 
\def \tbr {\tilde\br}
\def \dbr {\partial \br}
\def \lossOne {\loss_1}
\def \lossTwo {\loss_2}
\def \LossBound {{B_\loss}}
\def \Cons {C}
\def \modFn{\rho} 
\newcommand{\rClass}{\mathcal{H}}
\newcommand{\rClassBound}{B_r}
\newcommand{\rClassBoundMin}{b_r}
\newcommand{\rClassRangeTwo}{(\rClassBoundMin, \rClassBound)}
\newcommand{\rClassRangeTwoName}{I_r}

\def \supr {\sup_{\r \in \rClass}}
\def \supt {\sup_{t \in \rClassRangeTwoName}}

\def \pnu {p_{\mathrm{nu}}}
\def \pde {p_{\mathrm{de}}}
\def \pmod {p_{\mathrm{mod}}}
\newcommand{\rmax}{\overline{R}}

\def \Enu {\E_{\mathrm{nu}}}
\def \Ede {\E_{\mathrm{de}}}
\def \Emod {\EX\hEmod} 
\def \nde {n_{\mathrm{de}}}
\def \nnu {n_{\mathrm{nu}}}
\def \hEnu {\hat{\E}_{\mathrm{nu}}}
\def \hEde {\hat{\E}_{\mathrm{de}}}
\def \hEmod {\hat{\E}_{\mathrm{mod}}}
\newcommand{\Risk}{\mathrm{BD}_{\br}}
\newcommand{\hRisk}{\widehat{\mathrm{BD}}_{\br}}
\newcommand{\nnhRisk}{\widehat{\mathrm{nnBD}}_{\br}}
\newcommand{\EnnhRisk}{\E\nnhRisk}

\def \rstar {r^*}
\def \rbest {\bar{r}}
\def \hr {\hat{\r}}

\newcommand{\Rademacher}[2]{\Rad_{#1}^{#2}}
\newcommand{\Radnu}{\Rademacher{\nnu}{\pnu}}
\newcommand{\Radde}{\Rademacher{\nde}{\pde}}

\def \LipmodFn {L_\modFn}
\def \LipOne {L_{\lossOne}}
\def \LipTwo {L_{\lossTwo}}

\def \rad {\sigma}
\def \ERad {\E_{\rad}}
\def \Rad {\mathcal{R}}

\def \EX {\E}

\def \ghostMark {(\mathrm{gh})}
\def \EXghost {\E^{\ghostMark}}
\def \hEjk {\hE_{(i, j)}}
\def \hEghost {\hat{\E}^{\ghostMark}}
\def \hEghostjk {\hEghost_{(j, k)}}
\def \X {X}
\def \Xghost {X^{\ghostMark}}
\def \hS {\hat{S}}
\def \EhS {S}
\def \ljk {\loss_{(j, k)}}
\def \sumj {\sum_{j=1}^J}
\def \sumjk {\sum_{k=1}^{K_j}}
\def \Liprhoj {L_{\rho_j}}
\def \LipLossjk {L_{\ljk}}
\def \Radjk {\Rad_{\njk, \pjk}}

\def \njk {n_{(j, k)}}
\def \pjk {p_{(j, k)}}
\def \InSpace {\mathcal{X}}

\def \Xsetjk {\mathcal{D}_{(j, k)}}

\newcommand \Indicator[1]{\mathbbm{1}\{#1\}}
\def \Identity {\mathrm{Id}}
\def \supp {\mathrm{supp}}
\def \modFnIdSupp {\supp(\modFn - \Identity)}
\def \modFnIdLip {L_{\modFn - \Identity}}
\def \infr {\inf_{\r \in \rClass}}
\newcommand{\lOneR}{\lossOne(\r(X))}

\newcommand{\BiasTerm}{\Phi_{(\Cons, \br, \modFn)}(\nnu, \nde)}
\newcommand{\BiasTermMainText}{\Phi_{\Cons}^\br(\nnu, \nde)}
\def \supLossVal {\sup_{s: |s| \leq (1+\Cons)\LossBound}}

\def \InputBound {B_p}
\def \ParamBoundj {B_{W_j}}
\newcommand{\Radnp}{\Rademacher{n}{p}}

\def \E {\mathbb{E}}

\newcommand{\lTwoR}{\lossTwo(\r(X))}
\newcommand{\lOne}[1]{\lossOne^{#1}}
\newcommand{\lTwo}[1]{\lossTwo^{#1}}
\newcommand{\lzeroNrm}[1]{\|#1\|_0}
\newcommand{\hEmodTwo}{\hEmod'}
\newcommand{\nnuz}{\nnu^0}
\newcommand{\ndez}{\nde^0}
\newcommand{\AppdxEmpProcMcDiarmidRadSum}{\mathcal{R}}
\newcommand{\smallO}{
  \mathchoice
  {{\scriptstyle\mathcal{O}}}
  {{\scriptstyle\mathcal{O}}}
  {{\scriptscriptstyle\mathcal{O}}}
  {\scalebox{.6}{$\scriptscriptstyle\mathcal{O}$}}
}
\newcommand{\order}{\smallO{}}
\newcommand{\Order}[1]{\bigO\left(#1\right)}

\newcommand{\Orderp}[1]{\bigO_{\mathbb{P}}\left(#1\right)}
\newcommand{\sparseNetworkRadBoundConst}{c}
\newcommand{\Lmax}{\bar{L}}
\newcommand{\pmax}{\bar{p}}

\newcommand{\rClassM}{\rClass_M}
\newcommand{\bracketEntropy}[3]{H_B\left(#1, #2, #3\right)}
\newcommand{\coveringNumber}[3]{\mathcal{N}\left(#1, #2, #3\right)}
\newcommand{\inftyCoveringNumber}[2]{\coveringNumber{#1}{#2}{\|\cdot\|_\infty}}
\newcommand{\rClassMSupNormBoundConst}{c_0}

\newcommand{\IndLP}{\mathrm{Ind}_{\Lmax,\pmax}}
\newcommand{\IndLPM}{\stackrel{(L, p) \in \IndLP}{I_1(L, p) \leq M}}
\newcommand{\IndLPMInline}{(L, p) \in \IndLP: I_1(L, p) \leq M}
\newcommand{\rClassLP}{\rClass(L, p, s, F)}
\newcommand{\elled}[1]{\ell \circ #1}
\newcommand{\ellLip}{\nu}
\newcommand{\vmax}[2]{\max\left\{#1, #2\right\}}
\newcommand{\vmin}[2]{\min\left\{#1, #2\right\}}
\newcommand{\exponentialOrderOne}{\exp\left(- \frac{2 \alpha^2}{(\LossBound^2/\nde) + (\Cons^2\LossBound^2 / \nnu)}\right)}

\icmltitlerunning{Non-Negative Bregman Divergence Minimization for Deep Direct Density Ratio Estimation}

\begin{document}

\twocolumn[
\icmltitle{Non-Negative Bregman Divergence Minimization\\
for Deep Direct Density Ratio Estimation}




\begin{icmlauthorlist}
\icmlauthor{Masahiro Kato}{to}
\icmlauthor{Takeshi Teshima}{goo}
\end{icmlauthorlist}

\icmlaffiliation{to}{CyberAgent, Inc., Tokyo, Japan}
\icmlaffiliation{goo}{The University of Tokyo, Tokyo, Japan}

\icmlcorrespondingauthor{Masahiro Kato}{masahiro\_kato@cyberagent.co.jp}
\icmlcorrespondingauthor{Takeshi Teshima}{teshima@ms.k.u-tokyo.ac.jp}

\icmlkeywords{Density ratio estimation, Anomaly detection, Covariate shift adaptation, PU learning, Bregman divergence}

\vskip 0.3in
]



\printAffiliationsAndNotice{} 

\begin{abstract}
\emph{Density ratio estimation} (DRE) is at the core of various machine learning tasks such as anomaly detection and domain adaptation. In existing studies on DRE, methods based on \emph{Bregman divergence} (BD) minimization have been extensively studied. However, BD minimization when applied with highly flexible models, such as deep neural networks, tends to suffer from what we call \emph{train-loss hacking}, which is a source of \emph{overfitting} caused by a typical characteristic of empirical BD estimators. In this paper, to mitigate train-loss hacking, we propose a \emph{non-negative correction} for empirical BD estimators. Theoretically, we confirm the soundness of the proposed method through a generalization error bound. Through our experiments, the proposed methods show a favorable performance in inlier-based outlier detection.
\end{abstract}

\section{Introduction}
Density ratio estimation (DRE) has attracted a great deal of attention as an essential task in various machine learning problems, such as regression under a covariate shift \citep{shimodaira2000improving,Reddi2015}, learning with noisy labels \citep{liu2014noisy,fang2020rethinking}, anomaly detection \citep{smola09a,Hido2011,Abe2019}, two-sample testing \citep{keziou2005,kanamori2010,sugiyama2011a}, causal inference \citep{kato_uehara_2020}, change point detection \citep{kawahara2009}, and binary classification only from positive and unlabeled data (PU learning; \citealp{kato2018learning}). For instance, anomaly detection is not easy to apply based on standard machine learning methods since anomalous data are often scarce; however, we can solve this by estimating the density ratio when anomaly-free unlabeled test data are available \citep{hido2008}.

Among the various approaches to DRE, we focus on the Bregman divergence (BD) minimization framework \citep{bregman1967relaxation,sugiyama2011bregman}, which is a general framework that unifies various DRE methods, such as moment matching \citep{huang2007,gretton2009}, probabilistic classification \citep{qin1998,Cheng2004}, density matching \citep{nguyen2010}, and density-ratio fitting \citep{JMLR:Kanamori+etal:2009}. \citet{kato2018learning} proposed using the risk of PU learning for DRE, which also falls within this framework (Appendix~\ref{appdx:exist_DRE}).

Existing methods have mainly adopted linear-in-parameter models for DRE \citep{KanamoriStatistical2012a}. On the other hand, recent studies in machine learning have suggested that deep neural networks achieve significantly high performances for various tasks, such as computer vision (CV) \citep{Krizhevsky2012} and natural language processing (NLP) \citep{bengio2001}. These findings motivate us to use deep neural networks for DRE.

However, when using deep neural networks in combination with empirical BD minimization, we often observe a serious \emph{overfitting} problem as experimentally demonstrated in Figure~\ref{fig:numerical_exp1} of Section~\ref{sec:num_exp} and Figure~\ref{fig:numerical_exp2} of Appendix~\ref{sec:exp_val_estimators}. We observe that this is mainly because a model of the density ratio $r(X) = \frac{q(X)}{p(X)}$ between two probability densities $p$ and $q$ becomes large for high-dimensional data when using a flexible model. As the intuition behind this phenomenon, a flexible model can overfit the samples generated from $p(X)$ as if there were no common support between $p(X)$ and $q(X)$ (Figure~\ref{fig1:box_plot}). This hypothesis is inspired by \citet{KiryoPositiveunlabeled2017}, which reports a similar problem in PU learning. In the case of DRE through BD minimization, we conjecture that this phenomenon is caused by the objective function that monotonically decreases with respect to $r(X)$, misleading the model $r(X)$ to take on as large a value as possible on the specific data points $X$. Note that even if $r(X)$ is bounded, this phenomenon still manifests as the model takes on the largest possible value within its output range. \citet{Rhodes2020a} and \citet{Ansari2020iclr} independently found related problems in DRE. Whereas \citet{KiryoPositiveunlabeled2017} and \citet{Rhodes2020a} call their phenomena \emph{overfitting} and \emph{density-chasm problem}, respectively, we refer to our problem as \emph{train-loss hacking} because this problem is specific to methods based on BD minimization, and we can still observe this issue even when the true $r(x)=\frac{p(x)}{q(x)}$ is not significantly large. This problem is discussed in more detail in Section~\ref{subsec:diff_dre} (Figure~\ref{fig1:box_plot}).

Owing to this property, training a density ratio model with a flexible model tends to result in either a diverging empirical BD estimator or a model sticking to the upper bound of its output range. For instance, when the empirical BD divergence is not lower bounded, it often numerically diverges to negative infinity. Even when the loss function has a lower bound (see \emph{BKL} and \emph{Bounded uLSIF} introduced in Sections~\ref{sec:example_dre} and \ref{sec:num_exp}), the trained models tend to stick to the largest possible value of their output ranges (Bounded uLSIF in Figure~\ref{fig:numerical_exp1} and BKL-NN in Figure~\ref{fig:numerical_exp2}). Although train-loss hacking has rarely been discussed in existing studies on DRE, this problem is often encountered while using deep neural networks, as experimentally shown in Section~\ref{sec:num_exp}. One reason for this is that the existing studies use linear-in-parameter models \citep{KanamoriStatistical2012a} or simple shallow neural networks \citep{Hyunha2015,Abe2019} for the density ratio models, which tend to be inflexible in that they do not cause such a phenomenon.

To mitigate the train-loss hacking, we propose a general procedure to modify the empirical BD estimator. First, from the empirical BD divergence, we separate the term causing the train-loss hacking. We then apply a \emph{non-negative correction} to the term to make the model consistent with a constraint that should be satisfied within the population. Our idea of this correction is inspired by \citet{KiryoPositiveunlabeled2017}. However, their idea of a non-negative correction is only applicable to the binary classification setting; thus, we require a non-trivial rewriting of the BD to generalize the approach to our problem. We call our proposed objective function the \emph{non-negative BD} (nnBD). The proposed method can be regarded as a generalization of the method proposed by \citet{KiryoPositiveunlabeled2017}.


Our main contributions are (1) proposal of a general procedure to modify an empirical BD estimator to enable DRE with flexible models, (2) theoretical justification of the proposed estimator, and (3) experimental validation of the proposed method using benchmark data. 

\section{Problem setting}
\label{sec:prob}
Let $\mathcal{X}^{\mathrm{nu}}\subseteq\mathbb{R}^d$ and $\mathcal{X}^{\mathrm{de}}\subseteq\mathbb{R}^d$ be the spaces of the $d$-dimensional \emph{covariates}.
Here, ``nu'' and ``de'' indicate the numerator and denominator.
Let $p_{\mathrm{nu}}$ and $p_{\mathrm{de}}$ be the probability densities over $\mathcal{X}^{\mathrm{nu}}$ and $\mathcal{X}^{\mathrm{de}}$, respectively.
We have independent and identically distributed (i.i.d.) samples from these distributions: ${\bf{X}}^\mathrm{nu} = \big\{X^{\mathrm{nu}}_j\big\}^{n_\mathrm{nu}}_{j=1}\iid p_{\mathrm{nu}}$ and ${\bf{X}}^\mathrm{de} = \big\{X^{\mathrm{de}}_i\big\}^{n_\mathrm{de}}_{i=1}\iid p_{\mathrm{de}}$.

\paragraph{Basic assumption and goal.}
Throughout this paper, we assume that $p_{\mathrm{nu}}(X)$ and $p_{\mathrm{de}}(X)$ are strictly positive over $\mathcal{X}^{\mathrm{nu}}$ and $\mathcal{X}^{\mathrm{de}}$, respectively.
We also assume $\mathcal{X}^{\mathrm{nu}}\subseteq\mathcal{X}^{\mathrm{de}}$, which is a typical assumption in the literature on DRE, e.g., Section 2.1 of \citet{JMLR:Kanamori+etal:2009}.
The goal of DRE is to estimate $r^*(x)=\frac{p_{\mathrm{nu}}(x)}{p_{\mathrm{de}}(x)}$ from the samples $\bf{X}^\mathrm{nu}$ and $\bf{X}^\mathrm{de}$.

\paragraph{Additional notation.}
Let $\mathbb{E}_{\mathrm{nu}}$ and $\mathbb{E}_{\mathrm{de}}$ denote the expectations with respect to $p_{\mathrm{nu}}(X)$ and $p_{\mathrm{de}}(X)$, respectively; in addition, $\hat{\mathbb{E}}_{\mathrm{nu}}$ and $\hat{\mathbb{E}}_{\mathrm{de}}$ denote the sample averages over $\big\{X^{\mathrm{nu}}_j\big\}^{n_\mathrm{nu}}_{j=1}$ and $\big\{X^{\mathrm{de}}_i\big\}^{n_\mathrm{de}}_{i=1}$, respectively.

\begin{table*}
\caption{Summary of DRE methods \citep{sugiyama2011bregman}. For PULogLoss, we use $C < \frac{1}{\overline{R}}$.}
\label{tbl:methods_from_br_minimization}
\begin{center}
\scalebox{0.87}[0.87]{
\begin{tabular}{llll}
\hline
Method & $f(t)$  & Lower bound of $\widehat{\mathrm{BD}}_f$ & Reference\\
\hline
LSIF & $(t-1)^2/2$ & Not bounded & \citet{JMLR:Kanamori+etal:2009}\\
Kernel Mean Matching & $(t-1)^2/2$ & Not bounded & \citet{gretton2009}\\
UKL & $t\log (t) - t$ & Not bounded & \citet{nguyen2010}\\
KLIEP & $t\log (t) - t$  & Not bounded & \citet{sugiyama2008}\\
BKL (LR) &  $t\log (t) - (1 + t)\log(1 + t)$ & Bounded & \citet{hastie_09_elements-of_statistical-learning}\\
PULogLoss  & $C\log\left(1-t\right) + Ct\left(\log\left(t\right)-\log\left(1-t\right)\right)$ for $0 < t < 1$ & Not bounded & \citet{kato2018learning}\\
\hline
\end{tabular}}
\end{center}
\end{table*}

\subsection{Density ratio matching by BD minimization}
Among existing DRE methods, we focus on \emph{density ratio matching through BD minimization} (DRM-BD; \citealp{sugiyama2011bregman}), which is a framework that unifies various DRE methods \citep{gretton2009,sugiyama2008,JMLR:Kanamori+etal:2009,nguyen2010}.

DRM-BD estimates the density ratio by minimizing the objective function derived as follows:
Let \(\rClassRangeTwo \subset [0, \infty)\), and let \(\br: \rClassRangeTwo \to \Re\) be a twice continuously differentiable convex function with a bounded derivative $\partial f$ (Table~\ref{tbl:methods_from_br_minimization}).
We quantify the discrepancy from the true density ratio function $r^*$ to a density ratio model $r$ by
\begin{multline}
\label{br_dr_relevant}
\mathrm{BD}_f(r^*\|r):=\mathbb{E}_{\mathrm{de}}\left[\partial f(r(X))r(X) - f(r(X))\right]\\
\qquad - \mathbb{E}_{\mathrm{nu}}\left[\partial f(r(X))\right]
\end{multline}
which is equal to the BD \citep{bregman1967relaxation} defined as $\mathbb{E}_{\mathrm{de}}[\ddot{\mathrm{BD}}_f(r^*(X)\|r(X))]$, where 
\begin{align*}
\ddot{\mathrm{BD}}_f(t^*\|t):=f(t^*) - f(t) - \partial f(t)(t^*-t),
\end{align*}
ignoring the constant $\overline{\mathrm{BD}}=\mathbb{E}_{\mathrm{de}}\big[f(r^*(X))\big]$. Then, given a hypothesis class $\mathcal{H}$, DRM-BD estimates $r^*$ using a minimizer of the sample analog of \eqref{br_dr_relevant}:
\begin{multline}
\label{sample_br}
\widehat{\mathrm{BD}}_f(r):=\hat{\mathbb{E}}_{\mathrm{de}}\Big[\partial f\big(r(X_i)\big)r(X_i) - f\big(r(X_i)\big)\Big]\\
\qquad- \hat{\mathbb{E}}_{\mathrm{nu}}\Big[\partial f\big(r(X_j)\big)\Big].
\end{multline}

\subsection{Examples of DRE}
\label{sec:example_dre}
\citet{sugiyama2011bregman} showed that BD minimization can unify various DRE methods. Furthermore, \citet{pmlr-v48-menon16} showed an equivalence between conditional probability estimation and DRE from the BD minimization perspective. In addition, by generalizing the results of \citet{ICML:duPlessis+etal:2015} and \citet{kato2018learning}, we derive a novel method for DRE from PU learning in Appendix~\ref{appdx:example_f}. We summarize the DRE methods in Table~\ref{tbl:methods_from_br_minimization}. Here, the empirical risks of \emph{least-square importance fitting} (LSIF), \emph{unnormalized Kullback–Leibler} (UKL) divergence, \emph{binary Kullback–Leibler} (BKL) divergence, and \emph{PU learning with log Loss} (PULogLoss) are given as 
\begin{align*}
&\widehat{\mathrm{BD}}_{\mathrm{LSIF}}(r) := \frac{1}{2} \hat{\mathbb{E}}_{\mathrm{de}}[r^2(X_i)] - \hat{\mathbb{E}}_{\mathrm{nu}}[r(X_j)],\\
&\widehat{\mathrm{BD}}_{\mathrm{UKL}}(r) := \hat{\mathbb{E}}_{\mathrm{de}}\left[r(X_i)\right] - \hat{\mathbb{E}}_{\mathrm{nu}} \left[\log\big(r(X_j)\big)\right],\\
&\widehat{\mathrm{BD}}_{\mathrm{BKL}}(r) := - \hat{\mathbb{E}}_{\mathrm{de}}\left[\mathrm{BKL1}(X_i)\right] - \hat{\mathbb{E}}_{\mathrm{nu}} \left[\mathrm{BKL2}(X_j)\right],\\
&\widehat{\mathrm{BD}}_{\mathrm{PU}}(r):= -\hat{\mathbb{E}}_{\mathrm{de}} \left[\log\big(1-r(X_i)\big)\right]\\
&\ \ \ \ \ \ \ \ \ \ \ \ \ \ \  + C\hat{\mathbb{E}}_{\mathrm{nu}} \left[-\log\big(r(X_j)\big)+ \log\big(1-r(X_j)\big)\right], 
\end{align*}
where $0 < C < \frac{1}{\overline{R}}$, $\overline{R}$ is an upper bound on $r^*$, $\mathrm{BKL1}(x) = \log\left(\frac{1}{1+r(x)}\right)$, and $\mathrm{BKL2}(x) = \log\left(\frac{r(x)}{1+r(x)}\right)$. Here, $\widehat{\mathrm{BD}}_{\mathrm{LSIF}}(r)$ and $\widehat{\mathrm{BD}}_{\mathrm{PU}}(r)$ correspond to LSIF and PULogLoss, respectively. We can derive the \emph{Kullback–Leibler importance estimation procedure} (KLIEP) and logistic regression-based DRE (LR) from $\widehat{\mathrm{BD}}_{\mathrm{UKL}}(r)$ and $\widehat{\mathrm{BD}}_{\mathrm{BKL}}(r)$. In $\widehat{\mathrm{BD}}_{\mathrm{PU}}(r)$, we restrict the model's output to be within $(0,1)$ and the estimated model $r$ becomes an estimator of $Cr^*$. Details of these methods are provided in Appendix~\ref{appdx:exist_DRE}.

\begin{figure}[t]
  \begin{center}
    \includegraphics[width=77mm]{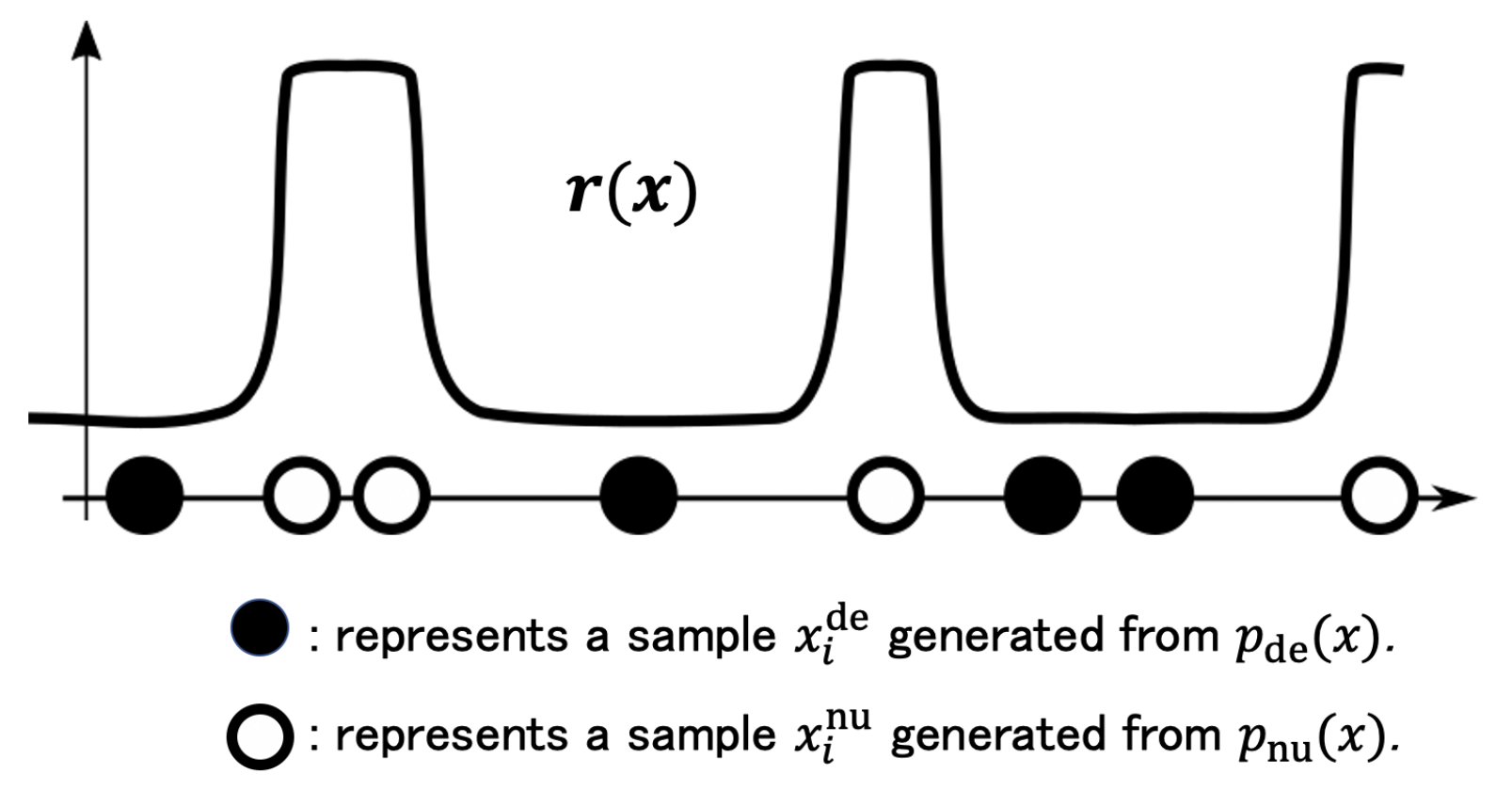}
  \end{center}
  \caption{Illustration of the train-loss hacking phenomenon. Given finite data points, a sufficiently flexible model $r$ can easily make the second term of the objective function $\widehat{\mathrm{BD}}_f$ largely negative, resulting in an unreasonable minimizer.
  }
  \label{fig1:box_plot}
\end{figure}

\subsection{Train-loss hacking problem}
\label{subsec:diff_dre}
DRM-BD with neural networks often suffers from an overfitting.
For instance, in Section~\ref{sec:num_exp}, we show that the LSIF with neural networks suffers from a serious overfitting issue. We conjecture that a conceivable cause of the overfitting in DRE-BD is the \emph{train-loss hacking}. This hypothesis is inspired by \citet{KiryoPositiveunlabeled2017}, which tackled a similar problem in PU learning; see Appendix~\ref{section: train loss hacking in PU} for a brief review.

Train-loss hacking is a phenomenon in which $r(X_j)$ increases to a large value for $\big\{X^{\mathrm{nu}}_j\big\}^{n_\mathrm{nu}}_{j=1} \iid p_{\mathrm{nu}}(x)$ when a model is trained to minimize an objective function consisting of multiple separate samples. Recall that, in the case of DRE-BD, the objective function \eqref{sample_br} consists of two empirical averages.
Among these two, we can make the second term $-\hat{\mathbb{E}}_{\mathrm{nu}}\big[\partial f\big(r(X_j)\big)\big]$ small by making the model to take as large values as possible at the specific input points $\bf{X}^\mathrm{nu}$ (Figure~\ref{fig1:box_plot}). In fact, since $f$ is a convex function, $\partial f$ is an increasing function; hence, this term monotonically decreases as $r$ increases in $\bf{X}^\mathrm{nu}$. As a result, when there is no lower bound on $-\hat{\mathbb{E}}_{\mathrm{nu}}\big[\partial f\big(r(X_j)\big)\big]$, it often numerically diverges to negative infinity\footnote{Even if we use a bounded model for $r$, it still tends to diverge during the numerical computation.}. Even when there is a lower bound on $-\hat{\mathbb{E}}_{\mathrm{nu}}\big[\partial f\big(r(X_j)\big)\big]$, $r(X_j)$ tends to take on the largest possible value of its output range at the points $\bf{X}^\mathrm{nu}$. Naively ``capping'' the model's output, e.g., composing $r(X)$ with $\max\{\cdot, B\}$ where \(B > 0\) is a constant, fails to remedy the issue as the model still tends to stick to its largest possible value. We experimentally demonstrate this by implementing the Bounded uLSIF (Figure~\ref{fig:numerical_exp1}).

This is a critical issue, since merely making the output large on $\bf{X}^\mathrm{nu}$ is unlikely to be a reasonable training criterion for DRE, and it may only lead to an unreasonable density ratio estimator (Figure~\ref{fig1:box_plot}).
The issue becomes salient when the model has a high flexibility.
If the hypothesis class has an extremely limited flexibility, this may not be an issue since the remaining term $\hat{\mathbb{E}}_{\mathrm{de}}\Big[\partial f\big(r(X_i)\big)r(X_i) - f\big(r(X_i)\big)\Big]$ is likely to introduce a trade-off.
However, when highly flexible models such as deep neural networks are employed, the model can easily fit to $\bf{X}^\mathrm{nu}$ and $\bf{X}^\mathrm{de}$ separately (Figure~\ref{fig1:box_plot}).






\section{Deep direct DRE based on non-negative risk estimator}
\label{sec:D3RE}
Although DRE with flexible models suffers from serious train-loss hacking, we still have a strong motivation to use them for applications, such as CV and NLP. In this section, we describe our approach to modify the DRM-BD objective function to mitigate the train-loss hacking problem.

\subsection{Non-negative BD}
To alleviate the train-loss hacking problem, we propose a \emph{non-negative BD estimator} that modifies an empirical BD estimator \eqref{sample_br} to be robust against the problem. The proposed method is inspired by \citet{KiryoPositiveunlabeled2017}, which suggested a non-negative correction to the empirical risk of PU learning based on the knowledge that a part of the population risk is non-negative. However, in DRE, it is not straightforward to employ this approach because we do not know which part of the population risk \eqref{br_dr_relevant} is non-negative. In this paper, by assuming an upper bound $\overline{R}$ on the density ratio $r^*$, we detect which part of the risk of DRE \eqref{br_dr_relevant} is non-negative in the population. Then, we apply a non-negative correction to an empirical BD estimator \eqref{sample_br} based on the non-negativity of the corresponding part of the population risk. This non-negative correction also corresponds to a generalization of non-negative PU learning \citep{KiryoPositiveunlabeled2017}.

To enable our approach to mitigate the train-loss hacking, we apply the following assumption:
\begin{assumption}\label{assumption: bounded density ratio}
The density ratio $r^*$ is bounded from above,
i.e., $\overline{R} = \sup_{X \in \mathcal{X}^{\mathrm{de}}}r^*(X) < \infty$.
\end{assumption}
Then, we arbitrarily specify a constant $C$ such that $0 < C < \frac{1}{\overline{R}}$.
Using $C$, we make the following assumption.
\begin{assumption}\label{assumption:tilde-f2} A function
$\tilde{f}$ defined by
\begin{align}
\label{eq:tilde_f}
\partial f (t) = C \big(\partial f(t)t - f(t)\big)+\tilde f(t)
\end{align}
is bounded from above.
\end{assumption}
Then, we rewrite the DRM-BD objective \eqref{br_dr_relevant} as
\begin{multline}
\label{eq:decomp_BR}
\mathrm{BD}_f(r^*\|r)=\annot{\mathbb{E}_{\mathrm{de}}[\lossOne(r(X))] - C\mathbb{E}_{\mathrm{nu}}[\lossOne(r(X))]}{($\ast$)}\\
\qquad\qquad\qquad + \mathbb{E}_{\mathrm{nu}}\lossTwo(r(X)) - (1-C)A,
\end{multline}
where $\lossOne$ and $\lossTwo$ are
\begin{align*}
\lossOne(t) &:= \dbr(t) t - \br(t) + A,\quad
\lossTwo(t) := - \tbr(t),
\end{align*}
and $A$ is a constant such that $\lossOne(t) \geq 0$ for all $t\in (b_r, B_r)$.
Now, we make the following observation.
\paragraph{Observation.} The ($\ast$) part in \eqref{eq:decomp_BR} is non-negative since both $\lossOne$ and $p_{\mathrm{de}}-C p_{\mathrm{nu}}$ are non-negative under Assumption~\ref{assumption: bounded density ratio} and $0 < C < \frac{1}{\overline{R}}$.

Based on this observation, we propose using the following modified empirical risk:
\begin{multline}
\label{eq:nnBD}
\widehat{\mathrm{nnBD}}_f(r) :=
\left(\hat{\mathbb{E}}_{\mathrm{de}}\big[\lossOne(r(X_i))\big] - C \hat{\mathbb{E}}_{\mathrm{nu}}\big[\lossOne(r(X_j))\big]\right)_+\\
\quad+\hat{\mathbb{E}}_{\mathrm{nu}}\big[\lossTwo(r(X_j))\big]
\end{multline}
where $(\cdot)_+ := \max\{0, \cdot\}$. Note that the nonnegativity of ($\ast$) is always satisfied in the population quantity; however, it can be violated in finite samples, allowing for train-loss hacking.
Our \emph{deep direct DRE} (D3RE) is based on minimizing $\widehat{\mathrm{nnBD}}_f(r)$
over a hypothesis class of the density ratio \(\rClass \subset \{r: \Re^d \to \rClassRangeTwo\}\), where \(0 \leq \rClassBoundMin < \rmax < \rClassBound\).

\paragraph{Instantiations of the D3RE objective functions.} The above strategy can be instantiated with various functions $f$ proposed for DRM-BD. Here, we introduce nnBD corresponding to LSIF, UKL, BKL, and PULogLoss as follows:
\begin{align*}
&\widehat{\mathrm{nnBD}}_{\mathrm{LSIF}}(r):=-\hat{\mathbb{E}}_{\mathrm{nu}}\left[r(X_j)-\frac{C}{2}r^2(X_j)\right]\\
&\ \ \ \ \ \ \ \ \ \ \ \ \ \ \ \ \ \ \ \ \ \ +\left(\frac{1}{2}\hat{\mathbb{E}}_{\mathrm{de}}\left[r^2(X_i)\right] - \frac{C}{2}\hat{\mathbb{E}}_{\mathrm{nu}} \left[r^2(X_j)\right]\right)_+,\ \quad\\
&\widehat{\mathrm{nnBD}}_{\mathrm{UKL}}(r):= - \hat{\mathbb{E}}_{\mathrm{nu}} \left[\log\big(r(X_j)\big) - Cr(X_j)\right]\\
&\ \ \ \ \ \ \ \ \ \ \ \ \ \ \ \ \ \ \ \ \ \ \ \ \ \ \ \ \ \ \ \ \ +  \left(\hat{\mathbb{E}}_{\mathrm{de}}\left[r(X_i)\right] - C\hat{\mathbb{E}}_{\mathrm{nu}}\left[r(X_j)\right]\right)_+,\nonumber\\
&\widehat{\mathrm{nnBD}}_{\mathrm{BKL}}(r):= - \hat{\mathbb{E}}_{\mathrm{nu}} \left[\mathrm{BKL2}(X_j) + C\mathrm{BKL1}(X_j)\right]\nonumber\\
&\ \ \ \ \ \ \ \ \ + \Bigg(-\hat{\mathbb{E}}_{\mathrm{de}}\left[\mathrm{BKL1}(X_i)\right]+ C\hat{\mathbb{E}}_{\mathrm{nu}}\left[\mathrm{BKL1}(X_j)\right]\Bigg)_+,\\
&\widehat{\mathrm{nnBD}}_{\mathrm{PU}}(r):=-C\hat{\mathbb{E}}_{\mathrm{nu}}\left[\log\big(r(X_j)\big)\right]\\
&+\left(C\hat{\mathbb{E}}_{\mathrm{nu}}\left[ \log\big(1-r(X_j)\big)\right] - \hat{\mathbb{E}}_{\mathrm{de}}\left[\log\big(1-r(X_i)\right]\right)_+.\nonumber
\end{align*}
More detailed derivation of $\tilde f$ is in Appendix~\ref{appdx:example_f}. 

The algorithm for D3RE is described in Algorithm~\ref{alg}. For training with a large amount of data, we adopt a stochastic optimization by splitting the dataset into mini-batches.
In stochastic optimization, we separate the samples into $N$ mini-batches as
$(\big\{X^{\mathrm{nu}}_i\big\}^{n_{\mathrm{nu},j} }_{i=1}, \big\{X^{\mathrm{de}}_i\big\}^{n_{\mathrm{de}, j}}_{i=1})$ ($j = 1, \ldots, N$,
where $n_{\mathrm{nu}, j}$ and $n_{\mathrm{de}, j}$ are the sample sizes for each mini-batch. Then, we consider the sample average in each mini-batch. Let $\hat{\mathbb{E}}^j_{\mathrm{nu}}$ and $\hat{\mathbb{E}}^j_{\mathrm{de}}$ be sample averages over $\big\{X^{\mathrm{nu}}_i\big\}^{n_{\mathrm{nu},j} }_{i=1}$ and $\big\{X^{\mathrm{de}}_i\big\}^{n_{\mathrm{de}, j}}_{i=1}$. In addition, we use regularization, such as L1 and L2 penalties, as denoted by $\mathcal{R}(r)$.

To improve the performance, we can heuristically employ \emph{gradient ascent} from \citet{KiryoPositiveunlabeled2017} when $\hat{\mathbb{E}}_{\mathrm{de}}\big[\lossOne(r(X))\big] - C \hat{\mathbb{E}}_{\mathrm{nu}}\big[\lossOne(r(X))\big]$ becomes less than $0$, i.e., the model is updated in the direction that increases the term. Note that gradient ascent is not essential in D3RE, and we can obtain similar results even without it (see experiments in Appendix~\ref{appdx:res_wo_ga}). We recommend practitioners to use a gradient ascent and those concerned with a theoretical guarantee to use a plain gradient descent.

\paragraph{Choice of $C$.}
Although we use an upper bound of the density ratio in the formulation, we do not require a tight one. The main role of the upper bound is to prevent the density ratio model from diverging, and as long as we successfully prevent divergence, the proposed algorithms work well. We find that D3RE is robust against a loose specification of the upper bound to a certain extent in our experiment (the right graph in Figure~\ref{fig:numerical_exp1} of Section~\ref{sec:num_exp}). Thus, in practice, selecting $C$ does not require accurate knowledge of $\overline{R}$. In fact, in inlier-based outlier detection experiments, the proposed methods under a loose specification of the upper bound achieve a preferable performance. However, selecting a hyper-parameter $C$ that is much smaller than $1/\overline{R}$ may damage the empirical performance as shown in Section~\ref{sec:num_exp}. This, of course, does not mean that $1/\overline{R}$ should not be small; if $1/\overline{R}$ is small, $C$ can also be small. 

\paragraph{Non-negative PU learning.} \citet{KiryoPositiveunlabeled2017} proposed a non-negative correction for PU learning (nnPU). In this paper, we propose a non-negative correction for DRE, inspired by \citet{KiryoPositiveunlabeled2017}; however, our extension is nontrivial because the relationship between DRE and PU learning has not been well understood. Another contribution of this paper is that it clarifies the relationship between DRE and PU learning, as described in Appendix~\ref{appdx:exist_DRE} and Section~\ref{sec:experiments}. We find that the class-prior in PU learning corresponds to the upper bound of $r^*$ in DRE, and that as \citet{Sugiyama:2012:DRE:2181148} generalized DRE in terms of BD divergence minimization, the risk of PU learning can also be generalized through BD divergence minimization. This finding had been implied by \citet{kato2018learning}, although it had not been formally shown. This finding clarifies the relationship between DRE and PU learning, and thus makes it possible to apply the non-negative correction to DRE, such as nnPU.

\begin{algorithm}[tb]
   \caption{D3RE}
   \label{alg}
\begin{algorithmic}
   \STATE {\bfseries Input:} Training data $\big\{X^{\mathrm{nu}}_j\big\}^{n_\mathrm{nu}}_{j=1}$ and $\big\{X^{\mathrm{de}}_i\big\}^{n_\mathrm{de}}_{i=1}$, the algorithm for stochastic optimization such as Adam \citep{kingma2014method}, the learning rate $\gamma$, the regularization coefficient $\lambda$ and function $\mathcal{R}(r)$, and a constant $C > 0$. 
   \STATE {\bfseries Output:} A density ratio estimator $\hat{r}$.
   \WHILE{No stopping criterion has been met:}
   \STATE Create $N$ mini-batches $$\left\{\left(\big\{X^{\mathrm{nu}}_j\big\}^{n_{\mathrm{nu},k} }_{j=1}, \big\{X^{\mathrm{de}}_i\big\}^{n_{\mathrm{de}, k}}_{i=1}\right)\right\}^N_{k=1}.$$
   \FOR{$k=1$ to $N$}
   \IF{$\hat{\mathbb{E}}^k_{\mathrm{de}}\big[\lossOne(r(X))\big] - C \hat{\mathbb{E}}^k_{\mathrm{nu}}\big[\lossOne(r(X))\big] \geq 0$:}
   \STATE{Gradient decent:} set gradient \begin{align*}&\nabla_r \big\{\hat{\mathbb{E}}^k_{\mathrm{nu}}\big[\lossTwo(r(X))\big]+\hat{\mathbb{E}}^k_{\mathrm{de}}\big[\lossOne(r(X))\big]\\
   &\ \ \ \ \ \ \ \ \ \ \ \ \ \ \ \ \ \ \ \ - C \hat{\mathbb{E}}^k_{\mathrm{nu}}\big[\lossOne(r(X))\big] + \lambda \mathcal{R}(r)\big\}.\end{align*}
   \ELSE 
   \STATE{Gradient ascent:} set gradient $$\nabla_r \big\{ - \hat{\mathbb{E}}^k_{\mathrm{de}}\big[\lossOne(r(X))\big] + C \hat{\mathbb{E}}^k_{\mathrm{nu}}\big[\lossOne(r(X))\big] + \lambda \mathcal{R}(r)\big\}.$$
   \ENDIF
   \STATE Update $r$ with the gradient and the learning rate $\gamma$. 
   \ENDFOR
   \ENDWHILE
\end{algorithmic}
\end{algorithm} 

\subsection{Motivation and intuitive justification of D3RE}
\label{sec:motive}
Here, we describe how the above non-negative risk correction alleviates the train-loss hacking problem.

\paragraph{D3RE and an unbounded empirical risk.} First, we consider the case where the empirical BD is unbounded. First, we assume the following on $\tilde f(t)$ of \eqref{eq:tilde_f}.
Assumption~\ref{assumption:tilde-f2} is satisfied by most of the loss functions which appear in the previously proposed DRE methods (see Appendix~\ref{appdx:example_f} for examples). Under Assumption~\ref{assumption:tilde-f2}, because $\tilde f(t)$ is bounded above, the train-loss hacking  $-\hat{\mathbb{E}}_{\mathrm{nu}}\Big[\partial f\big(r(X_j)\big)\Big]\to-\infty$ to minimize the empirical risk \eqref{sample_br} is caused by 
\begin{align*}
-\hat{\mathbb{E}}_{\mathrm{nu}}\Big[ C\left\{\partial f(r(X_j))r(X_j) -  f(r(X_j))\right\}\Big]\to -\infty
\end{align*}
because 
\begin{multline*}
\annot{-\hat{\mathbb{E}}_{\mathrm{nu}}\Big[\partial f\big(r(X_j)\big)\Big]}{$\rightarrow-\infty$} = \annot{
-\hat{\mathbb{E}}_{\mathrm{nu}}\Big[\tilde{f}(r(X_j))\Big]}{Bounded}\\
\annot{-\hat{\mathbb{E}}_{\mathrm{nu}}\Big[ C\left\{\partial f(r(X_j))r(X_j) -  f(r(X_j))\right\}\Big]}{$\rightarrow-\infty$}.
\end{multline*}
This observation implies that our non-negative correction \eqref{eq:nnBD}
prevents train-loss hacking by effectively introducing the correction to the problematic term ($\ast$ in \eqref{eq:decomp_BR}).

\paragraph{D3RE and a bounded empirical risk.} Next, we consider the case where $\partial f\big(r(X_j)$ or model $r(x)$ is bounded. Even in these cases, train-loss hacking can occur. For instance, if $\partial f (t) = \log(t) - \log (1+t)$ (BKL), $\partial f(t)$ is upper-bounded by $0$, and $-\hat{\mathbb{E}}_{\mathrm{nu}}\big[\partial f\big(r(X_j)\big)\big]$ does not diverge to $-\infty$. However, we can infinitely decrease $-\hat{\mathbb{E}}_{\mathrm{nu}}\big[\partial f\big(r(X_j)\big)\big]$ to $0$ by making $r(X_j)\to \infty$, which causes train-loss hacking. On the other hand, when $r(x)$ is upper-bounded, we can minimize $-\hat{\mathbb{E}}_{\mathrm{nu}}\big[\partial f\big(r(X_j)\big)\big]$ by training $r(X_j)$ to stick to the upper bound at $\big\{X^{\mathrm{nu}}_i\big\}^{n_\mathrm{nu}}_{i=1}$. Therefore, the upper-bounding $\partial f\big(r(X_j)$ or model $r(x)$ does not solve the train-loss hacking. However, for these cases, the proposed non-negative risk correction approach is empirically shown to be effective, as shown in  Figure~\ref{fig:numerical_exp1} of Section~\ref{sec:num_exp} and Figure~\ref{fig:numerical_exp2} of Appendix~\ref{sec:exp_val_estimators}. In these results, Bounded LSIF and BKL correspond to the upper bounding of the model $r(x)$ and $\partial f\big(r(X_j)$, respectively. Experimentally, DRE methods without the non-negative correction fail to learn the density ratio, while the non-negative correction succeeded in stabilizing the performance.


\section{Theoretical justification of D3RE}
\label{sec:theoretical_analysis}
In this section, we confirm the validity of D3RE by providing a generalization error bound. We derive two types of guarantees, one in terms of the BD risk and the other the $L^2$-distance. 
Given \(n \in \mathbb{N}\) and a distribution \(p\), we define the \emph{Rademacher complexity} \(\Radnp\) of a function class \(\rClass\) as $\Radnp(\rClass) := \E_p\ERad\left[\supr\left|\frac{1}{n} \sum_{i=1}^n \rad_i \r(\X_i)\right|\right]$, where \(\{\sigma_i\}_{i=1}^n\) are independent uniform sign variables and $\{X_i\}_{i=1}^n \iid p$. We omit \(\rstar\) from the notation \(\Risk\) when there is no ambiguity.

\subsection{Generalization error bound on BD}
Let \(\rClassRangeTwoName := \rClassRangeTwo\). Theorem~\ref{thm:main-text:estimation-error-bound} in Appendix~\ref{section:appendix:estimation-error-bound} provides a generalization error bound in terms of the Rademacher complexities of a hypothesis class and the following assumption. 

\begin{assumption}
Following (i)--(iv) hold:
\begin{description}
\item[(i)] there exists an empirical risk minimizer \(\hr \in \argmin_{\r \in \rClass} \nnhRisk(\r)\) and a population risk minimizer \(\rbest \in \argmin_{\r \in \rClass} \Risk(\r)\);
\item[(ii)] \(\LossBound := \supt\{\max\{|\lossOne(t)|, |\lossTwo(t)|\}\} < \infty\);
\item[(iii)] \(\lossOne\) (resp. \(\lossTwo\)) is  \(\LipOne\)-Lipschitz (resp. \(\LipTwo\)-Lipschitz) on \(\rClassRangeTwoName\);
\item[(iv)] \(\infr (\Ede - \Cons\Enu)\lOneR > 0\).
\end{description}
\label{assumption:main:est-error-bound}
\end{assumption}
For the boundedness and Lipschitz continuity in Assumption~\ref{assumption:main:est-error-bound} to hold for the loss functions involving a logarithm (UKL, BKL, PU), a technical assumption \(\rClassBoundMin > 0\) is sufficient. 

Then, we introduce Assumption~\ref{assumption:main:rademacher-bound-nn} \citep[Theorem~1]{GolowichSizeIndependent2019} to bound the complexity of the hypothesis class.
\begin{assumption}[Neural networks with bounded complexity]\label{assumption:main:rademacher-bound-nn} The probability densities \(\pnu\) and \(\pde\) have bounded supports: \(\sup_{x \in \mathcal{X}^{\mathrm{de}}} \|x\| < \infty\), and a hypothesis class \(\rClass\) consists of real-valued neural networks of depth \(L\) over the domain \(\mathcal{X}\),
where each parameter matrix \(W_j\) has the Frobenius norm at most \(\ParamBoundj \geq 0\) and \(1\)-Lipschitz activation functions \(\varphi_j\) that are positive-homogeneous (i.e., \(\varphi_j\) is applied element-wise and \(\varphi_j(\alpha t) = \alpha \varphi_j(t)\) for all \(\alpha \geq 0\)).
\end{assumption}
Under Assumption~\ref{assumption:main:rademacher-bound-nn}, Lemma~\ref{lem:main-text:rademacher-bound-golowich} in Appendix~\ref{appdx:rad_comp_bound} reveals \(\Radnu(\rClass) = \mathcal{O}(1/\sqrt{\nnu})\) and \(\Radde(\rClass) = \mathcal{O}(1/\sqrt{\nde})\). By combining these results with Theorem~\ref{thm:main-text:estimation-error-bound} in Appendix~\ref{section:appendix:estimation-error-bound}, we obtain the following theorem.
\begin{theorem}[Generalization error bound for D3RE]
\label{thm:est_error_bound}
Under Assumptions~\ref{assumption:main:est-error-bound} and \ref{assumption:main:rademacher-bound-nn}, for any \(\delta \in (0, 1)\), we have with probability at least \(1 - \delta\), 
\begin{equation*}\begin{aligned}
&\Risk(\hr) - \Risk(\rbest)\leq \frac{\kappa_1}{\sqrt{\nde}} + \frac{\kappa_2}{\sqrt{\nnu}}\\
&\ \ \ + 2(1+\Cons)\LossBound \exp\left(- \frac{2 \alpha^2}{(\LossBound^2/\nde) + (\Cons^2\LossBound^2 / \nnu)}\right)\\
&\ \ \ + \LossBound \sqrt{8 \left(\frac{1}{\nde}+ \frac{(1 + \Cons)^2}{\nnu}\right) \log\frac{1}{\delta}},
\end{aligned}\end{equation*}
where \(\kappa_1, \kappa_2\) are constants that depend on \(\Cons, \br, B_{p_{\mathrm{de}}}, B_{p_{\mathrm{nu}}}, L\), and \(\ParamBoundj\).
\label{thm:main:d3re}
\end{theorem}
See Remark~\ref{rem:explicit} in Appendix~\ref{section:appendix:estimation-error-bound} for more explicit forms of $\kappa_1$ and $\kappa_2$. This generalization error bound provides theoretical guarantees for various applications. For instance, by defining $f(t) = \log\left(1-t\right) + Ct\left(\log\left(t\right)-\log\left(1-t\right)\right)$ for $0 < t < 1$, $\Risk(r)$ becomes the risk functional of PU learning (see Appendix~\ref{appdx:exist_DRE} for the derivation). Then, the generalization error bound provides a classification error bound for PU learning, which is a special case of the binary classification problem \citet{KiryoPositiveunlabeled2017}. Note that the dependency of the bound on \(\br, B_{p_{\mathrm{de}}}, B_{p_{\mathrm{nu}}}, L\), and \(\ParamBoundj\) is typical for classification with Lipschitz functions \citep[Corollary~15]{bartlett2003rademacher}. The third term of the RHS corresponds to the bias caused by the use of non-negative correction.

\begin{figure*}[t]
\vspace{-0.cm}
\begin{center}
 \includegraphics[keepaspectratio,width=170mm]{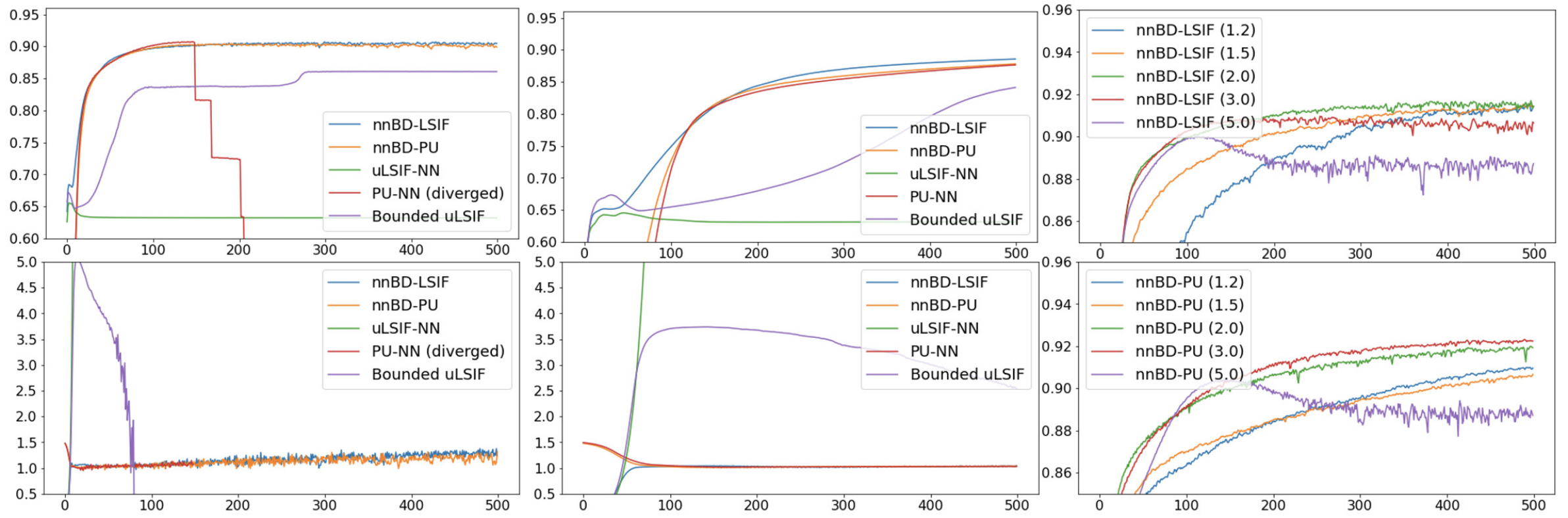}
\end{center}
\vspace{-0.35cm}
\caption{Results of Section~\ref{sec:num_exp}. The horizontal axes represent epochs. Left and center figures: The results under different learning rates, $1\times10^{-4}$ and $1\times10^{-5}$, respectively, where the vertical axes of the upper graphs show the AUROCs, and those of the lower graphs show $\hat{\mathbb{E}}_{\mathrm{de}}[\hat{r}(X_i)]$. Right figure: Results of sensitivity analyses, where the vertical axes show the AUROCs.}
\label{fig:numerical_exp1}
\end{figure*}

\subsection{Estimation error bound on \(\Ltwo\) norm}
\label{sec:appendix:strong-convexity}
Next, we derive an estimation error bound for $\hat{r}$ on the \(\Ltwo\) norm. We aim to derive the standard convergence rate of non-parametric regression; that is, under the appropriate conditions, the order of $\|\hr - \rstar\|_{\Ltwo(\pde)}$ is nearly $\Orderp{1/(\nde\land\nnu)}$ \citep{KanamoriStatistical2012a}. Note that unlike the generalization error bound of the BD, we require a stronger assumption on the loss function, namely, a strong convexity.
In Theorem~\ref{thm:est_error_bound}, for a multilayer perception with ReLU activation function (Definition~\ref{appdx:sparse-network-function-class}), we derive the convergence rate of the \(\Ltwo\) distance, which is the same rate as that of the nonparametric regression using the Gaussian kernel and the LSIF loss \citep{KanamoriStatistical2012a}. This result also corresponds to a faster convergence rate than Theorem~\ref{thm:est_error_bound}.
The proof is shown in Appendix~\ref{appdx:l2norm}. To complement this result, we empirically investigate the estimator error using an artificially generated dataset with the known true density ratio in Section~\ref{sec:regression_exp}. 
\begin{theorem}[\(L^2\) Convergence rate]
  \label{appdx:thm:convergence-rate}
   Assume \(f\) is \(\mu\)-strongly convex.
  Let \(\rClass\) be defined as in
  Definition~\ref{appdx:sparse-network-function-class} and assume \(\rstar = \frac{\pnu}{\pde} \in \rClass\).
  In addition, assume the same conditions as Theorem~\ref{thm:estimation-error-bound}.
  Then, for any \(0 < \gamma < 2\), as $\nde, \nnu \to \infty$,
  \begin{align*}
  \|\hr - \rstar\|_{\Ltwo(\pde)} \leq \Orderp{(\vmin{\nde}{\nnu})^{-1/(2+\gamma)}}.
  \end{align*}
\end{theorem}

This $L^2$ distance bound is useful for statistical inference. For example, double/debiased machine learning with cross-fitting proposed by \citet{chernozhukov2016} allows for semiparametric inference under estimators of nuisance parameters with appropriate convergence rates. By using cross-fitting, \citet{kato_uehara_2020} proposed causal inference under covariate shifts when the convergence rate of the density ratio satisfies an appropriate convergence rate. Further, it is expected to be applied to two-sample homogeneity using neural networks, as in  \citet{kanamori2010}.

\section{Experiments}
\label{sec:experiments}
We experimentally show how existing DRE methods fail and D3RE succeeds when using neural networks\footnote{A code of the conducted experiments is available at \url{https://github.com/MasaKat0/D3RE}.}. 

\begin{table*}[t]
\caption{Results of Section~\ref{sec:regression_exp}: MSEs and SDs of DRE using synthetic datasets. The lowest MSE methods are highlighted in bold.}
\label{tbl:synthetic}
\begin{center}
\scalebox{0.87}[0.87]{
\begin{tabular}{l|l|r|r|rrrrrrrrr}
\hline
{} & {} &    \multirow{2}{*}{uLSIF}   & \multirow{2}{*}{LSIF-NN} &  \multicolumn{9}{c}{D3RE (nnBD-LSIF)} \\
{} & {} &   &      &      $C=0.8$ &      $C=1$ &      $C=2$ &      $C=3$ &      $C=4$ &      $C=5$ &      $C=10$ &      $C=15$ &      $C=20$ \\
\hline
\multirow{2}{*}{$\mathrm{dim}=10$} & MSE &  2.378 &  1.272 &  1.750 &  1.695 &  1.191 &  0.964 &  0.873 &  \textbf{0.833} &  0.948 &  1.079 &   1.170 \\
{} & SD &  1.143 &  0.413 &  0.570 &  0.563 &  0.523 &  0.487 &  0.459 &  0.424 &  0.370 &  0.331 &   0.387 \\
\hline
\multirow{2}{*}{$\mathrm{dim}=20$} & MSE &  1.684 &  2.694 &  1.704 &  1.646 &  1.307 &  \textbf{1.272} &  1.337 &  1.444 &  2.066 &  2.697 &   3.098 \\
{} & SD &  0.372 &  0.409 &  0.380 &  0.368 &  0.328 &  0.297 &  0.283 &  0.288 &  0.285 &  0.346 &   0.374 \\
\hline
\multirow{2}{*}{$\mathrm{dim}=30$} & MSE &  1.786 &  3.724 &  1.811 &  1.747 &  \textbf{1.488} &  1.577 &  1.798 &  2.019 &  3.238 &  4.306 &   5.432 \\
{} & SD &  0.456 &  0.460 &  0.459 &  0.449 &  0.411 &  0.400 &  0.401 &  0.379 &  0.370 &  0.464 &   0.543 \\
\hline
\multirow{2}{*}{$\mathrm{dim}=50$} & MSE &  1.791 &  8.717 &  1.817 &  1.753 &  \textbf{1.609} &  1.818 &  2.194 &  2.614 &  4.848 &  6.955 &   8.798 \\
{} & SD &  0.562 &  1.518 &  0.571 &  0.555 &  0.513 &  0.503 &  0.484 &  0.465 &  0.488 &  0.597 &   0.672 \\
\hline
\multirow{2}{*}{$\mathrm{dim}=100$} & MSE &  1.723 &  4.849 &  1.748 &  1.693 &  \textbf{1.626} &  1.860 &  2.226 &  2.709 &  5.528 &  8.605 &  11.557 \\
{} & SD &  0.574 &  4.182 &  0.575 &  0.571 &  0.540 &  0.532 &  0.495 &  0.563 &  0.672 &  0.790 &   1.140 \\
\hline
\end{tabular}
}
\end{center}
\vspace{-0.3cm}
\end{table*}

\begin{table*}
\caption{Average AUROC curve (Mean) with the standard deviation (SD) over $5$ trials of anomaly detection
methods. For all datasets, each model was trained on a single class and tested against all other classes. The best result is in bold.}
\label{tbl:exp_anomaly_detection}
\begin{center}
\scalebox{0.87}[0.87]{
\begin{tabular}{l|rr|rr|rr|rr|rr|rr|rr}
\toprule
{\tt MNIST} & \multicolumn{2}{c|}{uLSIF-NN} & \multicolumn{2}{c|}{nnBD-LSIF} & \multicolumn{2}{c|}{nnBD-PU} & \multicolumn{2}{c|}{nnBD-LSIF} & \multicolumn{2}{c|}{nnBD-PU} & \multicolumn{2}{c|}{Deep SAD} & \multicolumn{2}{c}{GT}\\
Network & \multicolumn{2}{c|}{LeNet} & \multicolumn{2}{c|}{LeNet} & \multicolumn{2}{c|}{LeNet} & \multicolumn{2}{c|}{WRN} & \multicolumn{2}{c|}{WRN} & \multicolumn{2}{c|}{LeNet} & \multicolumn{2}{c}{WRN}\\
\hline
Inlier Class &             Mean &    SD &      Mean &      SD &      Mean &      SD  &      Mean &      SD &      Mean &      SD &      Mean &      SD &      Mean &      SD\\
\hline
0 &  0.999 &  0.000 &  0.997 &  0.000 &  0.999 &  0.000 &  {\bf 1.000} &  0.000 &  {\bf 1.000} &  0.000 &  0.592 &  0.051 & 0.963 &  0.002 \\
1 &  {\bf 1.000} &  0.000 &  0.999 &  0.000 &  {\bf 1.000} &  0.000 &  {\bf 1.000} &  0.000 &  {\bf 1.000} &  0.000 &  0.942 &  0.016 & 0.517 &  0.039 \\
2 &  0.997 &  0.001 &  0.994 &  0.000 &  0.997 &  0.001 &  {\bf 1.000} &  0.000 &  {\bf 1.000} &  0.001 &  0.447 &  0.027 & 0.992 &  0.001 \\
3 &  0.997 &  0.000 &  0.995 &  0.001 &  0.998 &  0.000 &  {\bf 1.000} &  0.000 &  {\bf 1.000} &  0.000 &  0.562 &  0.035 & 0.974 &  0.001 \\
4 &  0.998 &  0.000 &  0.997 &  0.001 &  0.999 &  0.000 &  {\bf 1.000} &  0.000 &  {\bf 1.000} &  0.000 &  0.646 &  0.015 & 0.989 &  0.001 \\
\bottomrule
\end{tabular}}
\scalebox{0.87}[0.87]{
\begin{tabular}{l|rr|rr|rr|rr|rr|rr|rr}
\toprule
{\tt CIFAR-10} & \multicolumn{2}{c|}{uLSIF-NN} & \multicolumn{2}{c|}{nnBD-LSIF} & \multicolumn{2}{c|}{nnBD-PU} & \multicolumn{2}{c|}{nnBD-LSIF} & \multicolumn{2}{c|}{nnBD-PU} & \multicolumn{2}{c|}{Deep SAD} & \multicolumn{2}{c}{GT}\\
Network & \multicolumn{2}{c|}{LeNet} & \multicolumn{2}{c|}{LeNet} & \multicolumn{2}{c|}{LeNet} & \multicolumn{2}{c|}{WRN} & \multicolumn{2}{c|}{WRN} & \multicolumn{2}{c|}{LeNet} & \multicolumn{2}{c}{WRN}\\
\hline
Inlier Class &             Mean &    SD &      Mean &      SD &      Mean &      SD  &      Mean &      SD &      Mean &      SD &      Mean &      SD &      Mean &      SD\\
\hline
plane &  0.745 &  0.056 &  0.934 &  0.002 &  {\bf 0.943} &  0.001 &  0.925 &  0.004 &  0.923 &  0.001 &  0.627 &  0.066 &  0.697 &  0.009 \\
car &  0.758 &  0.078 &  0.957 &  0.002 &  {\bf 0.968} &  0.001 &  0.965 &  0.002 &  0.960 &  0.001 &  0.606 &  0.018 &  0.962 &  0.003 \\
bird &  0.768 &  0.012 &  0.850 &  0.007 &  {\bf 0.878} &  0.004 &  0.844 &  0.004 &  0.858 &  0.004 &  0.404 &  0.006 &  0.752 &  0.002 \\
cat &  0.745 &  0.037 &  0.820 &  0.003 &  {\bf 0.856} &  0.002 &  0.810 &  0.009 &  0.841 &  0.002 &  0.517 &  0.018 &  0.727 &  0.014 \\
deer &  0.758 &  0.036 &  0.886 &  0.004 &  {\bf 0.909} &  0.002 &  0.864 &  0.008 &  0.872 &  0.002 &  0.704 &  0.052 &  0.863 &  0.014 \\
\bottomrule
\end{tabular}}
\scalebox{0.87}[0.87]{
\begin{tabular}{l|rr|rr|rr|rr|rr|rr|rr}
\toprule
{\tt FMNIST} & \multicolumn{2}{c|}{uLSIF-NN} & \multicolumn{2}{c|}{nnBD-LSIF} & \multicolumn{2}{c|}{nnBD-PU} & \multicolumn{2}{c|}{nnBD-LSIF} & \multicolumn{2}{c|}{nnBD-PU} & \multicolumn{2}{c|}{Deep SAD} & \multicolumn{2}{c}{GT}\\
Network & \multicolumn{2}{c|}{LeNet} & \multicolumn{2}{c|}{LeNet} & \multicolumn{2}{c|}{LeNet} & \multicolumn{2}{c|}{WRN} & \multicolumn{2}{c|}{WRN} & \multicolumn{2}{c|}{LeNet} & \multicolumn{2}{c}{WRN}\\
\hline
Inlier Class &             Mean &    SD &      Mean &      SD &      Mean &      SD  &      Mean &      SD &      Mean &      SD &      Mean &      SD &      Mean &      SD\\
\hline
T-shirt/top &  0.960 &  0.005 &  0.981 &  0.001 &  {\bf 0.985} &  0.000 &  0.984 &  0.001 &  0.982 &  0.000 &  0.558 &  0.031 &  0.890 &  0.007 \\
Trouser &  0.961 &  0.010 &  0.998 &  0.000 &  {\bf 1.000} &  0.000 &  0.998 &  0.000 &  0.998 &  0.000 &  0.758 &  0.022 &  0.974 &  0.004 \\
Pullover &  0.944 &  0.012 &  0.976 &  0.001 &  0.980 &  0.001 &  {\bf 0.983} &  0.002 &  0.972 &  0.001 &  0.617 &  0.046 &  0.902 &  0.005 \\
Dress &  0.973 &  0.006 &  0.986 &  0.001 &  {\bf 0.992} &  0.000 &  0.991 &  0.001 &  0.986 &  0.000 &  0.525 &  0.038 &  0.843 &  0.014 \\
Coat &  0.958 &  0.006 &  0.978 &  0.001 &  {\bf 0.983} &  0.000 &  0.981 &  0.002 &  0.974 &  0.000 &  0.627 &  0.029 &  0.885 &  0.003 \\
\bottomrule
\end{tabular}}
\end{center}
\vspace{-0.3cm}
\end{table*}

\subsection{Experiments with image data}
\label{sec:num_exp}
To investigate how D3RE prevents train-loss hacking from occurring, we consider a setting of inlier-based outlier detection. For binary labels $y\in\{-1, +1\}$, we consider training a classifier only from $p(X\mid y=+1)$ and $p(X)$ to find a positive data point in the test data sampled from $p(X)$. The goal is to maximize the area under the receiver operating characteristic (AUROC) curve, which is a criterion often used for anomaly detection, by estimating the density ratio $r^*(X)=p(X\mid y=+1)/p(X)$. We construct positive and negative datasets from the {\tt CIFAR-10} \citep{Krizhevsky2009LearningML} dataset with $10$ classes. The positive dataset comprises `airplane,' `automobile,' `ship,' and `truck'; the negative dataset comprises `bird,' `cat,' `deer,' `dog,' `frog,' and `horse.' We use $1,000$ positive samples generated from $p(X\mid y=+1)$ and $1,000$ unlabeled samples generated from $p(X)$ to train the models. Then, we calculate the AUROCs using $10,000$ test samples generated from $p(X)$. In this case, it is desirable to set $C < \frac{1}{2}$ because $\frac{p(X\mid y=+1)}{0.5p(X\mid y=+1) + 0.5p(X\mid y=-1)}=\frac{1}{0.5 + 0.5\frac{p(X\mid y=-1)}{p(X\mid y=+1)}}$. For demonstrative purposes, we use a basic CNN architecture from the PyTorch tutorial \citep{pytorch}. The details are shown in Appendix~\ref{sec:app:network2}.
The model is trained by the Adam optimizer \citep{kingma2014method} without a weight decay and with the parameters \((\beta_1, \beta_2, \epsilon)\) fixed at the default values of the implementation in PyTorch \citep{pytorch}, namely \((0.9, 0.999, 10^{-8})\).

First, we compare two of the proposed estimators, \(\widehat{\mathrm{nnBD}}_{\mathrm{PU}}\) (nnBD-PU) and \(\widehat{\mathrm{nnBD}}_{\mathrm{LSIF}}\) (nnBD-LSIF), with two existing estimators, \(\widehat{\mathrm{BD}}_{\mathrm{PU}}\) (PU-NN) and \(\widehat{\mathrm{BD}}_{\mathrm{LSIF}}\) (uLSIF-NN) with neural networks. We use the logistic loss for PULogLoss. In addition, we conduct an experiment with uLSIF-NN using a naively capped model $\tilde{r}(X) = \min\{r(X), 1/C\}$ (Bounded uLSIF). We fix the hyperparameter $C$ at $1/3$. We report the results for two learning rates, $1\times10^{-4}$ and $1\times10^{-5}$. We conduct $10$ trials, and report the average AUROCs. We also compute $\hat{\mathbb{E}}_{\mathrm{de}}[\hat{r}(X)]$, which should be close to $1$ when the density ratio is successfully estimated since $\int (p(x| y=+1)/p(X)) p(X) dx = 1$. These results are shown in the left and center figures in Figure~\ref{fig:numerical_exp1}. In all cases, the proposed estimators outperform the other methods. We consider that the instabilities of PU-NN and LISF-NN are caused by the unboundedness of the objective function (also see \citet{KiryoPositiveunlabeled2017}, where similar experimental results are reported). The results also demonstrate that naive capping (Bounded uLSIF) fails to prevent train-loss hacking from occurring and leads to suboptimal behavior. As discussed in Sections~\ref{subsec:diff_dre} and \ref{sec:motive}, naive capping is insufficient for this problem because an unreasonable model such that $r(X^{\mathrm{de}}_i) = 0$ and $r(X^{\mathrm{nu}}_j) = 1/C$ can still be a minimizer by decreasing one part of the empirical BD, e.g., $\frac{1}{2}\hat{\mathbb{E}}_{\mathrm{de}}\left[r^2(X)\right] - \frac{C}{2}\hat{\mathbb{E}}_{\mathrm{nu}} \left[r^2(X)\right] = 0-\frac{1}{2C}$. 

Next, we investigate the sensitivity of D3RE to the hyperparameter $C$. We choose $C$ from $\{1/1.2, 1/1.5, 1/2.0. 1/3.0, 1/5.0\}$. The other settings remain unchanged from the previous experiment, where the exact upper bound $\overline{R}$ is $2.0$. The results are shown on the right-hand side of Figure~\ref{fig:numerical_exp1}. While estimators with $1/C \simeq 2.0$ show a superior performance, the method is robust to the choice of $C$ to a certain extent. Additional experimental results are reported in Appendix~\ref{appdx:sec:exp_image_data}.

Note that this experimental setting is similar to that of PU learning \citep{elkan2008learning,KiryoPositiveunlabeled2017}. In PU learning experiments, we mainly consider a binary classification problem, and the class-prior $p(y=+1)$ is given; that is, the goals and the presence of the information are the differences between the experimental settings of inlier-based outlier detection and PU learning. In this paper, we successfully related PU learning methods to DRE. The class-prior in PU learning plays a similar role to the upper bound of $r^*$ in DRE.

\subsection{Experiments on $L^2$ error}
\label{sec:regression_exp}
We empirically investigate the $L^2$ error in the proposed D3RE. We compare our method with the uLSIF. For uLSIF \citep{JMLR:Kanamori+etal:2009}, we use an open-source implementation\footnote{\url{https://github.com/hoxo-m/densratio_py}.}, which uses a linear-in-parameter model with the Gaussian kernel \citep{KanamoriStatistical2012a}. For D3RE, we use nnBD-LSIF and $3$-layer perceptron with a ReLU activation function, where the number of the nodes in the middle layer is $100$. We conducted nnBD-LSIF for all $C\in\{0.8, 1, 2, 3, 4, 5, 10, 15, 20\}$. We also compare these methods with a naively implemented LSIF with a $3$-layer perceptron. Let the dimensions of the domain be $d$ and
\begin{align*}
p_{\mathrm{nu}}(X) = \mathcal{N}(X; \mu^{\mathrm{nu}}, I_d),\ p_{\mathrm{de}}(X) = \mathcal{N}(X; \mu^{\mathrm{de}}, I_d),
\end{align*}
where $\mathcal{N}(X; \mu, \Sigma)$ denotes the multivariate normal distribution with mean $\mu$ and $\Sigma$, $\mu^{\mathrm{nu}}$ and $\mu^{\mathrm{de}}$ are $d$-dimensional vectors $\mu^{\mathrm{nu}} = (1,0,\dots, 0)^\top$ and $\mu^{\mathrm{de}} = (0,0,\dots, 0)^\top$, and $I_d$ is a $d$-dimensional identity matrix. We fix the sample sizes at $\nnu = \nde = 1,000$ and estimate the density ratio using uLSIF, LSIF, and D3RE (nnBD-LSIF). To measure the performance, we use the mean squared error (MSE) and the standard deviation (SD) averaged over $50$ trials. Note that in this setting, we know the true density ratio $r^*$. The results are shown in Table~\ref{tbl:synthetic}. The proposed nnBD-LSIF method estimates the density ratio more accurately than the other methods with a lower MSE. In many cases of the results, nnBD-LSIF achieves the best performance at approximately $C=2$. This result implies that we do not need to know the exact $C$ to achieve a high level performance in D3RE.

\section{Inlier-based outlier detection}
\label{sec:application}
As an application of D3RE, we perform \emph{inlier-based outlier detection} experiments with benchmark datasets. In addition to {\tt CIFAR-10}, we use {\tt MNIST} \citep{lecun-gradientbased-learning-applied-1998} and {\tt fashion-MNIST} ({\tt FMNIST}) \citep{Xiao2017FashionMNISTAN}, both of which have $10$ classes. \citet{hido2008,Hido2011} applied the a direct DRE for inlier-based outlier detection; that is, finding outliers in a test set based on a training set consisting only of inliers by using the ratio of training and test data densities as an outlier score. \citet{Hyunha2015} and \citet{Abe2019} proposed using shallow neural networks with DRE to deal with this problem. In relation to the experimental setting of Section~\ref{sec:num_exp}, the problem setting can be seen as a transductive variant of PU learning \citep{kato2018learning}.

We follow the setting proposed by \citet{NIPS2018_8183}. There are ten classes in each dataset, {\tt MNIST}, {\tt CIFAR-10}, and {\tt FMNIST}. We use one class as an inlier class and treat all other classes as outliers. For example, in the case of {\tt CIFAR-10}, there are $5,000$ train data per class. On the other hand, there are $1,000$ test data for each class, which amounts to $1,000$ inlier samples and $9,000$ outlier samples. The AUROC is used as a metric to evaluate whether the outlier class can be detected in the outlier samples. We compare the proposed methods with the benchmark methods of deep semi-supervised anomaly detection (DeepSAD) \citep{ruff2020deep} and geometric transformation (GT) \citep{NIPS2018_8183}. The details of each method are shown in Appendix~\ref{appdx:detail_anomaly}. To make a fair comparison, we use LeNet and Wide ResNet for D3RE, which are the same neural network architectures as those used in \citet{NIPS2018_8183} and \citet{ruff2020deep}. The detailed structures are shown in Appendix~\ref{sec:app:network}. Owing to the space limitation, some of the experimental results with {\tt MNIST}, {\tt CIFAR-10} and {\tt FMNIST} is shown in Table~\ref{tbl:exp_anomaly_detection}. The full results are shown in Table~\ref{tbl:full_exp_anomaly_detection} in Appendix~\ref{appdx:sec:exp_in_out_detect}. In almost all cases, the average AUROCs of the proposed methods are better than those of the existing methods. The largest performance gain is seen in the CIFAR-10, where the mean AUROC is improved by $0.157$ on average between the uLSIF-NN and nnBD-LSIF.
Although GT and DeepSAD are designed for different problem setups, to the best of our knowledge, there are no other appropriate state-of-the-art alternatives to these algorithms under this setting.

In Appendix~\ref{appdx:other_appl}, we also introduce other applications such as covariate shift adaptation. 

\citet{Togashi2021} applied our proposed method to personalized ranking from implicit feedback in a recommender systems. For this task, there are two approaches, pointwise and pairwise, and the former of which is known to be computationally efficient, whereas the latter shows better accuracy than the former. In that study, they reformulated a pointwise approach using the density ratio and also added the essence of the pairwise approach.

\section{Conclusion}
We proposed a non-negative correction to the empirical BD for DRE. Using the prior knowledge of the upper bound of the density ratio, we can prevent train-loss hacking from occurring when using flexible models. In our theoretical analyses, we provided generalization error bounds for the proposed method. In our experiments, we empirically confirmed the effectiveness of our proposed approach.

\section*{Acknowledgments}
The authors would like to thank Hirono Okamoto for his constructive advice.\\
TT was supported by Masason Foundation.


\bibliographystyle{icml2021}
\bibliography{D3RE.bbl}

\begin{thebibliography}{115}
\bibitem[Pearson(1900)]{pearson1900}
Pearson, K.
\newblock On the criterion that a given system of deviations from the probable
  in the case of a correlated system of variables is such that it can
  reasonably be supposed to have arisen from random sampling.
\newblock \emph{Philosophical Magazine}, 5\penalty0 (50):\penalty0 157--175,
  1900.
\end{thebibliography}

\onecolumn
\appendix

\section{Details of existing methods for DRE}
\label{appdx:exist_DRE}
In this section, we overview examples of DRE methods in the framework of the density ratio matching under BD. 

\paragraph{Least Squares Importance Fitting (LSIF).} 
LSIF minimizes the squared error between a density ratio model $r$ and the true density ratio $r^*$ defined as follows \citep{JMLR:Kanamori+etal:2009}:
\begin{align*}
R_{\mathrm{LSIF}}(r) = \mathbb{E}_{\mathrm{de}}[(r(X) - r^*(X))^{2}]= \mathbb{E}_{\mathrm{de}}[(r^*(X))^{2}] - 2\mathbb{E}_{\mathrm{nu}}[r(X)] + \mathbb{E}_{\mathrm{de}}[(r(X))^{2}].
\end{align*}
In the \emph{unconstrained LSIF} (uLSIF) \citep{JMLR:Kanamori+etal:2009}, we ignore the first term in the above equation and
estimate the density ratio by the following minimization problem:
\begin{align}
\label{ulsif}
\hr = \argmin_{r\in\mathcal{H}} \left[\frac{1}{2} \hat{\mathbb{E}}_{\mathrm{de}}[(r(X))^{2}] - \hat{\mathbb{E}}_{\mathrm{nu}}[r(X)] + \Reg(r)\right],
\end{align}
where $\Reg$ is a regularization term. This empirical risk minimization is equal to minimizing the empirical BD defined in \eqref{sample_br} with $f(t)=(t-1)^2/2$.

\paragraph{Unnormalized Kullback–Leibler (UKL) divergence and KL Importance Estimation Procedure (KLIEP).}
The KL importance estimation procedure (KLIEP) is derived from the unnormalized Kullback–Leibler (UKL) divergence objective \citep{sugiyama2008,nguyen2010,tsuboi2009,yamada2009,yamada2010}, which uses $f(t) = t\log (t) - t$. Ignoring the terms which are irrelevant for the optimization, we obtain the 
unnormalized Kullback–Leibler (UKL) divergence objective \citep{nguyen2010,Sugiyama:2012:DRE:2181148} as
\begin{align*}
&\mathrm{BD}_{\mathrm{UKL}}(r)= \mathbb{E}_{\mathrm{de}}\left[r(X)\right] - \mathbb{E}_{\mathrm{nu}} \left[\log\big(r(X)\big)\right].
\end{align*}
Directly minimizing UKL is proposed by \citet{nguyen2010}. The KLIEP also solves the same problem with further imposing a constraint that the ratio model $r(X)$ is non-negative for all $X$ and is normalized as
\begin{align*}
\hat{\mathbb{E}}_{\mathrm{de}}\left[r(X)\right] = 1.
\end{align*}
Then, following is the optimization criterion of KLIEP \citep{sugiyama2008}:
\begin{align*}
&\max_r  \hat{\mathbb{E}}_{\mathrm{nu}} \left[\log\big(r(X)\big)\right]\\
&\mathrm{s.t.}\ \hat{\mathbb{E}}_{\mathrm{de}}\left[r(X)\right] = 1\ \mathrm{and}\ r(X) \geq 0\ \mathrm{for}\ \mathrm{all}\ X.
\end{align*}

\paragraph{Logistic Regression (LR).} By using $f(t)=\log (t) - (1 + t)\log(1 + t)$, we obtain the following BD called the binary Kullback–Leibler (BKL) divergence:
\begin{align*}
&\mathrm{BD}_{\mathrm{BKL}}(r)= - \mathbb{E}_{\mathrm{de}}\left[\log\left(\frac{1}{1+r(X)}\right)\right] - \mathbb{E}_{\mathrm{nu}} \left[\log\left(\frac{r(X)}{1+r(X)}\right)\right].
\end{align*}
This BD is derived from a formulation based on the logistic regression \citep{hastie_09_elements-of_statistical-learning,sugiyama2011bregman}.

\paragraph{PU Learning with the log loss.} 
Consider a binary classification problem and
let $X$ and $y \in \{\pm 1\}$ be the feature and the label of a sample, respectively.
In PU learning, the goal is to train a classifier only using positive data sampled from $p(X\mid y=+1)$, and unlabeled data sampled from $p(X)$ in binary classification \citep{elkan2008learning}.
More precisely, this problem setting of PU learning is called the \emph{case-control scenario} \citep{elkan2008learning,NIPS2016_6354}. Let $\mathcal{G}$ be the set of measurable functions from $\mathcal{X}$ to $[\epsilon, 1-\epsilon]$, where $\epsilon \in (0, 1/2)$ is a small positive value. For a loss function $\ell: \mathbb{R}\times \{\pm1\}\to \Rpositive$, \citet{ICML:duPlessis+etal:2015} showed that the classification risk of $g \in \mathcal{G}$ in the PU problem setting can be expressed as
\begin{align}
\label{true}
&R_{\mathrm{PU}}(g)=\pi \int \Big(\ell(g(X), +1)  - \ell(g(X), -1)\Big) p(X\mid y=+1)dX + \int \ell(g(X), -1)] p(X)dX.
\end{align}
According to \citet{kato2018learning}, we can derive the following risk for DRE from the risk for PU learning \eqref{true} as follows:
\begin{align*}
&\mathrm{BD}_{\mathrm{PU}}(g)= \frac{1}{\overline{R}} \mathbb{E}_{\mathrm{nu}} \left[-\log\left(g(X)\right) + \log\left(1-g(X)\right) \right] - \mathbb{E}_{\mathrm{de}}\left[ \log\left(1-g(X)\right)\right],
\end{align*}
and
\citet{kato2018learning} showed that $g^* = \argmin_{g\in\mathcal{G}}\mathrm{BD}_{\mathrm{PU}}(g)$ satisfies the following:
\begin{proposition}
It holds almost everywhere that
\begin{align*}
g^*(X) = 
\begin{cases}
1-\varepsilon &(X\notin \DTwo),\\
C \frac{p_{\mathrm{nu}}(X)}{p_{\mathrm{de}}(X)} &(X\in \DOne\cap \DTwo),\\
\varepsilon &(X\notin \DOne),
\end{cases}
\end{align*}
where $C=\frac{1}{\overline{R}}$, $\DOne = \{X\mid C p_{\mathrm{nu}}(X) \geq \epsilon p_{\mathrm{de}}(X) \}$, and $\DTwo = \{X|C p_{\mathrm{nu}}(X) \leq (1-\epsilon)p_{\mathrm{de}}(X) \}$.
\end{proposition}

Using this result, we define the empirical version of $\mathrm{BD}_{\mathrm{PU}}(g)$ as follows:
\begin{align*}
&\widehat{\mathrm{BD}}_{\mathrm{PU}}(r^*\|r):=C\hat{\mathbb{E}}_{\mathrm{nu}} \left[-\log\big(r(X_i)\big)+ \log\big(1-r(X_j)\big)\right] -\hat{\mathbb{E}}_{\mathrm{de}} \left[\log\big(1-r(X_i)\big)\right].\nonumber
\end{align*}

To see that this is also a BD minimization method, define $f(t)$ as
\begin{align*}
f(t)=C\log\left(1-t\right) + Ct\left(\log\left(t\right)-\log\left(1-t\right)\right).
\end{align*}
Then, we have 
\begin{align*}
\partial f(t) = -\frac{C}{1-t} + C(\log(t)-\log(1-t)) + Ct\left(\frac{1}{t} + \frac{1}{1-t}\right).
\end{align*}
Therefore, we have
\begin{align*}
&\mathrm{BD}_f(r):=\mathbb{E}_{\mathrm{de}}\Big[\partial f\big(r(X_i)\big)r(X_i) - f\big(r(X_i)\big)\Big] - \mathbb{E}_{\mathrm{nu}}\Big[\partial f\big(r(X_j)\big)\Big]\\
&= \mathbb{E}_{\mathrm{de}}\Big[-\frac{Cr(X_i)}{1-r(X_i)} + Cr(X_i)(\log(r(X_i))-\log(1-r(X_i))) + Cr^2(X_i)\left(\frac{1}{r(X_i)} + \frac{1}{1-r(X_i)}\right)\Big]\\
&\ \ \  -\mathbb{E}_{\mathrm{de}}\Big[\log\left(1-r(X_i)\right) + Cr(X_i)\left(\log\left(r(X_i)\right)-\log\left(1-r(X_i)\right)\right)\Big]\\
&\ \ \ - \mathbb{E}_{\mathrm{nu}}\Big[-\frac{C}{1-r(X_i)} + C(\log(r(X_i))-\log(1-r(X_i))) + Cr(X_i)\left(\frac{1}{r(X_i)} + \frac{1}{1-r(X_i)}\right)\Big]\\
&= \mathbb{E}_{\mathrm{de}}\Big[-\frac{Cr(X_i)}{1-r(X_i)} + Cr(X_i)(\log(r(X_i))-\log(1-r(X_i))) + \frac{Cr(X_i)}{1-r(X_i)}\Big]\\
&\ \ \  -\mathbb{E}_{\mathrm{de}}\Big[\log\left(1-r(X_i)\right) + Cr(X_i)\left(\log\left(r(X_i)\right)-\log\left(1-r(X_i)\right)\right)\Big]\\
&\ \ \ - \mathbb{E}_{\mathrm{nu}}\Big[-\frac{C}{1-r(X_i)} + C(\log(r(X_i))-\log(1-r(X_i))) + \frac{C}{1-r(X_i)}\Big]\\
&= \mathbb{E}_{\mathrm{de}}\Big[\log\left(1-r(X_i)\right)\Big]- C\mathbb{E}_{\mathrm{nu}}\Big[ \log(r(X_i))-\log(1-r(X_i))\Big].
\end{align*}

\begin{remark}[DRE and PU learning]
\citet{pmlr-v48-menon16} showed that minimizing a proper CPE loss is equivalent to minimizing a BD to the true density ratio, and demonstrated the viability of using existing losses from one problem for the other for CPE and DRE.
\citet{kato2018learning} pointed out the relation between the PU learning and density ratio estimation and leveraged it to solve a sample selection bias problem in PU learning.
In this paper, we introduced the BD with $f(t)=\log\left(1-Ct\right) + Ct\left(\log\left(Ct\right)-\log\left(1-Ct\right)\right)$, inspired by the objective function of PU learning with the log loss.
In the terminology of \citet{pmlr-v48-menon16}, this $f$ results in a DRE objective without a \emph{link function}. In other words, it yields a direct DRE method.
\end{remark}

\section{Examples of \texorpdfstring{$\tilde f$}{f}}
\label{appdx:example_f}
Here, we show the examples of $\tilde f$ such that $\partial f (t) = C \big(\partial f(t)t - f(t)\big)+\tilde f(t)$, where $\tilde f(t)$ is bounded from above, and $\partial f(t)t - f(t) + A$ is non-negative.

First, we consider $f(t)=(t-1)^2/2$, which results in the LSIF objective. Because $\partial f (t) = t-1$, we have
\begin{align*}
&t-1 = C\big((t-1)t - (t-1)^2/2\big) + \tilde f(t)\\
&\Leftrightarrow \tilde f(t) = -C\big((t-1)t - (t-1)^2/2\big) + t-1 = -\frac{C}{2}t^2 + \frac{C}{2} + t - 1.
\end{align*}
The function is a concave quadratic function, therefore it is upper bounded.

Second, we consider $f(t)=t\log(t)-t$, which results in the UKL or KLIEP objective. Because $\partial f (t) = \log(t)$, we have
\begin{align*}
&\log(t) = C\big(\log(t)t - t\log(t)+t\big) + \tilde f(t)\\
&\Leftrightarrow \tilde f(t) = - tC + \log(t).
\end{align*}
We can easily confirm that the function is upper bounded by taking the derivative and finding that $t = 1/C$ gives the maximum. 

Third, we consider $f(t)=t\log(t) - (1+t)\log (1+t)$, which is used for DRE based on LR or BKL. Because $\partial f (t) = \log(t) - \log (1+t)$, we have
\begin{align*}
&\log(t) - \log (1+t) = C\big((\log(t) - \log (1+t))t - t\log(t) + (1+t)\log (1+t)\big) + \tilde f(t)\\
&\Leftrightarrow \tilde f(t) = - C\big(\log (1+t)\big) +  \log(t) - \log (1+t) = \log \left(\frac{C}{1+t}\right) + \log \left(\frac{t}{1+t}\right).
\end{align*}
We can easily confirm that the function is upper bounded as the terms involving $t$ always add up to be negative.

Fourth, we consider DRE based on PULog. By setting $f(t)=\log\left(1-t\right) + Ct\left(\log\left(t\right)-\log\left(1-t\right)\right)$, we can obtain the same risk functional introduced in \citet{KiryoPositiveunlabeled2017}.

\section{Train-loss hacking problem in PU classification}
\label{section: train loss hacking in PU}
Here, we introduce the train-loss hacking discussed in the PU learning literature \cite{KiryoPositiveunlabeled2017}. In a standard binary classification problem, we train a classifier $\psi$ by minimizing the following empirical risk:
\begin{align}
\label{eq:standard_er}
    \frac{1}{n}\sum^n_{i=1}\mathbbm{1}[y_i=+1]\ell(\psi(X_i)) + \frac{1}{n}\sum^n_{i=1}\mathbbm{1}[y_i=-1]\ell(-\psi(X_i)),
\end{align}
where $y_i\in\{\pm 1\}$ is a binary label, $X_i$ is a feature, and $\ell$ is a loss function. On the other hand, in PU learning formulated by \citet{ICML:duPlessis+etal:2015}, because we only have positive data $\{(y'_i=+1, X'_i)\}^{n'}_{i=1}$ and unlabeled data $\{(\bm{x''}_j)\}^{n''}_{j=1}$, we minimize the following alternative empirical risk:
\begin{align}
\label{eq:pu_er}
    \frac{\pi}{n'}\sum^{n'}_{i=1}\ell(\psi(X'_i) \annot{- \frac{\pi}{n'}\sum^{n'}_{i=1}\ell(-\psi(X'_i))}{Cause of train-loss hacking.} +  \frac{1}{n''}\sum^{n''}_{j=1}\ell(-\psi(X''_j)),
\end{align}
where $\pi$ is a hyperparameter representing $p(y=+1)$. Note that the empirical risk \eqref{eq:pu_er} is unbiased to the population binary classification risk \eqref{eq:standard_er} \citep{ICML:duPlessis+etal:2015}. While the the empirical risk \eqref{eq:standard_er} of the standard binary classification is lower bounded under an appropriate choice of $\ell$, the empirical risk \eqref{eq:pu_er} of PU learning proposed by \citet{ICML:duPlessis+etal:2015} is not lower bounded owing to the existence of the second term. Therefore, if a model is sufficiently flexible, we can significantly minimize the empirical risk only by minimizing the second term $-\frac{\pi}{n'}\sum^{n'}_{i=1}\ell(-\psi(X'_i))$ without increasing the other terms. \citet{KiryoPositiveunlabeled2017} proposed non-negative risk correction for avoiding this problem when using neural networks.

\section{Network structure used in Sections~\ref{sec:num_exp} and \ref{sec:application}}
\label{sec:app:network}
We explain the structures of neural networks used in the experiments. 

\subsection{Network structure used in Sections~\ref{sec:num_exp}}
In Section~\ref{sec:num_exp}, we used {\tt CIFAR-10} datasets. The model was a convolutional net  \citep{DB15a}: $(32\times32\times3)\mathchar`-C(3\times6, 3)\mathchar`-C(3\times16, 3)\mathchar`-128\mathchar`-84\mathchar`-1$, where the input is a $32\times32$ RGB image, $C(3\times6,3)$ indicates that $3$ channels of $3\times6$ convolutions followed by ReLU is used. This structure has been adopted from the tutorial of \citet{pytorch}.

\subsection{Network structure used in Sections~\ref{sec:application}}
\label{sec:app:network2}
\paragraph{Inlier-based Outlier Detection.} We used the same LeNet-type CNNs proposed in \citet{ruff2020deep}. In the CNNs, each convolutional module consists of a convolutional layer followed by leaky ReLU activations with leakiness $\alpha=0.1$ and $(2\times 2)$-max-pooling.
For MNIST, we employ a CNN with two modules: $(32\times32\times3)\mathchar`-C(3\times32, 5)\mathchar`-C(32\times64, 5)\mathchar`-C(64\times128, 5)\mathchar`-1$.
For {\tt CIFAR-10} we employ the following architecture: $(32\times32\times1)\mathchar`-C(1\times8, 5)\mathchar`-C(8\times4, 5)\mathchar`-1$ with a batch normalization \citep{43442} after each convolutional layer.

The WRN architecture was proposed in \citet{Zagoruyko2016} and it is also used in \citet{NIPS2018_8183}. This structure improved the performance of image recognition by decreasing the depth and increasing the width of the residual networks \citep{He2015}. We omit the detailed description of the structure here.

\paragraph{Covariate Shift Adaptation.} We used the $5$-layer perceptron with ReLU activations. The structure is $10000\mathchar`-1000\mathchar`-1000\mathchar`-1000\mathchar`-1000\mathchar`-1$.

\begin{figure}[t]
\begin{center}
 \includegraphics[width=137mm]{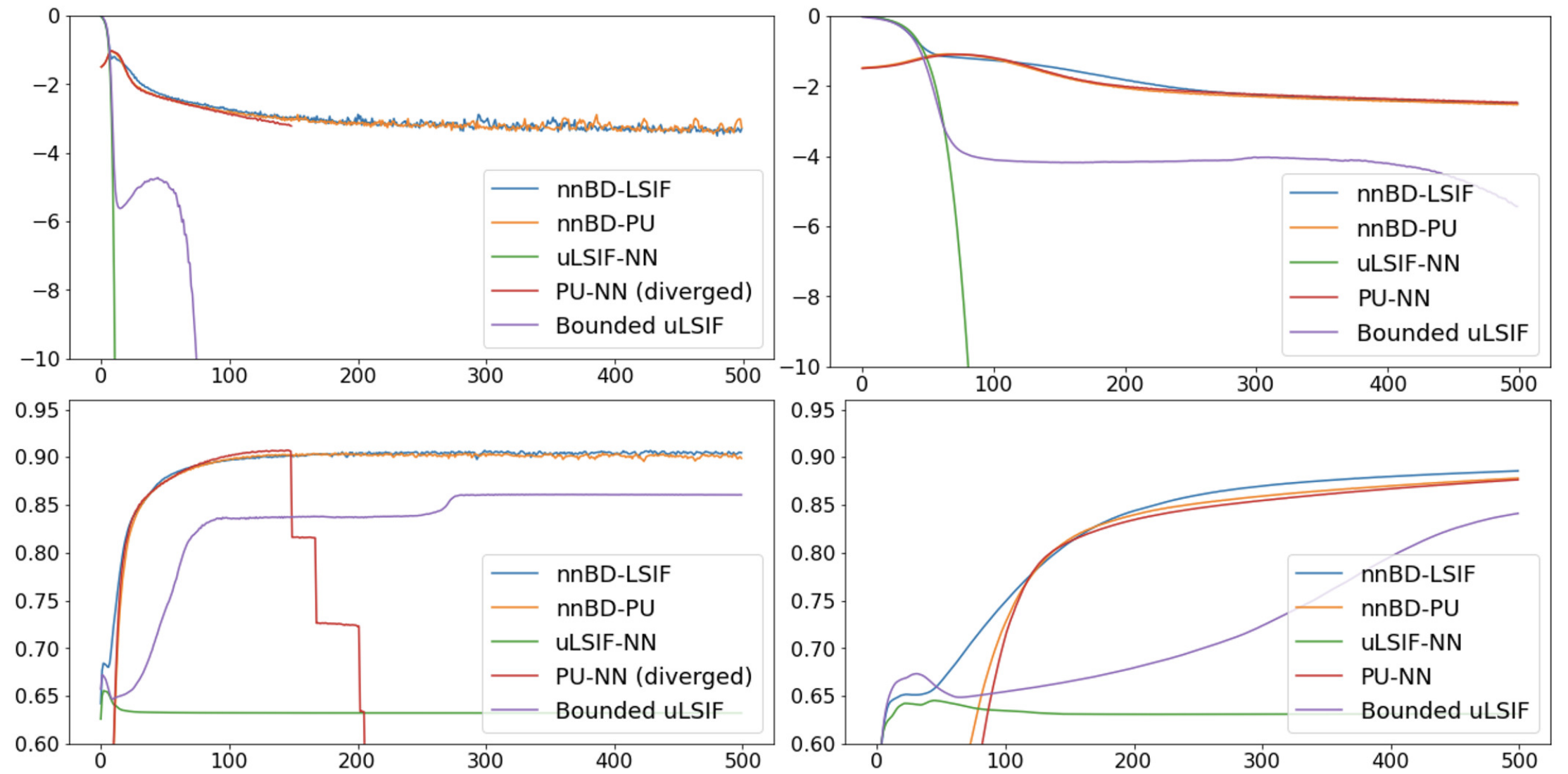}
\end{center}
\caption{The learning curves of the experiments in Section~\ref{sec:num_exp}. The horizontal axis is epoch. The vertical axes of the top figures indicate the training losses. The vertical axes of the bottom figures show the AUROC for the test data. The bottom figures are identical to the ones displayed in Section~\ref{sec:num_exp}.}
\label{fig:numerical_exp1_supp}
\end{figure} 

\begin{figure}[t]
\begin{center}
 \includegraphics[width=75mm]{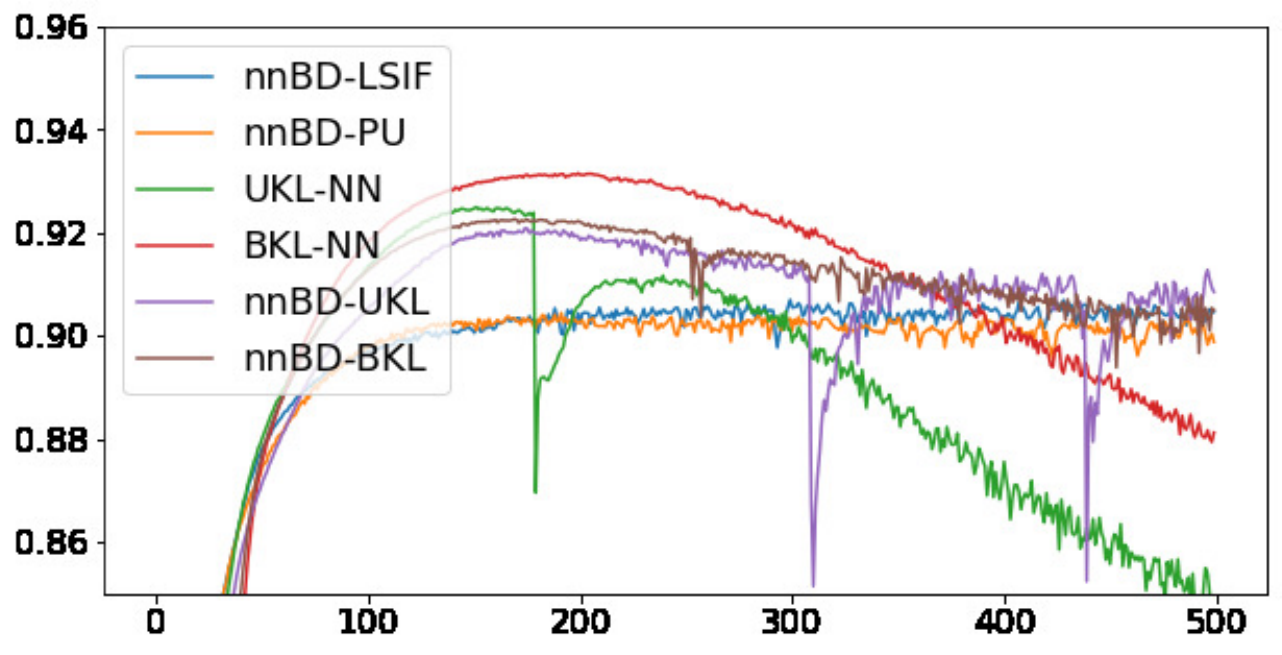}
\end{center}
\caption{The results of Section~\ref{sec:exp_val_estimators}. The horizontal axis is epoch, and the vertical axis is AUROC.}
\label{fig:numerical_exp2}
\end{figure} 

\section{Existing methods for anomaly detection}
\label{appdx:detail_anomaly}
This section introduces the existing methods for anomaly detection. DeepSAD is a method for semi-supervised anomaly detection, which tries to take advantage of labeled anomalies \citep{ruff2020deep}. GT proposed by \citet{NIPS2018_8183} trains neural networks based on a self-labeled dataset by performing $72$ geometric transformations. The anomaly score based on GT is calculated based on the Dirichlet distribution obtained by maximum likelihood estimation using the softmax output from the trained network.

In the problem setting of the DeepSAD, we have access to a small pool of labeled samples, e.g. a subset verified by some domain expert as being normal or anomalous. In the experimental results shown in \citet{ruff2020deep} indicate that, when we can use such samples, the DeepSAD outperforms the other methods. However, in our experimental results, such samples are not assumed to be available, hence the method does not perform well. The problem setting of \citet{ruff2020deep} and ours are both termed \emph{semi-supervised learning} in anomaly detection, but the two settings are different.

\section{Details of experiments}
The details of experiments are shown in this section. The description of the data is as follows:
\begin{description}
\item[{\tt MNIST}:] The {\tt MNIST} database is one of the most popular benchmark datasets for image classification, which consists of $28\times28$ pixel handwritten digits from $0$ to $9$ with $60,000$ train samples and $10,000$ test samples \citep{lecun-gradientbased-learning-applied-1998}. See \url{http://yann.lecun.com/exdb/mnist/}.
\item[{\tt CIFAR-10}:] The {\tt CIFAR-10} dataset consists of $60,000$ color images of size $32\times32$ from $10$ classes, each having $6000$. There are $50,000$ training images and $10,000$ test images \citep{Krizhevsky2012}. See \url{ https://www.cs.toronto.edu/~kriz/cifar.html}.
\item[{\tt fashion-MNIST}:]  The {\tt fashion-MNIST} dataset consists of $70,000$ grayscale images of size $28\times28$ from $10$ classes. There are $60,000$ training images and $10,000$ test images \citep{Xiao2017FashionMNISTAN}. See \url{https://github.com/zalandoresearch/fashion-mnist}.
\item[{\tt Amazon Review Dataset}:] \citet{blitzer-etal-2007-biographies} published the text data of Amazon review. The data originally consists of a rating ($0$-$5$ stars) for four different genres of products in the electronic commerce site Amazon.com: books, DVDs, electronics, and kitchen appliances. \citet{blitzer-etal-2007-biographies} also released the pre-processed and balanced data of the original data. The pre-processed data consists of text data with four labels $1$, $2$, $4$, and $5$. We map the text data into $10,000$ dimensional data by the TF-IDF mapping with that vocabulary size. In the experiment, for the pre-processed data, we solve the regression problem where the text data are the inputs and the ratings $1$, $2$, $4$, and $5$ are the outputs. When evaluating the performance, following \citet{pmlr-v48-menon16}, we calculate PD (=1-AUROUC) by regarding $4$ and $5$ ratings as positive labels and $1$ and $2$ ratings as negative labels.
\end{description}

\subsection{Experiments with image data}
\label{appdx:sec:exp_image_data}
We show the additional results of Section~\ref{sec:num_exp}. In  Figure~\ref{fig:numerical_exp1_supp}, we show the training loss of LSIF-based methods to demonstrate the train-loss hacking phenomenon caused by the objective function without a lower bound. In Figure~\ref{fig:numerical_exp1_supp}, even though the training loss of uLSIF-NN and that of Bounded uLSIF decrease more rapidly than that of nnBD-LSIF, the test AUROC score (the higher the better) either drops or fails to increase. These graphs are the manifestations of the severe train-loss hacking in DRE without our proposed device.

\subsubsection{Comparison with various estimators using nnBD divergence}
\label{sec:exp_val_estimators}
Let UKL-NN and BKL-NN be DRE method with the UKL and BKL losses with neural networks without non-negative correction. Finally, we examine the performances of nnBD-LSIF, nnBD-PU, UKL-NN, BKL-NN, nnBD-UKL, and nnBD-BKL. The learning rate was $1\times10^{-4}$, and the other settings were identical to those in the previous experiments. These results are shown in Figure~\ref{fig:numerical_exp2}. UKL-NN and BKL-NN also suffer from train-loss hacking although BKL loss seems to be more robust against the train-loss hacking  than the other loss functions. Although nnBD-UKL and nnBD-BKL show better performance in earlier epochs, nnBD-LSIF and nnBD-PU appear to be more stable. 

\subsubsection{Results without gradient ascent}
\label{appdx:res_wo_ga}
We also show the experimental results without the gradient ascent heuristic. Figure~\ref{fig:numerical_exp1_2} corresponds to the Figure~\ref{fig:numerical_exp1} without the gradient ascent heuristic. Figure~\ref{fig:numerical_exp1_supp_2} corresponds to the Figure~\ref{fig:numerical_exp1_supp} without the gradient ascent heuristic. Figure~\ref{fig:numerical_exp2_2} corresponds to the Figure~\ref{fig:numerical_exp2} without the gradient ascent heuristic. As shown in these experiments, although the gradient ascent/descent heuristic improve the performance, there is no significant difference between empirical performance with and without the heuristic. Therefore, we recommend practitioners to use the gradient ascent/descent heuristic, but if the reader concerns the theoretical guarantee, they can use the plain gradient descent algorithm; that is, naively minimize the proposed empirical nnBD risk.

\begin{figure}[t]
\vspace{-0.cm}
\begin{center}
 \includegraphics[keepaspectratio,width=110mm]{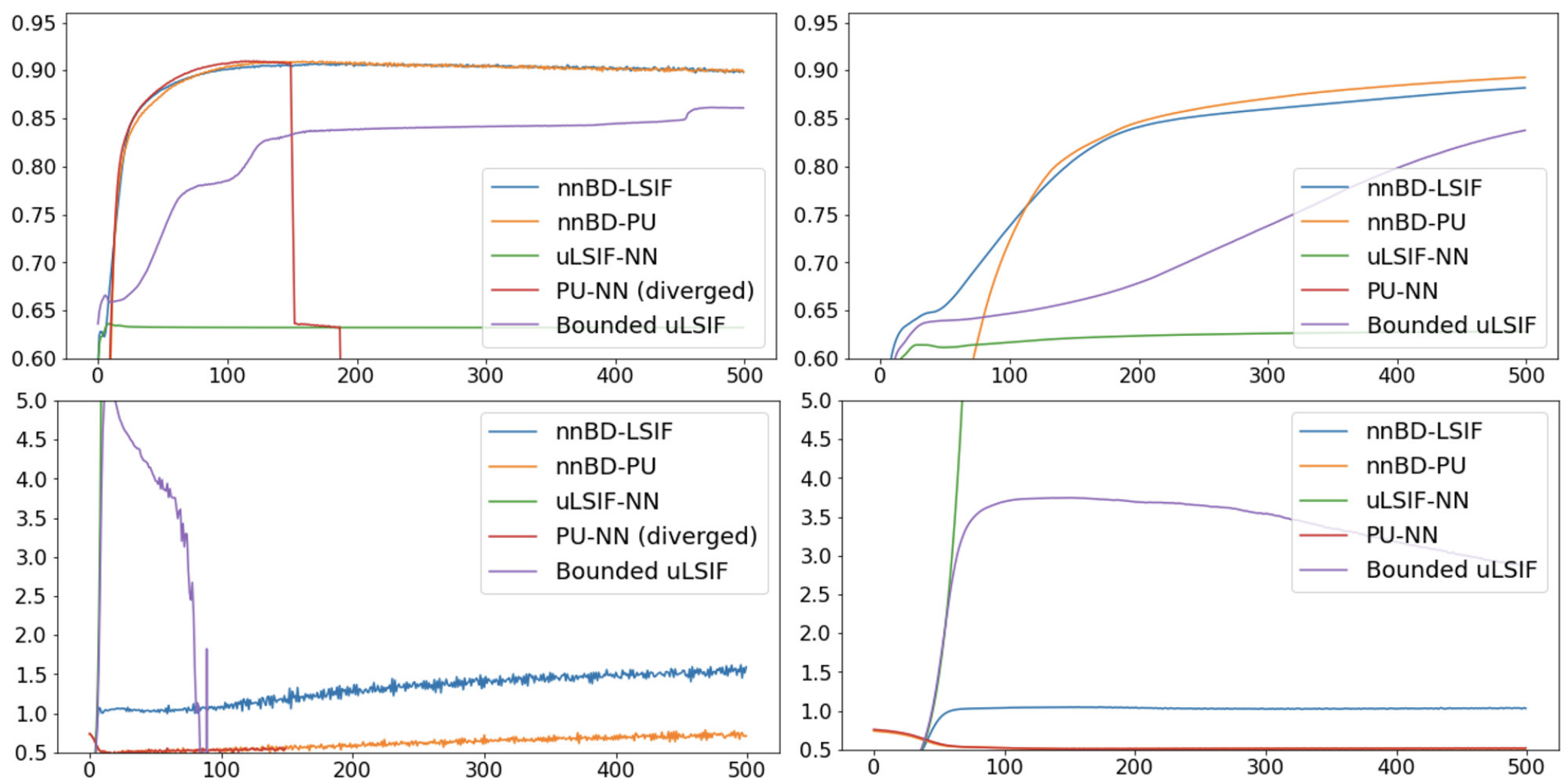}
\end{center}
\vspace{-0.3cm}
\caption{Experimental results of Section~\ref{sec:num_exp} without gradient ascent/descent heuristic. The horizontal axis is epoch, and the vertical axis is AUROC. The learning rates of the left and right graphs are $1\times10^{-4}$ and $1\times10^{-5}$, respectively. The upper graphs show the AUROCs and the lower graphs show $\hat{\mathbb{E}}_{\mathrm{de}}[\hat{r}(X)]$, which will approach $1$ when we successfully estimate the density ratio.}
\label{fig:numerical_exp1_2}
\end{figure}

\begin{figure}[t]
\begin{center}
 \includegraphics[width=137mm]{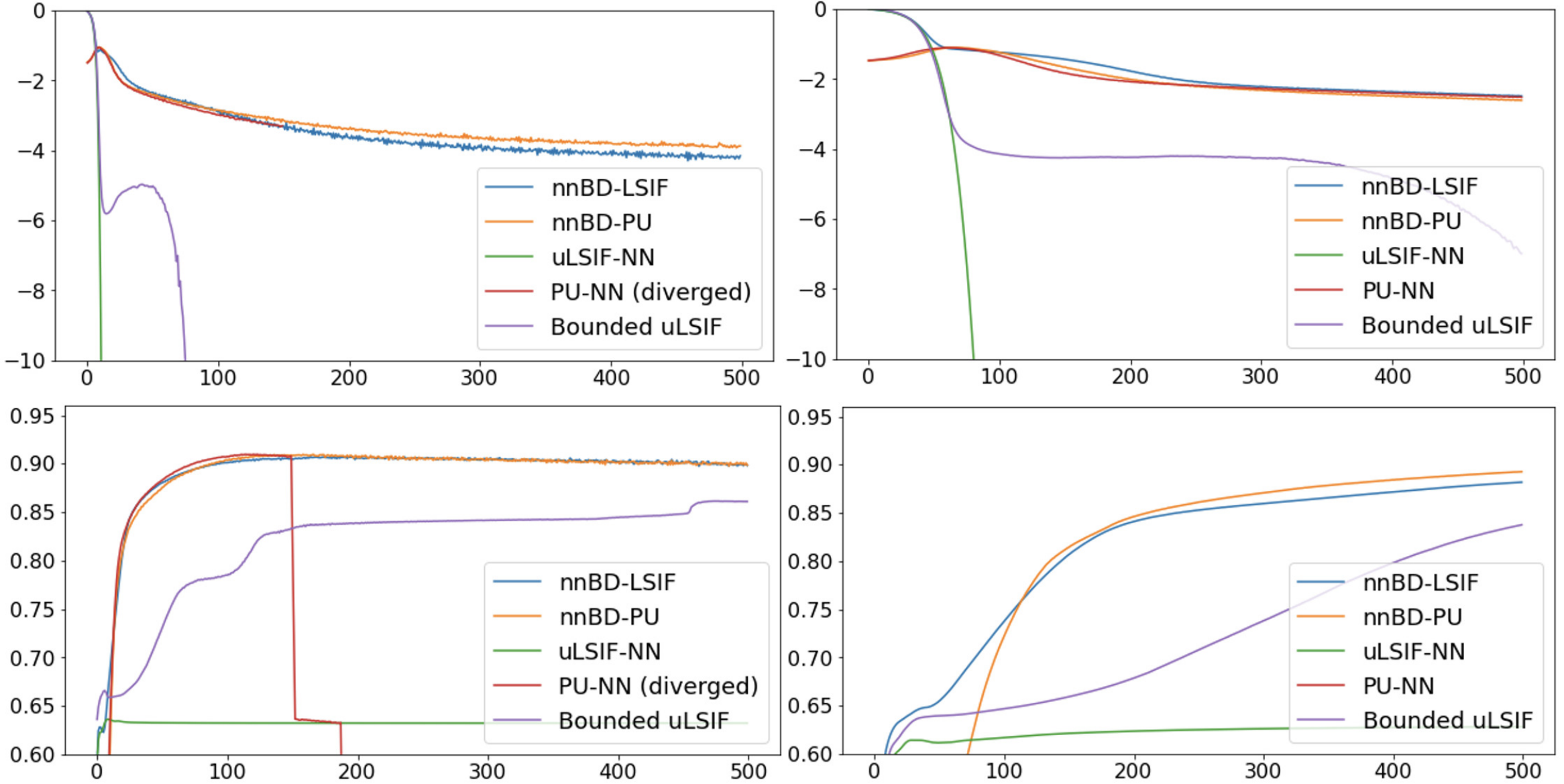}
\end{center}
\caption{The learning curves of the experiments in Section~\ref{sec:num_exp} without gradient ascent/descent heuristic. The horizontal axis is epoch. The vertical axes of the top figures indicate the training losses. The vertical axes of the bottom figures show the AUROC for the test data. The bottom figures are identical to the ones displayed in Section~\ref{sec:num_exp}.}
\label{fig:numerical_exp1_supp_2}
\end{figure} 

\begin{figure}[t]
\begin{center}
 \includegraphics[width=137mm]{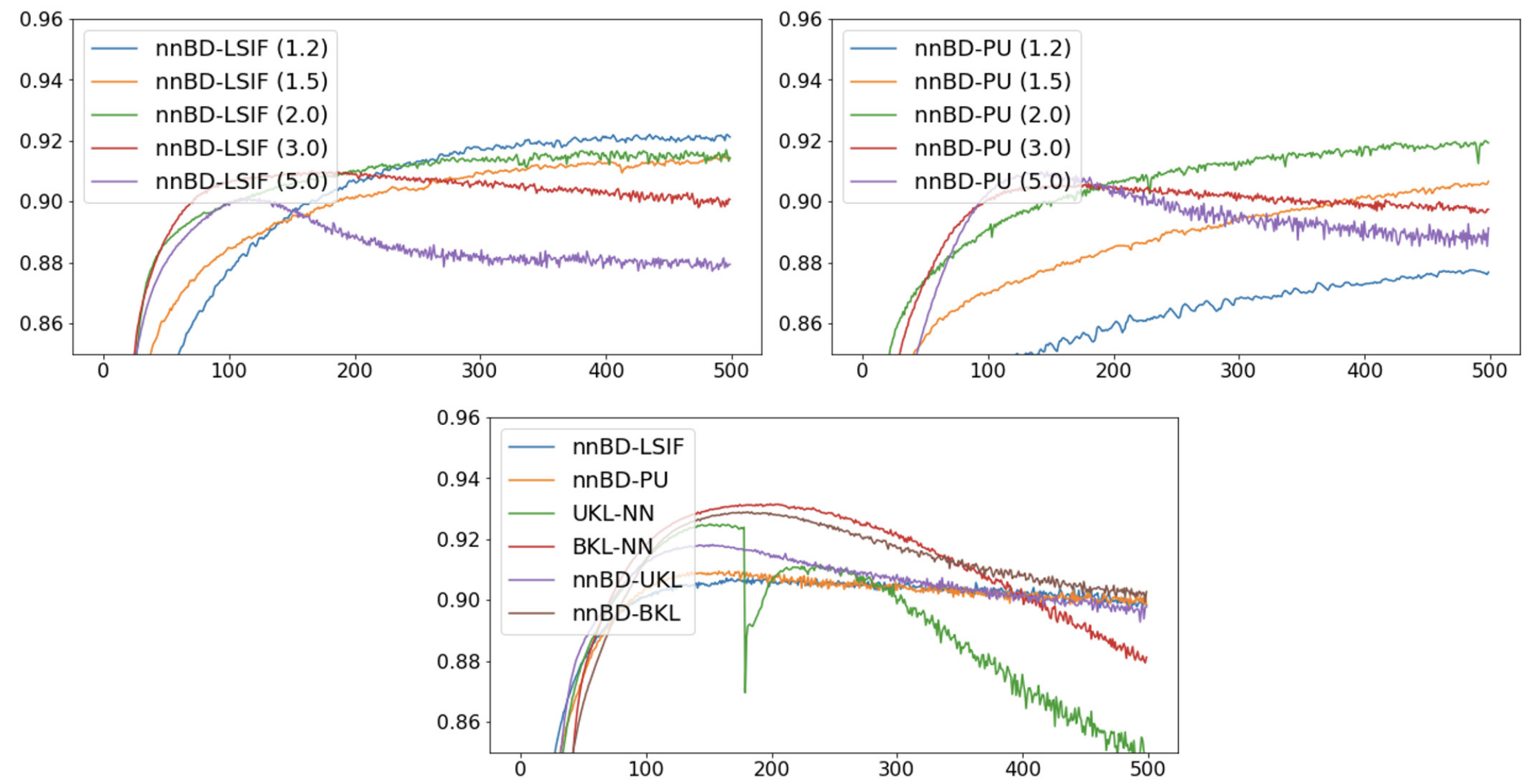}
\end{center}
\caption{The detailed experimental results for Section~\ref{sec:exp_val_estimators}. The horizontal axis is epoch, and the vertical axis is AUROC.}
\label{fig:numerical_exp2_2}
\end{figure}

\subsection{Experiments of inlier-based outlier detection}
In Table~\ref{tbl:full_exp_anomaly_detection}, we show the full results of inlier-based outlier detection. In almost all the cases, D3RE for inlier-based outlier detection outperforms the other methods. As explained in Section~\ref{appdx:detail_anomaly}, we consider that DeepSAD does not work well because the method assumes the availability of the labeled anomaly data, which is not available in our problem setting.

\label{appdx:sec:exp_in_out_detect}

In Table~\ref{tbl:full_exp_anomaly_detection2}, for different $1/C$ chosen from $\{1, 3, 5, 10\}$, we report the AUROCs of nnBD-LSIF with and without gradient ascent. As shown in the results, loose specification does not significantly decrease the performances. The gradient ascent technique improves the performances, but plain gradient descent still performs well.

\begin{remark}[Benchmark Methods]

Although GT is outperformed by our proposed method, the problem setting for the comparison is not in favor of GT as it does not assume the access to the test data. Recently proposed methods for semi-supervised anomaly detection by \citet{ruff2020deep} did not perform well without using other side information used in \citet{ruff2020deep}. On the other hand, there is no other competitive methods in this problem setting, to the best of our knowledge.
\end{remark}

\subsection{Experiments of covariate shift adaptation}
\label{appdx:sec:cov_adapt}
In Table~\ref{tbl:appdx:exp_covariate_shift}, we show the detailed results of experiments of covariate shift adaptation. Even when the training data and the test data follow the same distribution, the covariate shift adaptation based on D3RE improves the mean PD. We consider that this is because the importance weighting emphasizes the loss in the empirical higher-density regions of the test examples.

\begin{table}
\caption{Average area under the ROC curve (Mean) of anomaly detection methods averaged over $5$ trials with the standard deviation (SD). For all datasets, each model was trained on the single class, and tested against all other
classes. The best performing method in each experiment is in bold.
\emph{SD}: Standard deviation.
}
\label{tbl:full_exp_anomaly_detection}
\begin{center}
\scalebox{0.73}[0.73]{
\begin{tabular}{l|rr|rr|rr|rr|rr|rr|rr}
\toprule
{\tt MNIST} & \multicolumn{2}{c|}{uLSIF-NN} & \multicolumn{2}{c|}{nnBD-LSIF} & \multicolumn{2}{c|}{nnBD-PU} & \multicolumn{2}{c|}{nnBD-LSIF} & \multicolumn{2}{c|}{nnBD-PU} & \multicolumn{2}{c|}{Deep SAD} & \multicolumn{2}{c}{GT}\\
Network & \multicolumn{2}{c|}{LeNet} & \multicolumn{2}{c|}{LeNet} & \multicolumn{2}{c|}{LeNet} & \multicolumn{2}{c|}{WRN} & \multicolumn{2}{c|}{WRN} & \multicolumn{2}{c|}{LeNet} & \multicolumn{2}{c}{WRN}\\
\hline
Inlier Class &             Mean &    SD &      Mean &      SD &      Mean &      SD  &      Mean &      SD &      Mean &      SD &      Mean &      SD &      Mean &      SD\\
\hline
0 &  0.999 &  0.000 &  0.997 &  0.000 &  0.999 &  0.000 &  {\bf 1.000} &  0.000 &  {\bf 1.000} &  0.000 &  0.592 &  0.051 & 0.963 &  0.002 \\
1 &  {\bf 1.000} &  0.000 &  0.999 &  0.000 &  {\bf 1.000} &  0.000 &  {\bf 1.000} &  0.000 &  {\bf 1.000} &  0.000 &  0.942 &  0.016 & 0.517 &  0.039 \\
2 &  0.997 &  0.001 &  0.994 &  0.000 &  0.997 &  0.001 &  {\bf 1.000} &  0.000 &  {\bf 1.000} &  0.001 &  0.447 &  0.027 & 0.992 &  0.001 \\
3 &  0.997 &  0.000 &  0.995 &  0.001 &  0.998 &  0.000 &  {\bf 1.000} &  0.000 &  {\bf 1.000} &  0.000 &  0.562 &  0.035 & 0.974 &  0.001 \\
4 &  0.998 &  0.000 &  0.997 &  0.001 &  0.999 &  0.000 &  {\bf 1.000} &  0.000 &  {\bf 1.000} &  0.000 &  0.646 &  0.015 & 0.989 &  0.001 \\
5 &  0.997 &  0.000 &  0.996 &  0.001 &  0.998 &  0.000 &  {\bf 1.000} &  0.000 &  {\bf 1.000} &  0.000 &  0.502 &  0.046 & 0.990 &  0.001 \\
6 &  0.997 &  0.001 &  0.997 &  0.001 &  0.999 &  0.000 &  {\bf 1.000} &  0.000 &  {\bf 1.000} &  0.000 &  0.671 &  0.027 & 0.998 &  0.000 \\
7 &  0.996 &  0.001 &  0.993 &  0.001 &  0.998 &  0.001 &  {\bf 1.000} &  0.000 &  {\bf 1.000} &  0.001 &  0.685 &  0.032 & 0.927 &  0.004 \\
8 &  0.997 &  0.000 &  0.994 &  0.001 &  0.997 &  0.000 &  {\bf 0.999} &  0.000 &  {\bf 0.999} &  0.000 &  0.654 &  0.026 & 0.949 &  0.002 \\
9 &  0.993 &  0.002 &  0.990 &  0.002 &  0.994 &  0.001 &  {\bf 0.998} &  0.001 &  {\bf 0.998} &  0.001 &  0.786 &  0.021 & 0.989 &  0.001 \\
\bottomrule
\end{tabular}}
\end{center}
\begin{center}
\scalebox{0.73}[0.73]{
\begin{tabular}{l|rr|rr|rr|rr|rr|rr|rr}
\toprule
{\tt CIFAR-10} & \multicolumn{2}{c|}{uLSIF-NN} & \multicolumn{2}{c|}{nnBD-LSIF} & \multicolumn{2}{c|}{nnBD-PU} & \multicolumn{2}{c|}{nnBD-LSIF} & \multicolumn{2}{c|}{nnBD-PU} & \multicolumn{2}{c|}{Deep SAD} & \multicolumn{2}{c}{GT}\\
Network & \multicolumn{2}{c|}{LeNet} & \multicolumn{2}{c|}{LeNet} & \multicolumn{2}{c|}{LeNet} & \multicolumn{2}{c|}{WRN} & \multicolumn{2}{c|}{WRN} & \multicolumn{2}{c|}{LeNet} & \multicolumn{2}{c}{WRN}\\
\hline
Inlier Class &             Mean &    SD &      Mean &      SD &      Mean &      SD  &      Mean &      SD &      Mean &      SD &      Mean &      SD &      Mean &      SD\\
\hline
plane &  0.745 &  0.056 &  0.934 &  0.002 &  {\bf 0.943} &  0.001 &  0.925 &  0.004 &  0.923 &  0.001 &  0.627 &  0.066 &  0.697 &  0.009 \\
car &  0.758 &  0.078 &  0.957 &  0.002 &  {\bf 0.968} &  0.001 &  0.965 &  0.002 &  0.960 &  0.001 &  0.606 &  0.018 &  0.962 &  0.003 \\
bird &  0.768 &  0.012 &  0.850 &  0.007 &  {\bf 0.878} &  0.004 &  0.844 &  0.004 &  0.858 &  0.004 &  0.404 &  0.006 &  0.752 &  0.002 \\
cat &  0.745 &  0.037 &  0.820 &  0.003 &  {\bf 0.856} &  0.002 &  0.810 &  0.009 &  0.841 &  0.002 &  0.517 &  0.018 &  0.727 &  0.014 \\
deer &  0.758 &  0.036 &  0.886 &  0.004 &  {\bf 0.909} &  0.002 &  0.864 &  0.008 &  0.872 &  0.002 &  0.704 &  0.052 &  0.863 &  0.014 \\
dog &  0.728 &  0.103 &  0.875 &  0.004 &  {\bf 0.906} &  0.002 &  0.887 &  0.005 &  0.896 &  0.002 &  0.490 &  0.025 &  0.873 &  0.002 \\
frog &  0.750 &  0.060 &  0.944 &  0.003 &  {\bf 0.958} &  0.001 &  0.948 &  0.004 &  0.948 &  0.001 &  0.744 &  0.014 &  0.879 &  0.008 \\
horse &  0.782 &  0.048 &  0.928 &  0.003 &  0.948 &  0.002 &  0.921 &  0.007 &  0.927 &  0.002 &  0.519 &  0.015 &  {\bf 0.953} &  0.001 \\
ship &  0.780 &  0.048 &  0.958 &  0.003 &  {\bf 0.965} &  0.001 &  0.964 &  0.002 &  0.957 &  0.001 &  0.430 &  0.062 &  0.921 &  0.009 \\
truck &  0.708 &  0.081 &  0.939 &  0.003 &  {\bf 0.955} &  0.001 &  0.952 &  0.003 &  0.949 &  0.001 &  0.393 &  0.008 &  0.911 &  0.003 \\
\bottomrule
\end{tabular}}
\end{center}
\begin{center}
\scalebox{0.73}[0.73]{
\begin{tabular}{l|rr|rr|rr|rr|rr|rr|rr}
\toprule
{\tt FMNIST} & \multicolumn{2}{c|}{uLSIF-NN} & \multicolumn{2}{c|}{nnBD-LSIF} & \multicolumn{2}{c|}{nnBD-PU} & \multicolumn{2}{c|}{nnBD-LSIF} & \multicolumn{2}{c|}{nnBD-PU} & \multicolumn{2}{c|}{Deep SAD} & \multicolumn{2}{c}{GT}\\
Network & \multicolumn{2}{c|}{LeNet} & \multicolumn{2}{c|}{LeNet} & \multicolumn{2}{c|}{LeNet} & \multicolumn{2}{c|}{WRN} & \multicolumn{2}{c|}{WRN} & \multicolumn{2}{c|}{LeNet} & \multicolumn{2}{c}{WRN}\\
\hline
Inlier Class &             Mean &    SD &      Mean &      SD &      Mean &      SD  &      Mean &      SD &      Mean &      SD &      Mean &      SD &      Mean &      SD\\
\hline
T-shirt/top &  0.960 &  0.005 &  0.981 &  0.001 &  {\bf 0.985} &  0.000 &  0.984 &  0.001 &  0.982 &  0.000 &  0.558 &  0.031 &  0.890 &  0.007 \\
Trouser &  0.961 &  0.010 &  0.998 &  0.000 &  {\bf 1.000} &  0.000 &  0.998 &  0.000 &  0.998 &  0.000 &  0.758 &  0.022 &  0.974 &  0.004 \\
Pullover &  0.944 &  0.012 &  0.976 &  0.001 &  0.980 &  0.001 &  {\bf 0.983} &  0.002 &  0.972 &  0.001 &  0.617 &  0.046 &  0.902 &  0.005 \\
Dress &  0.973 &  0.006 &  0.986 &  0.001 &  {\bf 0.992} &  0.000 &  0.991 &  0.001 &  0.986 &  0.000 &  0.525 &  0.038 &  0.843 &  0.014 \\
Coat &  0.958 &  0.006 &  0.978 &  0.001 &  {\bf 0.983} &  0.000 &  0.981 &  0.002 &  0.974 &  0.000 &  0.627 &  0.029 &  0.885 &  0.003 \\
Sandal &  0.968 &  0.011 &  0.997 &  0.001 &  {\bf 0.999} &  0.000 &  {\bf 0.999} &  0.000 &  {\bf 0.999} &  0.000 &  0.681 &  0.023 &  0.949 &  0.005 \\
Shirt &  0.919 &  0.005 &  0.952 &  0.001 &  {\bf 0.958} &  0.001 &  0.944 &  0.005 &  0.932 &  0.001 &  0.618 &  0.015 &  0.842 &  0.004 \\
Sneaker &  0.991 &  0.001 &  0.994 &  0.002 &  {\bf 0.998} &  0.000 &  {\bf 0.998} &  0.000 &  {\bf 0.998} &  0.000 &  0.802 &  0.054 &  0.954 &  0.006 \\
Bag &  0.980 &  0.005 &  0.994 &  0.001 &  {\bf 0.999} &  0.000 &  0.998 &  0.000 &  {\bf 0.999} &  0.000 &  0.447 &  0.034 &  0.973 &  0.006 \\
Ankle boot &  0.992 &  0.001 &  0.985 &  0.015 & {\bf 0.999} &  0.000 &  0.997 &  0.000 &  0.996 &  0.000 &  0.583 &  0.023 &  0.996 &  0.000 \\
\bottomrule
\end{tabular}}
\end{center}

\end{table}

\begin{table}
\caption{We show average area under the ROC curve (Mean) of anomaly detection methods averaged over $5$ trials with the standard deviation (SD) for nnBD-LSIF with LeNet. We choose $1/C$, which represents a guessed upper bound, from $\{1, 3, 5, 10\}$. Each model is trained on the single class, and tested against all other
classes. We show both results with and without gradient ascent and $\circ$ denotes the use of the gradient ascent technique. The best performing method for each inlier class is highlighted in bold. The best performing method for each $1/C$ is highlighted in underline.}
\label{tbl:full_exp_anomaly_detection2}
\begin{center}
\scalebox{0.73}[0.73]{
\begin{tabular}{l|rr|rr|rr|rr|rr|rr|rr|rr}
\toprule
{\tt CIFAR-10} & \multicolumn{16}{c}{nnBD-LSIF} \\
Network & \multicolumn{16}{c}{LeNet} \\
\hline
$1/C$ (Guessed upper bound) & \multicolumn{4}{c|}{$1$} & \multicolumn{4}{c|}{$3$} & \multicolumn{4}{c|}{$5$} & \multicolumn{4}{c}{$10$}  \\
\hline
With gradient ascent & \multicolumn{2}{c|}{$\circ$} & \multicolumn{2}{c|}{}  & \multicolumn{2}{c|}{$\circ$} & \multicolumn{2}{c|}{}  & \multicolumn{2}{c|}{$\circ$} & \multicolumn{2}{c|}{} & \multicolumn{2}{c|}{$\circ$} & \multicolumn{2}{c}{} \\
\hline
Inlier Class &             Mean &    SD &      Mean &      SD &      Mean &      SD  &      Mean &      SD &      Mean &      SD &      Mean &      SD &      Mean &      SD &      Mean &      SD \\
\hline
plane &  0.491 &  0.009 &  {\underline {0.642}} &  0.019 &  {\underline {\bf 0.934}} &  0.002 &  0.918 &  0.003 &  {\underline {0.920}} &  0.003 &  0.899 &  0.002 &  {\underline {0.886}} &  0.007 &  0.839 &  0.009 \\
car &  0.521 &  0.032 &  {\underline {0.644}} &  0.011 &  {\underline {\bf 0.957}} &  0.002 &  0.950 &  0.002 &  {\underline {0.951}} &  0.003 &  0.939 &  0.004 &  {\underline {0.920}} &  0.006 &  0.894 &  0.013 \\
bird &  0.501 &  0.013 &  {\underline {0.622}} &  0.012 &  {\underline {\bf 0.850}} &  0.007 &  0.832 &  0.004 &  {\underline {0.835}} &  0.005 &  0.812 &  0.006 &  {\underline {0.818}} &  0.004 &  0.765 &  0.010 \\
cat &  0.491 &  0.015 &  {\underline {0.616}} &  0.014 &  {\underline {\bf 0.820}} &  0.003 &  0.807 &  0.003 &  {\underline {0.802}} &  0.007 &  0.770 &  0.005 &  {\underline {0.773}} &  0.011 &  0.721 &  0.006 \\
deer &  0.523 &  0.017 &  {\underline {0.658}} &  0.022 &  {\underline {\bf 0.886}} &  0.004 &  0.879 &  0.001 &  {\underline {0.873}} &  0.005 &  0.862 &  0.004 &  {\underline {0.852}} &  0.007 &  0.820 &  0.009 \\
dog &  0.514 &  0.018 &  {\underline {0.621}} &  0.011 &  {\underline {\bf 0.875}} &  0.004 &  0.855 &  0.005 &  {\underline {0.852}} &  0.008 &  0.820 &  0.007 &  {\underline {0.821}} &  0.009 &  0.758 &  0.017 \\
frog &  0.496 &  0.018 &  {\underline {0.671}} &  0.018 &  {\underline {\bf 0.944}} &  0.003 &  0.932 &  0.003 &  {\underline {0.927}} &  0.003 &  0.917 &  0.005 &  {\underline {0.886}} &  0.004 &  0.845 &  0.014 \\
horse &  0.506 &  0.017 &  {\underline {0.631}} &  0.018 &  {\underline {\bf 0.928}} &  0.003 &  0.910 &  0.003 &  {\underline {0.916}} &  0.005 &  0.885 &  0.003 &  {\underline {0.880}} &  0.007 &  0.823 &  0.020 \\
ship &  0.494 &  0.027 &  {\underline {0.680}} &  0.026 &  {\underline {\bf 0.958}} &  0.003 &  0.949 &  0.001 &  {\underline {0.956}} &  0.002 &  0.942 &  0.002 &  {\underline {0.933}} &  0.004 &  0.907 &  0.006 \\
truck &  0.506 &  0.013 &  {\underline {0.660}} &  0.016 &  {\underline {\bf 0.939}} &  0.003 &  0.930 &  0.003 &  {\underline {0.922}} &  0.003 &  0.907 &  0.007 &  {\underline {0.885}} &  0.007 &  0.843 &  0.018 \\
\bottomrule
\end{tabular}}
\end{center}

\end{table}

\begin{table}
\caption{Average PD (Mean) with standard deviation (SD) over $10$ trials with different seeds per method. The best performing method in terms of the mean PD is specified by bold face.} 
\label{tbl:appdx:exp_covariate_shift}\begin{center}
\scalebox{0.72}[0.72]{
\begin{tabular}{l|rr|rr|rr|rr|rr}
\toprule
Domains (Train $\to$ Test)&      \multicolumn{2}{c|}{books $\to$ books} &      \multicolumn{2}{c|}{dvd $\to$ books} &      \multicolumn{2}{c|}{dvd $\to$ dvd} &     \multicolumn{2}{c|}{elec $\to$ books} &      \multicolumn{2}{c}{elec $\to$ dvd}  \\
\hline
DRE method &      Mean &      SD &      Mean &      SD &      Mean &      SD &     Mean &      SD & Mean & SD \\
\hline
w/o IW &  0.093 &  0.003 &  0.128 &  0.008 &  0.100 &  0.005 &  0.212 &  0.012 &  0.187 &  0.008 \\
Kernel uLSIF &  0.089 &  0.002 &  0.114 &  0.006 &  0.094 &  0.004 &  0.200 &  0.009 &  0.179 &  0.006 \\
Kernel KLIEP &  0.089 &  0.002 &  0.116 &  0.006 &  0.094 &  0.004 &  0.205 &  0.011 &  0.184 &  0.008 \\
uLSIF-NN &  0.093 &  0.003 &  0.128 &  0.008 &  0.100 &  0.005 &  0.212 &  0.012 &  0.187 &  0.008 \\
PU-NN &  0.093 &  0.003 &  0.128 &  0.008 &  0.100 &  0.005 &  0.212 &  0.012 &  0.187 &  0.008 \\
nnBD-LSIF &  {\bf 0.086} &  0.002 &  {\bf 0.113} &  0.005 &  {\bf 0.091} &  0.004 &  {\bf 0.199} &  0.009 &  {\bf 0.176} &  0.005 \\
nnBD-PU &  0.090 &  0.003 &  {\bf 0.113} &  0.006 &  0.096 &  0.004 &  {\bf 0.199} &  0.009 &  {\bf 0.176} &  0.006 \\
\bottomrule
\end{tabular}}
\end{center}

\begin{center}
\scalebox{0.72}[0.72]{
\begin{tabular}{l|rr|rr|rr|rr|rr}
\toprule
Domains (Train $\to$ Test)&      \multicolumn{2}{c|}{elec $\to$ elec} &      \multicolumn{2}{c|}{kitchen $\to$ books} &      \multicolumn{2}{c|}{kitchen $\to$ dvd} &     \multicolumn{2}{c|}{kitchen $\to$ elec} &      \multicolumn{2}{c}{kitchen $\to$ kitchen}  \\
\hline
DRE method &      Mean &      SD &      Mean &      SD &      Mean &      SD &     Mean &      SD &      Mean &      SD  \\
\hline
w/o IW &  0.079 &  0.005 &  0.202 &  0.013 &  0.185 &  0.006 &  0.073 &  0.004 &  0.062 &  0.002 \\
Kernel uLSIF &  0.072 &  0.003 &  0.192 &  0.007 &  0.178 &  0.008 &  0.071 &  0.003 &  0.060 &  0.003 \\
Kernel KLIEP &  0.072 &  0.003 &  0.195 &  0.005 &  0.182 &  0.007 &  0.072 &  0.004 &  0.060 &  0.002 \\
uLSIF-NN &  0.079 &  0.005 &  0.202 &  0.013 &  0.185 &  0.006 &  0.073 &  0.004 &  0.062 &  0.002 \\
PU-NN &  0.079 &  0.005 &  0.202 &  0.013 &  0.185 &  0.006 &  0.073 &  0.004 &  0.062 &  0.002 \\
nnBD-LSIF &  {\bf 0.071} &  0.003 &  {\bf 0.189} &  0.008 &  {\bf 0.174} &  0.008 &  {\bf 0.068} &  0.003 &  {\bf 0.058} &  0.003 \\
nnBD-PU &  0.074 &  0.004 &  0.190 &  0.008 &  {\bf 0.174} &  0.008 &  {\bf 0.068} &  0.003 &  0.062 &  0.005 \\
\bottomrule
\end{tabular}}
\end{center}
\end{table}

\section{Other applications}
\label{appdx:other_appl}
In this section, we explain other potential applications of the proposed method.

\subsection{Covariate shift adaptation by importance weighting}
We consider training a model using input distribution different from the test input distribution, which is called \emph{covariate shift}, \citep{Bickel2009}. To solve this problem, the density ratio has been used via importance weighting (IW) \citep{shimodaira2000improving,yamada2010,Reddi2015}. 

We use a document dataset of Amazon\footnote{\url{http://john.blitzer.com/software.html}} \citep{blitzer-etal-2007-biographies} for multi-domain sentiment analysis \citep{blitzer-etal-2007-biographies}. This data consists of text reviews from four different product domains: book, electronics  (elec), dvd, and kitchen. Following \citet{10.5555/3042573.3042781} and \citet{pmlr-v48-menon16}, we transform the text data using TF-IDF to map them into the instance space $\mathcal{X} = \mathbb{R}^{10000}$ \citep{salton1986introduction}. Each review is endowed with four labels indicating the positivity of the review, and our goal is to conduct regression for these labels. To achieve this goal, we perform kernel ridge regression with the polynomial kernel. We compare regression without IW (w/o IW) with regression using the density ratio estimated by PU-NN, uLSIF-NN, nnBD-LSIF, nnBD-PU, uLSIF with Gaussian kernels (Kernel uLSIF), and KLIEP with Gaussian kernels (Kernel KLIEP). We conduct experiments on $2,000$ samples from one domain, and test $2,000$ samples. Following \citet{pmlr-v48-menon16}, we reduce the dimension into $100$ dimensions by principal component analysis when using Kernel uLSIF, Kernel KLEIP, and regressions. Following \citet{pmlr-v48-menon16} and \citet{Cortes2011regression}, the mean and standard deviation of the pairwise disagreement (PD), $1-\mathrm{AUROC}$, is reported. A part of results is in Table~\ref{tbl:exp_covariate_shift}. The full results are in Appendix~\ref{appdx:sec:cov_adapt}. The methods with D3RE show preferable performance, but the improvement is not significant compared with the image data. We consider this is owing to the difficulty of the covariate shift problem in this dataset.

\begin{table}
\caption{Average PD (Mean) with standard deviation (SD) over $10$ trials with different seeds per method. The best performing method in terms of the mean PD is specified by bold face.} 
\label{tbl:exp_covariate_shift}
\vspace{-0.2cm}
\begin{center}
\scalebox{0.68}[0.68]{
\begin{tabular}{l|rr|rr|rr|rr|rr|rr}
\toprule
Domains (Train $\to$ Test)&      \multicolumn{2}{c|}{book $\to$ dvd} &      \multicolumn{2}{c|}{book $\to$ elec} &      \multicolumn{2}{c|}{book $\to$ kitchen} &     \multicolumn{2}{c|}{dvd $\to$ elec} &      \multicolumn{2}{c|}{dvd $\to$ kitchen} &     \multicolumn{2}{c}{elec $\to$ kitchen} \\
\hline
DRE method &      Mean &      SD &      Mean &      SD &      Mean &      SD &     Mean &      SD &      Mean &      SD &      Mean &      SD \\
\hline
w/o IW &  0.126 &  0.008 &  0.174 &  0.010 &  0.166 &  0.009 &  0.162 &  0.006 &  0.146 &  0.010 &  0.074 &  0.005 \\
Kernel uLSIF &  0.122 &  0.009 &  0.162 &  0.009 &  0.159 &  0.007 &  0.153 &  0.006 &  0.142 &  0.007 &  0.068 &  0.005 \\
Kernel KLIEP &  0.130 &  0.010 &  0.164 &  0.009 &  0.161 &  0.007 &  0.154 &  0.006 &  0.143 &  0.006 &  0.070 &  0.005 \\
uLSIF-NN &  0.126 &  0.008 &  0.174 &  0.010 &  0.166 &  0.009 &  0.162 &  0.006 &  0.146 &  0.010 &  0.074 &  0.005 \\
PU-NN &  0.126 &  0.008 &  0.174 &  0.010 &  0.166 &  0.009 &  0.162 &  0.006 &  0.146 &  0.010 &  0.074 &  0.005 \\
nnBD-LSIF &  0.120 &  0.008 &  {\bf 0.160} &  0.008 &  0.157 &  0.008 &  {\bf 0.148} &  0.006 &  {\bf 0.138} &  0.007 &  {\bf 0.066} &  0.005 \\
nnBD-PU &  {\bf 0.119} &  0.008 &  {\bf 0.160} &  0.008 &  {\bf 0.156} &  0.007 &  {\bf 0.148} &  0.005 &  {\bf 0.138} &  0.007 &  {\bf 0.066} &  0.005 \\
\bottomrule
\end{tabular}}
\end{center}
\vspace{-0.5cm}
\end{table}

\paragraph{$f$-divergence estimation.} $f$-divergences \citep{Ali1966,Csiszar1967} are the discrepancy measures of probability densities based on the density ratio, hence the proposed method can be used for their estimation. They include the KL divergence \citep{Kullback51klDivergence}, the Hellinger distance \citep{Hellinger1909}, and the Pearson divergence \citep{pearson1900}, as examples.

\paragraph{Two-sample homogeneity test.} The purpose of a homogeneity test is to determine if two or more datasets come from the same distribution \citep{LoevingerJane1948Ttoh}.
For two-sample testing, using
a semiparametric $f$-divergence estimator with nonparametric density ratio models has been studied \citep{Keziou2003,keziou2005}. \citet{kanamori2010} and \citet{sugiyama2011a} employed direct DRE for the nonparametric DRE.

\paragraph{Generative adversarial networks.} Generative adversarial networks (GANs) are successful deep generative models, which learns to generate new data with the same distribution as the training data \citet{Goodfellow2014}. Various GAN methods have been proposed, amongst which \citet{Nowozin} proposed \emph{f-GAN}, which minimizes the variational estimate of $f$-divergence. \citet{den2} extended the idea of \citet{Nowozin} to use BD minimization for DRE. The estimator proposed in this paper also has a potential to improve the method of \citet{den2}.

\paragraph{Average treatment effect estimation and off-policy evaluation.} 
One of the goals in causal inference is to estimate the expected treatment effect, which is a \emph{counterfactual} value. Therefore, following the causality formulated by \citet{rubin1974}, we consider estimating the average treatment effect (ATE). Recently, from machine learning community, off-policy evaluation (OPE) is also proposed, which is a generalization of ATE \citep{dudik2011doubly,ImaiKosuke2014Cbps,wang2017optimal,narita2018,pmlr-v97-bibaut19a,Kallus2019IntrinsicallyES,Oberst2019}. OPE has garnered attention in applications such as advertisement design selection, personalized medicine, search engines, and recommendation systems \citep{kdd2009_ads, www2010_cb, AtheySusan2017EPL}. 

The problem in ATE estimation and OPE is sample selection bias. For removing the bias, the density ratio has a critical role. An idea of using the density ratio dates back to \citep{rosenbaum87}, which proposed an inverse probability weighting (IPW) method \citep{Horvitz1952} for ATE estimation. In the IPW method, we approximate the parameter of interest with the sample average with inverse assignment probability of treatment (action), which is also called propensity score. Here, it is known that using the true assignment probability yields higher variance than the case where we use an estimated assignment probability even if we know the true value \citep{hirano2003efficient,Henmi2004paradox,Henmi2007imp}. This property can be explained from the viewpoint of semiparametric efficiency \citep{bickel98}. While the asymptotic variance of the IPW estimator with an estimated propensity score can achieve the efficiency bound, that of the IPW estimator with the true propensity score does not. 

By extending the IPW estimator, more robust ATE estimators are proposed by \citet{rosen83}, which is known as a doubly robust (DR) estimator. The doubly robust estimator is not only robust to model misspecification but also useful in showing asymptotic normality. In particular, when using the density ratio and the other nuisance parameters estimated from the machine learning method, the conventional IPW and DR estimators do not have asymptotic normality \citep{ChernozhukovVictor2018Dmlf}. This is because the nuisance estimators do not satisfy Donsker's condition, which is required for showing the asymptotic normality of semiparametric models. However, by using the sample splitting method proposed by \citet{klaassen1987}, \citet{ZhengWenjing2011CTME}, and \citet{ChernozhukovVictor2018Dmlf}, we can show the asymptotic normality when using the DR estimator. Note that for the IPW estimator, we cannot show the asymptotic normality even if using sample-splitting. 

When using the IPW and DR estimator, we often consider a two-stage approach: in the first stage, we estimate the nuisance parameters, including the density ratio; in the second stage, we construct a semiparametric ATE estimator including the first-stage nuisance estimators. This is also called two-step generalized method of moments (GMM). On the other hand, from the causal inference community, there are also weighting-based covariate balancing methods \citep{qin2007empiricallikelihood,tan2010,Hainmueller2012,ImaiKosuke2014Cbps}. In particular, \citet{ImaiKosuke2014Cbps} proposed a covariate balancing propensity score (CBPS), which simultaneously estimates the density ratio and ATE. The idea of CBPS is to construct moment conditions, including the density ratios, and estimate the ATE and density ratio via GMM simultaneously. Although the asymptotic property of the CBPS is the same as other conventional estimators, existing empirical studies report that the CBPS outperforms them \citep{Wyss}. 

Readers may feel that the CBPS \citep{ImaiKosuke2014Cbps} has a close relationship with the direct DRE, but we consider that it is less relevant to the context of the direct DRE. From the DRE perspective, the method of \citet{ImaiKosuke2014Cbps} boils down to the method of \citet{gretton2009}, which proposed direct DRE through moment matching. The research motivation of \citet{ImaiKosuke2014Cbps} is to estimate the ATE with estimating a nuisance density ratio estimator simultaneously. Therefore, the density ratio itself is \emph{nuisance} parameter; that is, they are not interested in the estimation performance of the density ratio. Under their motivation, they are interested in a density ratio estimator satisfying the moment condition for estimating the ATE, not in a density ratio estimator predicting the true density ratio well. In addition, while the direct DRE method adopts linear-in-parameter models and neural networks (our work), it is not appropriate to use those methods with the CBPS \citep{ChernozhukovVictor2018Dmlf}. This is because the density ratio estimator does not satisfy Donsker's condition. Even naive Ridge and Lasso regression estimators do not satisfy the Donsker's condition. Therefore, when using machine learning methods for estimating the density ratio, we cannot show asymptotic normality of an ATE estimator obtained by the CBPS; therefore, we need to use the sample-splitting method by \citep{ChernozhukovVictor2018Dmlf}. This means that when using the CBPS, we can only use a naive parametric linear model without regularization or classic nonparametric kernel regression. Recently, for GMM with such non-Donsker nuisance estimators, \citet{chernozhukov2016} also proposed a new GMM method based on the conventional two-step approach. For these reasons, the CBPS is less relevant to the direct DRE context.

\paragraph{Off-policy evaluation with external validity.} By the problem setting of combining causal inference and domain adaptation, \citet{kato_uehara_2020} recently proposed using covariate shift adaptation to solve the \emph{external validity} problem in OPE, i.e., the case that the distribution of covariates is the same between the historical and evaluation data \citep{ColeStephenR.2010GEFR,PearlJudea2015EVFD}. 

\paragraph{Change point detection.} The methods for \emph{change-point detection} try to detect abrupt changes in time-series data \citep{basseville1993detection,brodsky1993nonparametric,Gustafsson2000,nguyen2011positive}. There are two types of problem settings in change-point detection, namely the real-time detection \citep{Adams07bayesianonline,Garnett2009,Paquet2007} and the retrospective detection \citep{basseville1993detection,Yamanishi2002}. In retrospective detection, which requires longer reaction periods, \citet{LiuSong2012} proposed using techniques of direct DRE. Whereas the existing methods rely on linear-in-parameter models, our proposed method enables us to employ more complex models for change point detection.

\paragraph{Similarity-based sentiment analysis.} \citet{kato2019pufinance} used the density ratio estimated from PU learning for sentiment analysis of text data based on similarity.

\section{Generalization error bound}
\label{section:appendix:estimation-error-bound}
The generalization error bound can be proved by building upon the proof techniques in \citet{KiryoPositiveunlabeled2017,LuMitigating2020}.
\paragraph{Notations for the theoretical analysis.}
We denote the set of real values by \(\Re\) and that of positive integers by \(\Na\).
Let \(\InSpace \subset \Re^d\).
Let \(\pnu(x)\) and \(\pde(x)\) be probability density functions over \(\InSpace\), and assume that the density ratio \(\rstar(x) := \frac{\pnu(x)}{\pde(x)}\) is existent and bounded: \(\rmax := \|\rstar\|_\infty < \infty\).
Assume \(0 < \Cons < \frac{1}{\rmax}\). Since \(\rmax \geq 1\) (because \(1 = \int \pde(x) \rstar(x) dx \leq 1 \cdot \|\rstar\|_\infty\)), we have \(\Cons \in (0, 1]\) and hence \(\pmod := \pde - \Cons\pnu > 0\).

\paragraph{Problem Setup.}
Let the hypothesis class of density ratio be \(\rClass \subset \{r: \Re^D \to \rClassRangeTwo =: \rClassRangeTwoName\}\), where \(0 \leq \rClassBoundMin < \rmax < \rClassBound\).
Let \(\br: \rClassRangeTwoName \to \Re\) be a twice continuously-differentiable convex function with a bounded derivative.
Define \(\tbr\) by \(\dbr(t) = \Cons (\dbr(t) t - \br(t)) + \tbr(t)\), where \(\dbr\) is the derivative of \(\br\) continuously extended to \(0\) and \(\rClassBound\).
Recall the definitions \(\lossOne(t) := \dbr(t) t - \br(t) + A\), \(\lossTwo(t) := - \tbr(t)\), and
\begin{equation*}\begin{split}
\Risk(\r) &:= \Ede\left[\dbr(\r(X))\r(X) - \br(\r(X)) + A\right] - \Enu\left[\dbr(\r(X))\right] \\
&= \Emod\left[\dbr(\r(X))\r(X) - \br(\r(X)) + A\right] - \Enu\left[\tbr(\r(X))\right] \\
&= \Emod\lossOne(r(X)) + \Enu\lossTwo(r(X)) \\
&\left(= (\Ede - \Cons\Enu)\lossOne(r(X)) + \Enu\lossTwo(r(X))\right), \\
\nnhRisk(\r) &:= \modFn\left(\hEmod\lossOne(r(X))\right) + \hEnu\lossTwo(r(X)) \\
&\left(= \modFn((\hEde - \Cons\hEnu)\lossOne(r(X)) + \hEnu\lossTwo(r(X))\right), \\
\end{split}\end{equation*}
where we denoted \(\hEmod = \hEde - \Cons\hEnu\) and \(\modFn\) is a consistent correction function with Lipschitz constant \(\LipmodFn\) (Definition~\ref{def:consistent-correction}).

\begin{remark}
The true density ratio \(\rstar\) minimizes \(\Risk\).
\end{remark}

\begin{definition}[Consistent correction function \cite{LuMitigating2020}]
A function \(f : \Re \to \Re\) is called a consistent correction function if it is Lipschitz continuous, non-negative and \(f(x) = x\) for all \(x \geq 0\).
\label{def:consistent-correction}
\end{definition}

\begin{definition}[Rademacher complexity]
Given \(n \in \mathbb{N}\) and a distribution \(p\), define the Rademacher complexity \(\Radnp(\rClass)\) of a function class \(\rClass\) as
\begin{equation*}\begin{split}
\Radnp(\rClass) := \E_p\ERad\left[\supr\left|\frac{1}{n} \sum_{i=1}^n \rad_i \r(\X_i)\right|\right],
\end{split}\end{equation*}
where \(\{\sigma_i\}_{i=1}^n\) are Rademacher variables (i.e., independent variables following the uniform distribution over \(\{-1, +1\}\)) and \({\{\X_i\}_{i=1}^n \overset{\text{i.i.d.}}{\sim} p}\).
\label{def:rademacher-complexity}
\end{definition}
The theorem in the paper is a special case of Theorem~\ref{thm:estimation-error-bound} with \(\modFn(\cdot) := \max\{0, \cdot\}\) (in which case \(\LipmodFn = 1\)) and Theorem~\ref{thm:main-text:estimation-error-bound}.
\begin{theorem}[Generalization error bound]
Assume that \(\LossBound := \supt\{\max\{|\lossOne(t)|, |\lossTwo(t)|\}\} < \infty\).
Assume \(\lossOne\) is \(\LipOne\)-Lipschitz and \(\lossTwo\) is \(\LipTwo\)-Lipschitz.
Assume that there exists an empirical risk minimizer \(\hr \in \argmin_{\r \in \rClass} \nnhRisk(\r)\)
and a population risk minimizer \(\rbest \in \argmin_{\r \in \rClass} \Risk(\r)\).
Also assume \(\infr \Emod\lOneR > 0\) and that \((\modFn - \Identity)\) is \((\modFnIdLip)\)-Lipschitz.
Then for any \(\delta \in (0, 1)\), with probability at least \(1 - \delta\), we have
\begin{equation*}\begin{split}
\Risk(\hr) - \Risk(\rbest) &\leq 8 \LipmodFn\LipOne\Radde(\rClass) + 8 (\LipmodFn\Cons\LipOne + \LipTwo)\Radnu(\rClass) \\
&\qquad+ 2 \BiasTerm + \LossBound \sqrt{8 \left(\frac{\LipmodFn^2}{\nde}+ \frac{(1 + \LipmodFn \Cons)^2}{\nnu}\right) \log\frac{1}{\delta}},
\end{split}\end{equation*}
where \(\BiasTerm\) is defined as in Lemma~\ref{lem:bias}.
\label{thm:estimation-error-bound}
\end{theorem}
\begin{proof}
Since \(\hr\) minimizes \(\nnhRisk\), we have
\begin{equation*}\begin{split}
\Risk(\hr) - \Risk(\rbest) &= \Risk(\hr) - \nnhRisk(\hr) + \nnhRisk(\hr) - \Risk(\rbest)\\
&\leq \Risk(\hr) - \nnhRisk(\hr) + \nnhRisk(\rbest) - \Risk(\rbest) \\
&\leq 2 \sup_{\r \in \rClass} |\nnhRisk(\r) - \Risk(\r)| \\
&\leq \annot{2 \sup_{\r \in \rClass} |\nnhRisk(\r) - \EnnhRisk(\r)|}{Maximal deviation} + \annot{2 \sup_{\r \in \rClass} |\EnnhRisk(\r) - \Risk(\r)|}{Bias}.
\end{split}\end{equation*}
We apply McDiarmid's inequality \cite{McDiarmidmethod1989,MohriFoundations2018} to the maximal deviation term.
The absolute value of the difference caused by altering one data point in the maximal deviation term is bounded from above by \(2 \LossBound \frac{\LipmodFn}{\nde}\) if the altered point is a sample from \(\pde\) and \(2 \LossBound \frac{1 + \LipmodFn \Cons}{\nnu}\) if it is from \(\pnu\).
Therefore, McDiarmid's inequality implies, with probability at least \(1 - \delta\), that we have
\begin{equation*}\begin{aligned}
&\sup_{\r \in \rClass} |\nnhRisk(\r) - \EnnhRisk(\r)| \\
&\quad\leq \annot{\mathbb{E} \left[\sup_{\r \in \rClass} |\nnhRisk(\r) - \EnnhRisk(\r)|\right]}{Expected maximal deviation} + \LossBound \sqrt{2 \left(\frac{\LipmodFn^2}{\nde}+ \frac{(1 + \LipmodFn \Cons)^2}{\nnu}\right) \log\frac{1}{\delta}}.
\end{aligned}\end{equation*}

Applying Lemma~\ref{lem:general-symmetrization} to the expected maximal deviation term and Lemma~\ref{lem:bias} to the bias term, we obtain the assertion.
\end{proof}
The following lemma generalizes the symmetrization lemmas proved in \citet{KiryoPositiveunlabeled2017} and \citet{LuMitigating2020}.
\begin{lemma}[Symmetrization under Lipschitz-continuous modification]
Let \(0 \leq a < b\), \(J \in \Na\), and \(\{K_j\}_{j=1}^J \subset \Na\).
Given i.i.d. samples \(\Xsetjk := \{X_i\}_{i=1}^{\njk}\) each from a distribution \(\pjk\) over \(\InSpace\),
consider a stochastic process \(\hS\) indexed by \(\F \subset (a, b)^\mathcal{X}\) of the form
\begin{equation*}\begin{split}
\hS(\f) = \sum_{j=1}^J \rho_j\left(\sumjk\hEjk [\ljk(\f(\X))]\right),
\end{split}\end{equation*}
where each \(\rho_j\) is a \(\Liprhoj\)-Lipschitz function on \(\Re\),
\(\ljk\) is a \(\LipLossjk\)-Lipschitz function on \((a, b)\),
and \(\hEjk\) denotes the expectation with respect to the empirical measure of \(\Xsetjk\).
Denote \(\EhS(\f) := \EX \hS(\f)\) where \(\EX\) is the expectation with respect to the product measure of \(\{\Xsetjk\}_{(j, k)}\).
Here, the index \(j\) denotes the grouping of terms due to \(\rho_j\), and \(k\) denotes each sample average term.
Then we have
\begin{equation*}\begin{split}
\EX \supf |\hS(\f) - \EhS(\f)| \leq 4 \sumj \sumjk \Liprhoj \LipLossjk \Radjk(\F).
\end{split}\end{equation*}
\label{lem:general-symmetrization}
\end{lemma}

\begin{proof}
First, we consider a continuous extension of \(\ljk\) defined on \((a, b)\) to \([0, b)\). Since the functions in \(\F\) take values only in \((a, b)\), this extension can be performed without affecting the values of \(\hS(f)\) or \(\EhS(f)\).
We extend the function by defining the values for \(x \in [0, a]\) as \(\ljk(x) := \lim_{x' \downarrow a} \ljk(x')\), where the right-hand side is guaranteed to exist since \(\ljk\) is Lipschitz continuous hence uniformly continuous. Then, \(\ljk\) remains a \(\Liprhoj\)-Lipschitz continuous function on \([0, b)\).
Now we perform symmetrization \citep{VapnikStatistical1998}, deal with \(\rho_j\)'s, and then bound the symmetrized process by Rademacher complexity.
Denoting independent copies of \(\{\X_{(j, k)}\}\) by \(\{\Xghost_{j, k}\}_{(j, k)}\) and the corresponding expectations as well as the sample averages with \(^{\ghostMark}\),
\begin{equation*}\begin{split}
&\EX \supf |\hS(\f) - \EhS(\f)| \\
&\leq \sumj \EX \supf |\rho_j(\sumjk \hEjk \ljk(\f(\X))) - \EXghost \rho_j(\sumjk \hEghostjk \ljk(\f(\Xghost)))| \\
&\leq \sumj \EX \EXghost \supf |\rho_j(\sumjk \hEjk \ljk(\f(\X))) - \rho_j(\sumjk \hEghostjk \ljk(\f(\Xghost)))| \\
&\leq \sumj \Liprhoj \sumjk \EX \EXghost \supf |\hEjk \ljk(\f(\X)) - \hEghostjk \ljk(\f(\Xghost))| \\
&= \sumj \Liprhoj \sumjk \EX \EXghost \supf |\hEjk (\ljk(\f(\X)) - \ljk(0)) - \hEghostjk (\ljk(\f(\Xghost)) - \ljk(0))| \\
&\leq \sumj \Liprhoj \sumjk \left(2 \Radjk(\{\ljk \circ \f - \ljk(0): \f \in \F\})\right) \\
&\leq \sumj \Liprhoj \sumjk 2 \cdot 2 \LipLossjk \Radjk(\F),
\end{split}\end{equation*}
where we applied Talagrand's contraction lemma for two-sided Rademacher complexity \citep{LedouxProbability1991,BartlettRademacher2001} with respect to \((t \mapsto \ljk(t) - \ljk(0))\) in the last inequality.
\end{proof}
\begin{lemma}[Bias due to risk correction]
Assume \(\infr \Emod\lOneR > 0\) and that \((\modFn - \Identity)\) is \((\modFnIdLip)\)-Lipschitz on \(\Re\).
There exists \(\alpha > 0\) such that
\begin{equation*}\begin{split}
\sup_{\r \in \rClass} |\EnnhRisk(\r) - \Risk(\r)| &\leq (1+\Cons)\LossBound\modFnIdLip \exp\left(- \frac{2 \alpha^2}{(\LossBound^2/\nde) + (\Cons^2\LossBound^2 / \nnu)}\right) \\
&=: \BiasTerm.
\end{split}\end{equation*}
\label{lem:bias}
\end{lemma}
\begin{remark}
Note that we already have \(\pmod \geq 0\) and \(\lossOne \geq 0\) and hence \(\infr \Emod\lOneR \geq 0\). Therefore, the assumption of Lemma~\ref{lem:bias} is essentially referring to the strict positivity of the infimum.
Here, \(\Emod\) and \(\Probability{\cdot}\) denote the expectation and the probability with respect to the joint distribution of the samples included in \(\hEmod\).
\end{remark}
\begin{proof}
Fix an arbitrary \(\r \in \rClass\). We have
\begin{equation*}\begin{aligned}
&|\EnnhRisk(\r) - \Risk(\r)| = |\E[\nnhRisk(\r) - \hRisk(\r)]| \\
&= |\E[\modFn(\hEmod\lOneR) - \hEmod\lOneR]| 
\leq \E \left[|\modFn(\hEmod\lOneR) - \hEmod\lOneR|\right] \\
&= \E \left[\Indicator{\modFn(\hEmod\lOneR) \neq \hEmod\lOneR}\cdot|\modFn(\hEmod\lOneR) - \hEmod\lOneR|\right] \\
&\leq \E\left[\Indicator{\modFn(\hEmod\lOneR) \neq \hEmod\lOneR}\right]\left(\supLossVal|\modFn(s) - s|\right) \\
\end{aligned}\end{equation*}
where \(\Indicator{\cdot}\) denotes the indicator function, and we used \(|\hEmod\lOneR| \leq (1 + C)\LossBound\). Further, we have
\begin{equation*}\begin{aligned}
&\supLossVal|\modFn(s) - s| \leq \supLossVal|(\modFn - \Identity)(s) - (\modFn - \Identity)(0)| + |(\modFn - \Identity)(0)| \\
&\leq \supLossVal \modFnIdLip |s - 0| + 0
\leq (1+\Cons)\LossBound\modFnIdLip,
\end{aligned}\end{equation*}
where \(\Identity\) denotes the identity function.
On the other hand, since \(\infr \Emod\lOneR > 0\) is assumed, there exists \(\alpha > 0\) such that for any \(\r \in \rClass\), \(\Emod\lOneR > \alpha\).
Therefore, denoting the support of a function by \(\supp(\cdot)\),
\begin{equation*}\begin{aligned}
&\E\left[\Indicator{\modFn(\hEmod\lOneR) \neq \hEmod\lOneR}\right] = \Probability{\hEmod\lOneR \in \modFnIdSupp} \\
&\leq \Probability{\hEmod\lOneR < 0}
\leq \Probability{\hEmod\lOneR < \Emod\lOneR - \alpha}
\end{aligned}\end{equation*}
holds.
Now we apply McDiarmid's inequality to the right-most quantity.
The absolute difference caused by altering one data point in \(\hEmod\lOneR\) is bounded by \(\frac{\LossBound}{\nde}\) if the change is in a sample from \(\pde\) and \(\frac{\Cons \LossBound}{\nnu}\) otherwise.
Therefore, McDiarmid's inequality implies
\begin{equation*}\begin{split}
\Probability{\hEmod\lOneR < \Emod\lOneR - \alpha} \leq \exp\left(- \frac{2 \alpha^2}{(\LossBound^2/\nde) + (\Cons^2\LossBound^2 / \nnu)}\right). \\
\end{split}\end{equation*}
\end{proof}

\begin{theorem}[Generalization error bound]
Under Assumption~\ref{assumption:main:est-error-bound}, for any \(\delta \in (0, 1)\), with probability at least \(1 - \delta\), we have $\Risk(\hr) - \Risk(\rbest) \leq
\LipOne\Radde(\rClass) + 8 (\Cons\LipOne + \LipTwo)\Radnu(\rClass) + 2\BiasTermMainText + \LossBound \sqrt{8 \left(\frac{1}{\nde}+ \frac{(1 + \Cons)^2}{\nnu}\right) \log\frac{1}{\delta}}$, where $\BiasTermMainText := (1+\Cons)\LossBound \exp\left(- \frac{2 \alpha^2}{(\LossBound^2/\nde) + (\Cons^2\LossBound^2 / \nnu)}\right)$ and \(\alpha > 0\) is a constant determined in the proof of Lemma~\ref{lem:bias} in Appendix~\ref{section:appendix:estimation-error-bound}.
\label{thm:main-text:estimation-error-bound}
\end{theorem}

\begin{remark}[Explicit form of the bound in Theorem~\ref{thm:est_error_bound}]
\label{rem:explicit}
Here, we show the explicit form of the bound in Theorem~\ref{thm:est_error_bound} as follows:
\begin{equation*}\begin{aligned}
&\Risk(\hr) - \Risk(\rbest)\\
&\leq \frac{\kappa_1}{\sqrt{\nde}} + \frac{\kappa_2}{\sqrt{\nnu}} + 2 \BiasTermMainText + \LossBound \sqrt{8 \left(\frac{1}{\nde}+ \frac{(1 + \Cons)^2}{\nnu}\right) \log\frac{1}{\delta}}\\
&=\LipOne\frac{B_{p_{\mathrm{de}}}\left(\sqrt{2 \log(2) L} + 1\right) \prod_{j=1}^L \ParamBoundj}{\sqrt{\nde}}\\
&\ \ \ + 8(C\LipOne + \LipTwo)\frac{B_{p_{\mathrm{nu}}}\left(\sqrt{2 \log(2) L} + 1\right) \prod_{j=1}^L \ParamBoundj}{\sqrt{\nnu}}\\
&\ \ \ + 2(1+\Cons)\LossBound \exp\left(- \frac{2 \alpha^2}{(\LossBound^2/\nde) + (\Cons^2\LossBound^2 / \nnu)}\right)\\
&\ \ \ + \LossBound \sqrt{8 \left(\frac{1}{\nde}+ \frac{(1 + \Cons)^2}{\nnu}\right) \log\frac{1}{\delta}}.
\end{aligned}\end{equation*}
\end{remark}

\section{Rademacher complexity bound}
\label{appdx:rad_comp_bound}
The following lemma provides an upper-bound on the Rademacher complexity for multi-layer perceptron models in terms of the Frobenius norms of the parameter matrices.
Alternatively, other approaches to bound the Rademacher complexity can be employed.
The assertion of the lemma follows immediately from the proof of Theorem~1 of \citet{GolowichSizeIndependent2019} after a slight modification to incorporate the absolute value function in the definition of Rademacher complexity.
\begin{lemma}[Rademacher complexity bound {\citep[Theorem~1]{GolowichSizeIndependent2019}}]
Assume the distribution \(p\) has a bounded support: \(\InputBound := \sup_{x \in \supp(p)} \|x\| < \infty\).
Let \(\rClass\) be the class of real-valued networks of depth \(L\) over the domain \(\mathcal{X}\),
where each parameter matrix \(W_j\) has Frobenius norm at most \(\ParamBoundj \geq 0\), and with \(1\)-Lipschitz activation functions \(\varphi_j\) which are positive-homogeneous (i.e., \(\varphi_j\) is applied element-wise and \(\varphi_j(\alpha t) = \alpha \varphi_j(t)\) for all \(\alpha \geq 0\)).
Then
\begin{equation*}\begin{split}
\Radnp(\rClass) \leq \frac{\InputBound\left(\sqrt{2 \log(2) L} + 1\right) \prod_{j=1}^L \ParamBoundj}{\sqrt{n}}.
\end{split}\end{equation*}
\label{lem:rademacher-bound-golowich}
\label{lem:main-text:rademacher-bound-golowich}
\end{lemma}
\begin{proof}
The assertion immediately follows once we modify the beginning of the proof of Theorem~1 by introducing the absolute value function inside the supremum of the Rademacher complexity as
\begin{equation*}\begin{split}
\ERad \left[\supr \left|\sum_{i=1}^n \rad_i \r(x_i)\right|\right] \leq \frac{1}{\lambda} \log \ERad \supr \exp\left(\lambda \left|\sum_{i=1} \rad_i \r(x_i) \right|\right).
\end{split}\end{equation*}
for \(\lambda > 0\).
The rest of the proof is identical to that of Theorem~1 of \citet{GolowichSizeIndependent2019}.
\end{proof}

\section{Proof of Theorem~\ref{appdx:thm:convergence-rate}}
\label{appdx:l2norm}
We consider relating the \(L^2\) error bound to the BD generalization error bound in the following lemma.
\begin{lemma}[\(\Ltwo\) distance bound]
\label{appdx:sec:appendix:strong-convexity}
Let \(\mathcal{H} := \{\r: \InSpace \to \rClassRangeTwo =: \rClassRangeTwoName | \int |r(x)|^2 \dx < \infty\}\) and assume \(\rstar \in \mathcal{H}\).
If \(\inf_{t \in \rClassRangeTwoName} \br''(t) > 0\), then there exists \(\mu > 0\) such that for all \(\r \in \mathcal{H}\),
\begin{equation*}\begin{split}
\|\r - \rstar\|_{\Ltwo(\pde)}^2 \leq \frac{2}{\mu}\left(\Risk(\r) - \Risk(\rstar)\right)
\end{split}\end{equation*}
holds.
\end{lemma}
\begin{proof}
Since \(\mu := \inf_{t \in \rClassRangeTwoName} \br''(t) > 0\), the function \(\br\) is \(\mu\)-strongly convex.
By the definition of strong convexity,
\begin{equation*}\begin{aligned}
&\Risk(\r) - \Risk(\rstar) = (\Risk(\r) - \Ede\br(\rstar(X))) - \annot{(\Risk(\rstar) + \Ede\br(\rstar(X)))}{\(=0\)}\\
&= \Ede\left[\br(\rstar(X)) - \br(\r(X)) + \dbr(\r(X))(\rstar(X) - \r(X))\right] \\
&\geq \Ede\left[\frac{\mu}{2}(\rstar(X) - \r(X))^2\right]
= \frac{\mu}{2} \|\rstar - \r\|_{\Ltwo(\pde)}^2.
\end{aligned}\end{equation*}
\end{proof}

\begin{lemma}[\(\ltwo\) distance bound]
Fix \(r \in \rClass\). Given \(n\) samples \(\{x_i\}_{i=1}^n\) from \(\pde\), with probability at least \(1 - \delta\), we have
\begin{equation*}\begin{split}
\frac{1}{n}\sum_{i=1}^n (\r(x_i) - \rstar(x_i))^2 \leq \annot{\E\left[(\r - \rstar)^2(X)\right]}{\(=\|\r - \rstar\|_{\Ltwo(\pde)}^2\)} + (2 \rmax)^2\sqrt{\frac{\log\frac{1}{\delta}}{2 n}}.
\end{split}\end{equation*}
\label{}
\end{lemma}
\begin{proof}
The assertion follows from McDiarmid's inequality after noting that altering one sample results in an absolute change bounded by \(\frac{1}{n}(2 \rmax)^2\).
\end{proof}

Thus, a generalization error bound in terms of \(\Risk\) can be converted to that of an \(\Ltwo\) distance when the true density ratio and the density ratio model are square-integrable and \(\br\) is strongly convex. However, when using the result of Theorem~\ref{thm:est_error_bound}, the convergence rate shown here is slower than $\Orderp{(\vmin{\nde}{\nnu})^{-1/(4)}}$. On the other hand, \citet{KanamoriStatistical2012a} derived $\Orderp{(\vmin{\nde}{\nnu})^{-1/(2+\gamma)}}$ convergence rate. To derive this bound when using neural network, we need to restrict the neural network models. In the following part, we prove Theorem~\ref{appdx:thm:convergence-rate} for the following hypothesis class $\rClass$. 

\begin{definition}[ReLU neural networks; \citealp{Schmidt-HieberNonparametric2020}]\label{appdx:sparse-network-function-class}
  For \(L \in \Na\) and \(p = (p_0, \ldots, p_{L+1}) \in \Na^{L+2}\),
  \begin{align*}
    \mathcal{F}(L, p) := &\{f: x \mapsto W_L\sigma_{v_L}W_{L-1}\sigma_{v_{L-1}} \cdots W_1 \sigma_{v_1} W_0 x :\\
                         &\qquad\qquad W_i \in \Re^{p_{i+1} \times p_i}, v_i \in \Re^{p_i} (i = 0, \ldots, L)\},
  \end{align*}
  where \(\sigma_v(y) := \sigma(y - v)\), and \(\sigma(\cdot) = \max\{\cdot, 0\}\) is applied in an element-wise manner.
  Then, for \(s \in \Na, F \geq 0, L \in \Na\), and \(p \in \Na^{L+2}\), define
  \begin{align*}
    \rClass(L, p, s, F) := \{f \in \mathcal{F}(L, p): \sum_{j=0}^L \lzeroNrm{W_j} + \lzeroNrm{v_j} \leq s, \|f\|_\infty \leq F\},
  \end{align*}
  where \(\lzeroNrm{\cdot}\) denotes the number of non-zero entries of the matrix or the vector, and \(\|\cdot\|_\infty\) denotes the supremum norm.
  Now, fixing \(\Lmax, \pmax, s \in \Na\) as well as \(F >0\), we define
  \[\IndLP := \{(L, p): L \in \Na, L \leq \Lmax, p \in [\pmax]^{L+2}\},\]
  and we consider the hypothesis class
  \begin{align*}
    \bar \rClass &:= \bigcup_{(L, p) \in \IndLP}\rClass(L, p, s, F) \\
    \rClass &:= \{r \in \bar\rClass: \mathrm{Im}(r) \subset \rClassRangeTwo\}.
  \end{align*}

  Moreover, we define \(I_1: \IndLP \to \Re\) and \(I: \rClass \to [0, \infty)\) by
  \begin{align*}
    I_1(L, p) &:= 2|\IndLP|^{\frac{1}{s+1}} (L+1) V^2,\\
    I(\r) &:= \max\left\{\|r\|_\infty, \min_{\stackrel{(L, p) \in \IndLP}{r \in \rClassLP}} I_1(L, p)\right\},
  \end{align*}
  where \(V := \prod_{l=0}^{L+1}(p_l + 1)\),
  and we define
  \[\rClassM := \{\r \in \rClass: I(\r) \leq M\}.\]
\end{definition}
 Note that the requirement for the hypothesis class of Theorem~\ref{thm:est_error_bound} is not as tight as that of Theorem~\ref{appdx:thm:convergence-rate}. Then, we prove Theorem~\ref{appdx:thm:convergence-rate} as follows:
\begin{proof}
Thanks to the strong convexity, by Lemma~\ref{appdx:sec:appendix:strong-convexity}, we have
\begin{align*}
  &\frac{\mu}{2}\|\hr - \rstar\|_{\Ltwo(\pde)}^2
    \leq \Risk(\hr) - \Risk(\rstar) \\
  &= \Risk(\hr) - \Risk(\rstar)\\
  &\qquad\annot{- \hRisk(\hr) + \hRisk(\hr)}{\(=0\)} \annot{- \nnhRisk(\hr) + \nnhRisk(\hr)}{\(=0\)}
    \annot{- \hRisk(\rstar) + \hRisk(\rstar)}{\(=0\)}\\
  &\leq \Risk(\hr) - \hRisk(\hr) + (\hRisk(\hr) - \nnhRisk(\hr))\\
  &\qquad+ (\nnhRisk(\rstar) - \hRisk(\rstar)) + \hRisk(\rstar) - \Risk(\rstar) \\
  &\leq \annot{(\Risk(\hr) - \Risk(\rstar) + \hRisk(\rstar) - \hRisk(\hr))}{\(=:A\)} + \annot{2 \supr{}|\hRisk(\r) - \nnhRisk(\r)|}{\(=:B\)},
\end{align*}
where we used \(\nnhRisk(\hr) \leq \nnhRisk(\rstar)\).
To bound \(A\), for ease of notation, let \(\lOne{\r} = \lossOne(\r(X))\) and \(\lTwo{\r} = \lossTwo(\r(X))\).
Then, since
\begin{align*}
  \Risk(\r) &= \Ede\lOneR - \Cons\Enu\lOneR + \Enu\lTwoR, \\
  \hRisk(\r) &= \hEde\lOneR - \Cons\hEnu\lOneR + \hEnu\lTwoR,
\end{align*}
we have
\begin{align*}
  &A = \Risk(\hr) - \Risk(\rstar) + \hRisk(\rstar) - \hRisk(\hr) \\
  &= (\Ede - \hEde)(\lOne{\hr} - \lOne{\rstar}) - \Cons(\Enu - \hEnu)(\lOne{\hr} - \lOne{\rstar}) + (\Enu - \hEnu)(\lTwo{\hr} - \lTwo{\rstar})\\
  &\leq |(\Ede - \hEde)(\lOne{\hr} - \lOne{\rstar})| + \Cons|(\Enu - \hEnu)(\lOne{\hr} - \lOne{\rstar})|
    + |(\Enu - \hEnu)(\lTwo{\hr} - \lTwo{\rstar})|
\end{align*}
By applying Lemma~\ref{appdx:lem:empirical-deviations}, for any \(0 < \gamma < 2\), we have
\[A \leq \Orderp{\vmax{\frac{\|\hr - \rstar\|_{\Ltwo(\pde)}^{1 - \gamma/2}}{\sqrt{\vmin{\nde}{\nnu}}}}{\frac{1}{(\vmin{\nde}{\nnu})^{2/(2+\gamma)}}}}.\]
On the other hand, by Lemma~\ref{appdx:lem:estimator-difference-order} and
Lemma~\ref{appdx:sparse-network-rademacher}, and the assumption \(\infr
\Emod\lOneR > 0\), there exists \(\alpha > 0\) such that we have \(B \leq \Orderp{\exponentialOrderOne}\).
Combining the above bounds on \(A\) and \(B\), for any \(0 < \gamma < 2\), we get
\begin{align*}
  \|\hr - \rstar\|_{\Ltwo(\pde)}^2 &\leq \Orderp{\vmax{\frac{\|\hr - \rstar\|_{\Ltwo(\pde)}^{1 - \gamma/2}}{\sqrt{\vmin{\nde}{\nnu}}}}{\frac{1}{(\vmin{\nde}{\nnu})^{2/(2+\gamma)}}}} \\
                                   &\qquad\qquad+ \Orderp{\exponentialOrderOne} \\
                                   &\leq \Orderp{\vmax{\frac{\|\hr - \rstar\|_{\Ltwo(\pde)}^{1 - \gamma/2}}{\sqrt{\vmin{\nde}{\nnu}}}}{\frac{1}{(\vmin{\nde}{\nnu})^{2/(2+\gamma)}}}}. \\
\end{align*}
As a result, we have \[\|\hr - \rstar\|_{\Ltwo(\pde)} \leq \Orderp{(\vmin{\nde}{\nnu})^{-\frac{1}{2+\gamma}}}.\]
\end{proof}

Each lemma used in the proof is provided as follows.

\subsection{Complexity of the hypothesis class}
For the function classes in Definition~\ref{appdx:sparse-network-function-class}, we have the following evaluations of their complexities.

\begin{lemma}[Lemma~5 in \citet{Schmidt-HieberNonparametric2020}]\label{appdx:sparse-network-entropy}
  For \(L \in \Na\) and \(p \in \Na^{L+2}\), let \(V := \prod_{l=0}^{L+1}(p_l + 1)\). Then, for any \(\delta > 0\),
  \[\log \mathcal{N}(\delta, \rClass(L, p, s, \infty), \|\cdot\|_\infty) \leq (s+1)\log(2 \delta^{-1} (L+1) V^2).\]
\end{lemma}

\begin{lemma}\label{appdx:sparse-network-rademacher}
  There exists \(\sparseNetworkRadBoundConst > 0\) such that
  \[\Radnu(\rClass) \leq \sparseNetworkRadBoundConst \nnu^{-1/2}, \quad \Radde(\rClass) \leq \sparseNetworkRadBoundConst \nde^{-1/2}.\]
  \begin{proof}
    By Dudley's entropy integral bound
    \citep[Theorem~5.22]{WainwrightHighDimensional2019} and Lemma~\ref{appdx:sparse-network-entropy}, we have
    \begin{align*}
      \Radnu(\rClass(L, p, s, F)) &\leq 32\int_0^{2F} \sqrt{\frac{\log \mathcal{N}(\delta, \rClass(L, p, s, F), \|\cdot\|_\infty)}{\nnu}} d\delta \\
                                  &= \left(32 \int_0^{2F} \left((s+1)\log(2 \delta^{-1} (L+1) V^2)\right)^{1/2} d\delta\right) \nnu^{-1/2}.
    \end{align*}
    Therefore, there exists \(\sparseNetworkRadBoundConst > 0\) such that
    \begin{align*}
      \Radnu(\rClass) \leq \sum_{(L, p) \in \IndLP} \Radnu(\rClass(L, p, s, F)) \leq \sparseNetworkRadBoundConst \nnu^{-1/2}.
    \end{align*}
    The same argument applies to \(\Radde(\rClass)\), and we obtain the assertion.
  \end{proof}
\end{lemma}

\begin{lemma}\label{appdx:sparse-network-complexity-bounds}
  There exists \(\rClassMSupNormBoundConst > 0\) such that for any \(\gamma > 0\), any \(\delta > 0\), and any \(M \geq 1\), we have
  \begin{align*}
    \log \mathcal{N}\left(\delta, \rClassM, \|\cdot\|_\infty\right) &\leq \frac{s+1}{\gamma} \left( \frac{M}{\delta} \right)^{\gamma}.
  \end{align*}
  and
  \begin{align*}
    \sup_{\r \in \rClassM} \|\r - \rstar\|_\infty &\leq \rClassMSupNormBoundConst M.
  \end{align*}

  \begin{proof}
    The first assertion is a result of the following calculation:
    \begin{align*}
      \log \mathcal{N}\left(\delta, \rClassM, \|\cdot\|_\infty\right)
      &\leq \log \sum_{\IndLPM} \inftyCoveringNumber{\delta}{\rClass(L, p, s, M)}\\
        &\leq \log \sum_{\IndLPM} \left(\frac{2}{\delta}(L+1)V^2\right)^{s+1}\\
        &\leq \log |\IndLP| \left(\frac{1}{\delta}M |\IndLP|^{-\frac{1}{s+1}}\right)^{s+1}\\
        &= (s+1)\log\left(\frac{M}{\delta}\right) < (s+1) \frac{1}{\gamma} \left( \frac{M}{\delta} \right)^{\gamma},
    \end{align*}
    where the first inequality follows from \(\rClassM \subset \bigcup_{\IndLPMInline}\rClassLP\),
    and the last inequality from
    \(\gamma\log x^{\frac{1}{\gamma}} = \log x < x\) that holds for all \(x, \gamma > 0\).

    The second assertion can be confirmed by noting that for any
    \(\r \in \rClassM\) with \(M \geq 1\),
    \begin{align*}
      \|\r - \rstar\|_\infty
      &\leq \|\r\|_\infty + \|\rstar\|_\infty
        \leq M + \|\rstar\|_\infty\\
      &\leq \left(1 + \frac{\|\rstar\|_\infty}{M}\right) M
        \leq (1+\|\rstar\|_\infty) M
    \end{align*}
    holds.

  \end{proof}
\end{lemma}

\begin{definition}[Derived function class and bracketing entropy]
  Given a real-valued function class \(\mathcal{F}\),
  define \(\ell \circ \mathcal{F} := \{\ell \circ f: f \in \mathcal{F}\}\).
  By extension, we define \(I: \elled\rClass \to [1, \infty)\) by \(I(\elled\r) = I(r)\) and \(\elled\rClassM := \{\elled\r : \r \in \rClassM\}\).
  Note that, as a result, \(\elled\rClassM\) coincides with \(\{\elled\r \in \elled\rClass: I(\elled\r) \leq M\}\).
\end{definition}

\begin{lemma}\label{appdx:lem:elled-sparse-network-complexity}
  Let \(\ell: \rClassRangeTwo \to \Re\) be a \(\ellLip\)-Lipschitz continuous function.
  Let \(\bracketEntropy{\delta}{\mathcal{F}}{\|\cdot\|_{\Ltwo(P)}}\) denote the bracketing entropy of \(\mathcal{F}\) with respect to a distribution \(P\).
  Then, for any distribution \(P\), any \(\gamma > 0\), any \(M \geq 1\), and any \(\delta > 0\), we have
  \begin{align*}
    \bracketEntropy{\delta}{\ell \circ \rClassM}{\|\cdot\|_{\Ltwo(P)}} &\leq \frac{(s+1)(2\ellLip)^{\gamma}}{\gamma} \left(\frac{M}{\delta} \right)^{\gamma}.
  \end{align*}
  Moreover, there exists \(\rClassMSupNormBoundConst > 0\) such that for any \(M \geq 1\) and any distribution \(P\),
  \begin{align*}
    \sup_{\elled\r \in \elled\rClassM} \|\elled\r - \elled\rstar\|_{\Ltwo(P)} &\leq \rClassMSupNormBoundConst\ellLip M, \\
    \sup_{\stackrel{\elled\r \in \elled\rClassM}{\|\elled\r - \elled\rstar\|_{\Ltwo(P)} \leq \delta}} \|\elled\r - \elled\rstar\|_\infty &\leq \rClassMSupNormBoundConst\ellLip M, \quad \text{for all } \delta > 0.
  \end{align*}

  \begin{proof}
    By combining Lemma~2.1 in \citet{vandeGeerEmpirical2000} with Lemma~\ref{appdx:sparse-network-entropy}, we have
    \begin{align*}
      \bracketEntropy{\delta}{\elled\rClassM}{\|\cdot\|_{\Ltwo(P)}}
      &\leq \log \mathcal{N}\left(\frac{\delta}{2}, \elled\rClassM, \|\cdot\|_\infty\right), \\
      &\leq \log \mathcal{N}\left(\frac{\delta}{2\ellLip}, \rClassM, \|\cdot\|_\infty\right)
        \leq \frac{s+1}{\gamma} \left(\frac{2\ellLip M}{\delta}\right)^{\gamma}.
    \end{align*}
    For \(M \geq 1\), we have
    \begin{align*}
      \sup_{\elled\r \in \elled\rClassM}\|\elled\r - \elled\rstar\|_{\Ltwo(P)}
      &\leq \sup_{\elled\r \in \elled\rClassM}\|\elled\r - \elled\rstar\|_{\infty}\\
      \sup_{\stackrel{\elled\r \in \elled\rClassM}{\|\elled\r - \elled\rstar\|_{\Ltwo(P)} \leq \delta}} \|\elled\r - \elled\rstar\|_\infty
      &\leq \sup_{\elled\r \in \elled\rClassM}\|\elled\r - \elled\rstar\|_\infty,
    \end{align*}
    and Lemma~\ref{appdx:sparse-network-entropy} implies
    \begin{align*}
      \sup_{\elled\r \in \elled\rClassM}\|\elled\r - \elled\rstar\|_{\infty}
      \leq \sup_{\r \in \rClassM} \ellLip \|\r - \rstar\|_\infty
      \leq \ellLip\rClassMSupNormBoundConst M.
    \end{align*}

  \end{proof}
\end{lemma}

\subsection{Bounding the empirical deviations}
\begin{lemma}\label{appdx:lem:empirical-deviations}
  Under the conditions of Theorem~\ref{appdx:thm:convergence-rate},
  for any \(0 < \gamma < 2\), we have
  \begin{align*}
    |(\Ede - \hEde)(\lOne{\hr} - \lOne{\rstar})| &= \Orderp{\vmax{\frac{\|\hr - \rstar\|_{\Ltwo(\pde)}^{1 - \gamma/2}}{\sqrt{\nde}}}{\frac{1}{\nde^{2/(2+\gamma)}}}}\\
    |(\Enu - \hEnu)(\lOne{\hr} - \lOne{\rstar})| &= \Orderp{\vmax{\frac{\|\hr - \rstar\|_{\Ltwo(\pde)}^{1 - \gamma/2}}{\sqrt{\nnu}}}{\frac{1}{\nnu^{2/(2+\gamma)}}}}\\
    |(\Enu - \hEnu)(\lTwo{\hr} - \lTwo{\rstar})| &= \Orderp{\vmax{\frac{\|\hr - \rstar\|_{\Ltwo(\pde)}^{1 - \gamma/2}}{\sqrt{\nnu}}}{\frac{1}{\nnu^{2/(2+\gamma)}}}}\\
  \end{align*}
  as \(\nnu, \nde \to \infty\).
  \begin{proof}
    Since \(0 < \gamma < 2\), we can apply Lemma~\ref{appdx:lem:van-de-geer} in combination
    with Lemma~\ref{appdx:lem:elled-sparse-network-complexity} to obtain
    \begin{align*}
      \supr\frac{|(\Ede - \hEde)(\lOne{\r} - \lOne{\rstar})|}{D_1(\r)} &= \Orderp{1},\\
      \supr\frac{|(\Enu - \hEnu)(\lOne{\r} - \lOne{\rstar})|}{D_2(\r)} &= \Orderp{1},\\
      \supr\frac{|(\Enu - \hEnu)(\lTwo{\r} - \lTwo{\rstar})|}{D_3(\r)} &= \Orderp{1},
    \end{align*}
    where
    \begin{align*}
      D_1(\r) &= \vmax{\frac{\|\lOne{\r} - \lOne{\rstar}\|_{\Ltwo(\pde)}^{1 - \gamma/2}I(\lOne{\r})^{\gamma/2}}{\sqrt{\nde}}}{\frac{I(\lOne{\r})}{\nde^{2/(2+\gamma)}}}, \\
      D_2(\r) &= \vmax{\frac{\|\lOne{\r} - \lOne{\rstar}\|_{\Ltwo(\pnu)}^{1 - \gamma/2}I(\lOne{\r})^{\gamma/2}}{\sqrt{\nnu}}}{\frac{I(\lOne{\r})}{\nnu^{2/(2+\gamma)}}}, \\
      D_3(\r) &= \vmax{\frac{\|\lTwo{\r} - \lTwo{\rstar}\|_{\Ltwo(\pnu)}^{1 - \gamma/2}I(\lTwo{\r})^{\gamma/2}}{\sqrt{\nnu}}}{\frac{I(\lTwo{\r})}{\nnu^{2/(2+\gamma)}}}, \\
    \end{align*}
    Noting that \(\supr I(\r) < \infty\),
    that \(\lossTwo, \lossOne\) are Lipschitz continuous,
    and that \(\|\hr - \rstar\|_{\Ltwo(\pnu)} \leq \left(\sup_{x \in \mathcal{X}}\left| \frac{\pnu(x)}{\pde(x)} \right| \right) \|\hr - \rstar\|_{\Ltwo(\pde)}\) holds,
    we have the assertion.
  \end{proof}
\end{lemma}

Following is a proposition originally presented in \citet{vandeGeerEmpirical2000}, which was rephrased in \citet{KanamoriStatistical2012a} in a form that is convenient for our purpose.
\begin{lemma}[Lemma~5.14 in \citet{vandeGeerEmpirical2000}, Proposition~1 in \citet{KanamoriStatistical2012a}]\label{appdx:lem:van-de-geer}
  Let \(\mathcal{F} \subset \Ltwo(P)\) be a function class and the map \(I(f)\)
  be a complexity measure of \(f \in \mathcal{F}\), where \(I\) is a
  non-negative function on \(\mathcal{F}\) and \(I(f_0) < \infty\) for a fixed
  \(f_0 \in \mathcal{F}\). We now define \(\mathcal{F}_M = \{f \in \mathcal{F} :
  I(f) \leq M\}\) satisfying \(\mathcal{F} = \bigcup_{M \geq 1} \mathcal{F}_M\).
  Suppose that there exist \(c_0 > 0\) and \(0 < \gamma < 2\) such that
  \[\sup_{f \in \mathcal{F}_M} \|f - f_0\| \leq c_0 M, \ \sup_{\stackrel{f \in
        \mathcal{F}_M}{\|f - f_0\|_{\Ltwo(P)} \leq \delta}} \|f - f_0\|_\infty
    \leq c_0 M, \quad \text{for all } \delta > 0,\]
  and that \(H_B(\delta, \mathcal{F}_M, P) = \Order{M/\delta}^\gamma\).
  Then, we have
  \[\sup_{f \in \mathcal{F}} \frac{\left| \int (f - f_0)d(P - P_n)
      \right|}{D(f)} = \Orderp{1}, \ (n \to \infty),\]
  where \(D(f)\) is defined by
  \[D(f) = \vmax{\frac{\|f - f_0\|_{\Ltwo(P)}^{1 - \gamma/2}I(f)^{\gamma/2}}{\sqrt{n}}}{\frac{I(f)}{n^{2/(2+\gamma)}}}.\]
\end{lemma}

\subsection{Bounding the difference of the BD estimators}

\begin{lemma}\label{appdx:lem:estimator-difference-order}
  Assume \(\Radde(\rClass) = \order(1) (\nde \to \infty)\) and \(\Radnu(\rClass) = \order(1) (\nnu \to \infty)\).
  Also assume the same conditions as Theorem~\ref{thm:estimation-error-bound}.
  Then,
  \begin{align*}
    \supr{}|\nnhRisk(\r) - \hRisk(\r)| = \Orderp{\exponentialOrderOne}
  \end{align*}
  as \(\nnu, \nde \to \infty\).
  \begin{proof}
    First, by combining Lemma~\ref{appdx:lem:estimator-difference},
    the assumption on the Rademacher complexities, and Markov's inequality,
    there exist \(\alpha > 0\) and \(\ndez, \nnuz \in \Na\) such that for
    any \(\nde \geq \ndez\) and \(\nnu \geq \nnuz\) and any \(\delta \in (0, 1)\), we have with probability at least \(1 - \delta\),
    \begin{align*}
      \supr{}|\nnhRisk(\r) - \hRisk(\r)| \leq \frac{(1+\Cons)\LossBound\modFnIdLip}{\delta}\exponentialOrderOne{}.
    \end{align*}
    Therefore, we have the assertion.
  \end{proof}
\end{lemma}

\begin{lemma}
  \label{appdx:lem:estimator-difference}
  Assume \(\Radde(\rClass) = \order(1) (\nde \to \infty)\) and \(\Radnu(\rClass) = \order(1) (\nnu \to \infty)\).
  Also assume the same conditions as Theorem~\ref{thm:estimation-error-bound}.
  Then, there exist \(\alpha > 0\) and \(\ndez, \nnuz \in \Na\) such that for any \(\nde \geq \ndez\) and \(\nnu \geq \nnuz\),
  \begin{align*}
    \E\left[\supr{}|\nnhRisk(\r) - \hRisk(\r)|\right] \leq (1+\Cons)\LossBound\modFnIdLip\exp\left(- \frac{2 \alpha^2}{(\LossBound^2/\nde) + (\Cons^2\LossBound^2 / \nnu)}\right)
  \end{align*}
  holds.
  \begin{proof}
    First, we have
    \begin{align*}
      &\E\left[\supr{}|\nnhRisk(\r) - \hRisk(\r)|\right] \\
      &= \E \left[\supr{} \left|\modFn(\hEmod\lOneR) - \hEmod\lOneR\right|\right] \\
      &= \E \left[\supr{}\Indicator{\modFn(\hEmod\lOneR) \neq \hEmod\lOneR}\cdot|\modFn(\hEmod\lOneR) - \hEmod\lOneR|\right] \\
      &\leq \E\left[\supr{}\Indicator{\modFn(\hEmod\lOneR) \neq \hEmod\lOneR}\right]\left(\supLossVal|\modFn(s) - s|\right),
    \end{align*}
    where \(\Indicator{\cdot}\) denotes the indicator function, and we used \(|\hEmod\lOneR| \leq (1 + C)\LossBound\). Further, we have
    \begin{equation*}\begin{aligned}
        &\supLossVal|\modFn(s) - s| \leq \supLossVal|(\modFn - \Identity)(s) - (\modFn - \Identity)(0)| + |(\modFn - \Identity)(0)| \\
        &\leq \supLossVal \modFnIdLip |s - 0| + 0
        \leq (1+\Cons)\LossBound\modFnIdLip,
      \end{aligned}\end{equation*}
    where \(\Identity\) denotes the identity function.
    On the other hand, since \(\infr \Emod\lOneR > 0\) is assumed, there exists \(\beta > 0\) such that for any \(\r \in \rClass\), \(\Emod\lOneR > \beta\).
    Therefore, denoting the support of a function by \(\supp(\cdot)\),
    \begin{align*}
      &\E\left[\supr{}\Indicator{\modFn(\hEmod\lOneR) \neq \hEmod\lOneR}\right] \\
      &= \E\left[\supr{} \Indicator{\hEmod\lOneR \in \modFnIdSupp}\right] \\
      &= \E\left[\supr{} \Indicator{\hEmod\lOneR < 0}\right] \\
      &= \E\left[\Indicator{\exists r \in \rClass: \hEmod\lOneR < 0}\right] \\
      &= \Probability{\exists r \in \rClass: \hEmod\lOneR < 0} \\
      &\leq \Probability{\exists r \in \rClass: \hEmod\lOneR < \Emod\lOneR - \beta} \\
      &\leq \Probability{\beta < \supr(\Emod\lOneR - \hEmod\lOneR)}.
    \end{align*}
    Take an arbitrary \(\alpha \in (0, \beta)\).
    Since \(\Radde(\rClass) \to 0 (\nde \to \infty)\) and \(\Radnu(\rClass) \to 0 (\nnu \to \infty)\),
    we can apply Lemma~\ref{appdx:lem:empirical-process-bound} and obtain the assertion.
  \end{proof}
\end{lemma}

\begin{lemma}
  \label{appdx:lem:empirical-process-bound}
  Let $\beta > \alpha > 0$.
  Assume that there exist \(\ndez, \nnuz \in \Na\) such that for any \(\nde \geq \ndez\) and \(\nnu \geq \nnuz\),
  \begin{align*}
    4 \LipOne\Radde(\rClass) + 4 C \LipOne\Radnu(\rClass) < \beta - \alpha.
  \end{align*}
  Then, for any \(\nde \geq \ndez\) and \(\nnu \geq \nnuz\), we have
  \begin{align*}
    &\Probability{\beta < \supr(\Emod\lOneR - \hEmod\lOneR)} \\
    &\quad\quad\leq \exponentialOrderOne{}.
  \end{align*}
  \begin{proof}
    First, we will apply McDiarmid's inequality.
    The absolute difference caused by altering one data point in \(\supr(\Emod\lOneR - \hEmod\lOneR)\)
    is bounded by \(\frac{\LossBound}{\nde}\) if the change is in a sample from \(\pde\) and \(\frac{\Cons \LossBound}{\nnu}\) otherwise.
    This can be confirmed by letting \(\hEmodTwo{}\) denote the sample averaging
    operator obtained by altering one data point in \(\hEmod{}\) and observing
    \begin{align*}
      &\supr\{\Emod\lOneR - \hEmod\lOneR\} - \supr\{\Emod\lOneR - \hEmodTwo\lOneR\} \\
      &\leq \supr\{\Emod\lOneR - \hEmod\lOneR - (\Emod\lOneR - \hEmodTwo\lOneR)\} \\
      &\leq \supr\{\hEmodTwo\lOneR - \hEmod\lOneR\}.
    \end{align*}
    The right-most expression can be bounded by \(\frac{\LossBound}{\nde}\) if the change is in a sample from \(\pde\) and \(\frac{\Cons \LossBound}{\nnu}\) otherwise.
    Likewise, \(\supr(\Emod\lOneR - \hEmodTwo\lOneR) - \supr(\Emod\lOneR -
    \hEmod\lOneR)\) can be bounded by one of these quantities.
    Therefore, we have
    \begin{align*}
      &\left|\supr\{\Emod\lOneR - \hEmod\lOneR\} - \supr\{\Emod\lOneR - \hEmodTwo\lOneR\}\right| \\
      &\leq \frac{\LossBound}{\nde} + \frac{\Cons \LossBound}{\nnu},
    \end{align*}
    and McDiarmid's inequality implies, for any \(\epsilon > 0\),
    \begin{equation}\label{appdx:eq:emp-proc-mcdiarmid}\begin{split}
        &\Probability{\epsilon < \supr(\Emod\lOneR - \hEmod\lOneR) - \E\left[\supr(\Emod\lOneR - \hEmod\lOneR)\right]} \\
        &\leq \exp\left(- \frac{2 \epsilon^2}{(\LossBound^2/\nde) + (\Cons^2\LossBound^2 / \nnu)}\right).
      \end{split}\end{equation}
    Now, applying Lemma~\ref{lem:general-symmetrization}, we have
    \begin{align*}
      &\E\left[\supr(\Emod\lOneR - \hEmod\lOneR)\right] \\
      &\leq \E\left[\supr |\Ede \lOneR - \hEde\lOneR|\right] + C\E\left[\supr |\Enu \lOneR - \hEnu\lOneR|\right] \\
      &\leq 4 \LipOne\Radde(\rClass) + 4 C \LipOne\Radnu(\rClass) =: \AppdxEmpProcMcDiarmidRadSum{}.
    \end{align*}
    By the assumption, if \(\nde \geq \ndez\) and \(\nnu \geq \nnuz\), we have
    \(\AppdxEmpProcMcDiarmidRadSum{} < \beta - \alpha\). Therefore,
    \begin{align*}
      \E\left[\supr(\Emod\lOneR - \hEmod\lOneR)\right] < \beta - \alpha < \beta,
    \end{align*}
    hence \(\beta - \E\left[\supr(\Emod\lOneR - \hEmod\lOneR)\right] > 0\).
    Therefore, we can take \(\epsilon = \beta - \E\left[\supr(\Emod\lOneR -
      \hEmod\lOneR)\right]\) in Equation~\eqref{appdx:eq:emp-proc-mcdiarmid} to obtain
    \begin{align*}
      &\Probability{\beta < \supr(\Emod\lOneR - \hEmod\lOneR)} \\
      &\leq \exp\left(- \frac{2 (\beta - \E\left[\supr(\Emod\lOneR - \hEmod\lOneR)\right])^2}{(\LossBound^2/\nde) + (\Cons^2\LossBound^2 / \nnu)}\right) \\
      &\leq \exp\left(- \frac{2 (\beta - \AppdxEmpProcMcDiarmidRadSum{})^2}{(\LossBound^2/\nde) + (\Cons^2\LossBound^2 / \nnu)}\right) \leq \exp\left(- \frac{2 \alpha^2}{(\LossBound^2/\nde) + (\Cons^2\LossBound^2 / \nnu)}\right),
    \end{align*}
    where we used \(0 < \alpha < \beta - \AppdxEmpProcMcDiarmidRadSum{}\).
  \end{proof}
\end{lemma}

\end{document}